\documentclass[english]{IEEEtran}
\usepackage[T1]{fontenc}
\usepackage[latin9]{inputenc}
\usepackage{color}
\usepackage[dvipsnames]{xcolor}
\usepackage{verbatim}
\usepackage{float}
\usepackage{amsmath}
\usepackage{amsthm}
\usepackage{amssymb}
\usepackage{stmaryrd}
\usepackage{stackrel}
\usepackage{graphicx}
\usepackage{wasysym}
\usepackage{mathtools}
\usepackage{hyperref}
\usepackage{url}   
\usepackage{booktabs} 
\usepackage{nicefrac} 
\usepackage{authblk}
\usepackage[normalem]{ulem}
\usepackage{microtype}
\usepackage{array}
\usepackage{subfig}
\usepackage[dvipsnames]{xcolor}
\usepackage{paralist}
\makeatletter

\floatstyle{ruled}
\newfloat{algorithm}{tbp}{loa}
\providecommand{\algorithmname}{Algorithm}
\floatname{algorithm}{\protect\algorithmname}

\theoremstyle{plain}
\newtheorem{thm}{\protect\theoremname}
\theoremstyle{definition}
\newtheorem{defn}{\protect\definitionname}
\theoremstyle{definition}
\newtheorem{problem}{\protect\problemname}
\theoremstyle{plain}
\newtheorem{lem}{\protect\lemmaname}
\theoremstyle{definition}
\newtheorem{example}{\protect\examplename}
\theoremstyle{remark}
\newtheorem{rem}{\protect\remarkname}
\theoremstyle{plain}

\newcommand*{\Scale}[2][4]{\scalebox{#1}{$#2$}}%
\usepackage{tikz}
\tikzset{>=latex}
\usetikzlibrary{fit,automata,positioning,angles,quotes}

\usepackage{cite}
\usepackage{algpseudocode}
\usepackage{setspace}
\IEEEoverridecommandlockouts

\makeatother

\usepackage{babel}
\providecommand{\corollaryname}{Corollary}
\providecommand{\definitionname}{Definition}
\providecommand{\examplename}{Example}
\providecommand{\lemmaname}{Lemma}
\providecommand{\problemname}{Problem}
\providecommand{\remarkname}{Remark}
\providecommand{\theoremname}{Theorem}

\newcommand{\UntilOp}{\mathcal{U}}
\newcommand{\Eventually}{\diamondsuit}
\newcommand{\Always}{\square}
\newcommand{\Next}{\Circle}
\newcommand{\AP}{{AP}}

\usepackage{amsmath}               
  {
      \theoremstyle{plain}
      \newtheorem{assumption}{Assumption}
      \theoremstyle{plain}
      \newtheorem{proposition}{Proposition}
  }

\SetSymbolFont{stmry}{bold}{U}{stmry}{m}{n}

\def\BibTeX{{\rm B\kern-.05em{\sc i\kern-.025em b}\kern-.08em
		T\kern-.1667em\lower.7ex\hbox{E}\kern-.125emX}}

\begin{document}

\title{Safety-Critical Learning of Robot Control with Temporal Logic Specifications

\thanks{$^{1}$Department of Mechanical Engineering, Lehigh University, Bethlehem, PA, USA.  Email: \{mingyu-cai, cvasile\}@lehigh.edu}
}

\author{Mingyu Cai$^{1}$, Cristian-Ioan Vasile$^{1}$}



\maketitle

\begin{abstract}
Reinforcement learning (RL) is a promising approach.
However, success is limited to real-world applications, because ensuring safe exploration and facilitating adequate exploitation is a challenge for controlling robotic systems with unknown models and measurement uncertainties.
The learning problem becomes even more difficult for complex tasks over continuous state-action.
In this paper, we propose a learning-based robotic control framework consisting of several aspects:
(1) we leverage Linear Temporal Logic (LTL) to express complex tasks over infinite horizons that are translated to a novel automaton structure;
(2) we detail an innovative reward scheme for LTL satisfaction with a probabilistic guarantee. Then, by applying a reward shaping technique, we develop a modular policy-gradient architecture exploiting the benefits of the automaton structure to decompose overall tasks and enhance the performance of learned controllers;
(3) by incorporating Gaussian Processes (GPs) to estimate the uncertain dynamic systems, we synthesize a model-based safe exploration during the learning process using Exponential Control Barrier Functions (ECBFs) that generalize systems with high-order relative degrees;
(4) to further improve the efficiency of exploration, we utilize the properties of LTL automata and ECBFs to propose a safe guiding process.
Finally, we demonstrate the effectiveness of the framework via several robotic environments.
We show an ECBF-based modular deep RL algorithm that achieves near-perfect success rates and safety guarding with high probability confidence during training. 

\global\long\def\Inf{\operatorname{Inf}}%
\global\long\def\Acc{\operatorname{Acc}}%
\end{abstract}


\section{INTRODUCTION}

Reinforcement learning (RL) is a sequential decision-making
process and focuses on learning optimal policies for robots that maximize the long-term reward via sampling from unknown environments \cite{Sutton2018}.
Markov decision processes (MDP) are often employed to model the dynamics robots and interaction with environments.
Growing research has been devoted to studying RL-based motion planning over MDPs without any prior
knowledge of the complex robot dynamics and uncertainty models.
This approach has been successfully employed in robotics where it was extended to continuous state-action spaces via actor-critic methods~\cite{Schulman2015,Lillicrap2016,Schulman2017}.
However, the key feature of RL is its sole dependence on the exploration of the environment.
The challenge of interpreting the inner workings of many RL algorithms makes it intractable to encode the behaviors of the systems during training, especially while the learning parameters have not yet converged to a stable control policy.
Imposing safety critical conditions on robotic systems that avoid failure and protect them from physical harm are prime example of behaviors that need to be enforced.
Due to the safety-critical requirements of real-world robotic applications during the learning process, most modern RL algorithms have limited success on physical systems beyond simulated applications. 
In this work, we propose a safety-critical control framework for continuous-time systems with uncertain and nominal robotic dynamics, integrated with data-driven machine learning and formal methods. {\color{black}
As an example, in Fig.~\ref{fig:Demonstration},
we consider a satellite map of Mars' surface labeled with unsafe craters and areas of interest.
The ground robotic rover needs to complete a safety-critical and complex goal-oriented task defined in a formal language, and is subject to external uncertainties due to the inaccurate physical parameters and complex terrain.
}

\begin{figure}
	\centering{}\includegraphics[scale=0.47]{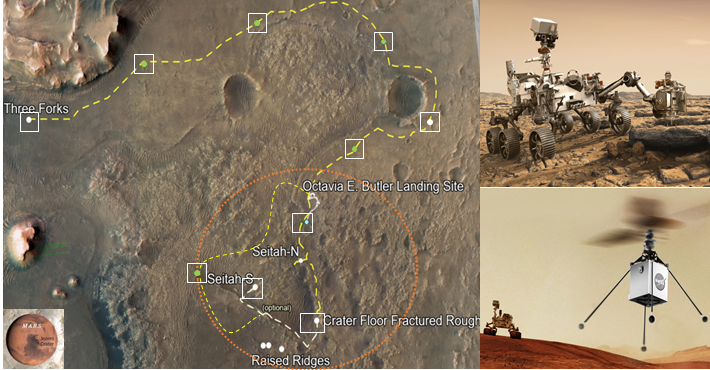}\caption{\label{fig:Demonstration} The Mars exploration requires a ground rover to complete the complex high-level specifications involving regions of interest while avoiding hazardous areas, which are defined by humans based on the satellite map created using aerial imagery from the helicopter. Due to complex environmental settings, there exist unknown uncertainties in the dynamic systems. The motivation is to learn desired control policies while online identifying unknown model uncertainties to guarantee safe exploration.}
\end{figure}

In the field of intersections between Artificial Intelligence and robotics, control barrier function (CBF) was introduced as a tool to ensure safety-critical constraints are met~\cite{Ames2016}.
{\color{black}Recently, several methods addressing the issue of model uncertainty in the safety-critical problem using a data-driven approach have been proposed. The work~\cite{choi2020reinforcement} pre-specified the control barrier function and control Lyapunov function to achieve stability and safety, while reinforcement learning is applied to learn the mismatch between the nominal model and true dynamics. This strategy has also been extended using GPs to provide guarantees of model identification~\cite{castaneda2021pointwise}. All these works achieve the objective of asymptotic stability replying on the control Lyapunov function that is challenging to design manually. In contrast, we leverage deep RL to synthesize optimal control via interactions, while guaranteeing safe exploration during learning. This problem of safe exploration is sometimes referred to as the safe RL.}

In policy optimization algorithms~\cite{Lillicrap2015, Schulman2015}), RL-agents are free but dangerous to explore any behavior during learning, as long as it leads to performance improvement.
Safe RL is an emerging research field focusing on finding optimal policies that maximize the expected return while ensuring safe exploration i.e., satisfying safety-critical constraints during the learning process~\cite{Garcia2015}.
As for unknown models with noisy measurements, Gaussian Processes (GPs)~\cite{Seeger2008} have been shown to be an efficient data-driven method for online non-parametric model estimation with probabilistic confidence.
The work~\cite{Fisac2018} utilizes GPs for reachability analysis and iterative predictions.  
By integrating GPs and CBF, existing works successfully certify learning-based policies or dynamic systems while ensuring safe exploration~\cite{Wang2018,Cheng2019, castaneda2021pointwise, dhiman2021control, emam2021safe}.
 All of them propose a safety layer to correct the neural network controllers that would cause the system to leave the safe region, where the GP is used for online model identification that is projected into the CBF constraints.
Richard et al.~\cite{Cheng2019} developed a guiding process and guaranteed the performance when applying trust region policy optimization (TRPO).
However, all these works mainly focus on conventional simple objectives, and not complex high-level robotic specifications.
In this framework, we tackle control problems under dynamic uncertainties, achieving temporal logic goals via a safe RL approach.  

\subsection{Contributions}

In this paper, we generalize the deep RL-based approach to LTL specifications over both infinite and finite horizons while maintaining safe exploration.
Our approach is realized in a policy gradient RL setting.
The LTL formulas are converted to a novel automaton, i.e., Embedded LDGBA (E-LDGBA) to synchronously record unvisited accepting sets, which enforce the LTL satisfaction. Based on that,
we develop a dense automaton reward scheme to guide the RL-agent search towards the desired behaviors, based on which we create a modular version of the Deep Deterministic Policy Gradient (DDPG)~\cite{Lillicrap2016} to achieve goals with ideal performance. 
Notably, the modular architecture admits a distributed mechanism
allowing to inspect and improve the performance (success rates) of overall satisfaction compositionally.

Then, a safe exploration module is presented and associated with an online model identification using Bayesian learning, which can be regarded as a model-based "shield" through GPs and ECBF that can handle systems with higher relative degrees and provide guarantees as a bounded confidence evaluation for the estimated model. Since safe exploration will limit the RL-agent's interactions, it reduces the efficiency of finding optimal policies. To address the issue, we design an exploration guiding procedure by integrating the safety violation of the LTL automaton structure and the perturbations of ECBF, to facilitate the efficiency of exploration of deep RL and promote the outputs of the deep RL derived from the set of safety policies.
Finally, we verify that the safety-critical layer and exploration guiding module do not impact the original optimality of learned policies for LTL satisfaction. We provide detailed comparisons with multiple baselines that show the increased performance and safety of our deep RL algorithm. 

Our previous work~\cite{Cai2021modular} focused only on learning a policy satisfying
LTL specifications.
However, it did not consider the safety of the robot during the learning process.
Moreover, it did not strive to optimize exploration efficiency and did not attempt to model learning of system uncertainty.
In this work, we address these issues and propose a new framework that efficiently learns control policies for robots while maintaining safety during training.
From a technical point of view, in this paper, we use GPs and ECBFs in a novel way to achieve safe learning.
In addition, we introduce E-LDGBAs and prove that they are language equivalent to their corresponding LDGBAs, i.e., they capture the same properties.
Lastly, we define safety properties and show how to use the barrier function to enforce them.

\subsection{Related Works}

Recently, there has been increased interest in synthesizing optimal controllers for robotic systems subject to several high-level temporal logics. Mature works \cite{Kloetzer2009,Guo2015,Lahijanian2016,sahin2019multirobot} abstract the discrete interactions between robots and environments subject to Linear Temporal Logic (LTL) for planning, decision-making, and optimization. An actor-critic method has been applied for the optimization of temporal logic motion and control in \cite{wang2015temporal}. However, all of them assume that dynamic systems are known to generate low-level navigation controllers.
Assuming the robotic dynamics to be unknown, learning-based approaches are proposed in \cite{Icarte2018, Camacho2019, aksaray2021probabilistically} by taking deterministic finite automatons (DFAs) as reward machines to guide the LTL satisfaction. Another work \cite{hasanbeig2019reinforcement, bozkurt2020control} develop a reward scheme using limit deterministic (generalized)
B\"uchi automaton (LDBA or LDGBA) \cite{Sickert2016} with probabilistic guarantees of satisfaction. To consider continuous space, previous work \cite{Cai2021modular} proposes a provably-correct framework leveraging deep neural networks, to effectively guide the agent toward
task satisfaction. However, none of the above consider the safety-critical aspects during the learning process.

Considering safe RL-based control synthesis subject to LTL, Li et al. \cite{Li2019} first proposes a safe RL using CBF guided by the robustness of Truncated LTL (TLTL) for satisfaction in dynamic environments.
However, TLTL can only express properties over finite horizons. In contrast, other related works \cite{vasile2020reactive,kantaros2020reactive,kantaros2020stylus,luo2021abstraction,srinivasan2020control} investigate LTL formulas over infinite horizons and continuous space, which generalize the case of finite horizons. One work \cite{schillinger2019hierarchical} has proposed a hierarchical structure for planning and learning over discrete space. But it's not scalable to the state-action space. The works \cite{vasile2020reactive,kantaros2020reactive,kantaros2020stylus,luo2021abstraction} have proposed sampling-based strategies that incrementally build trees to approximate product models. The results~\cite{srinivasan2020control,jagtap2020formal} apply the reachability-based CBFs to derive the robots to satisfy the acceptance sequences. Even though these results are abstraction-free, they assume the robot dynamics are known and can follow the high-level planned paths by designing appropriate low-level controllers. Differently, our work tackles nominal dynamical systems with disturbance
for which synthesis of optimal policies satisfying LTL properties is challenging.

\subsection{Organization}
{\color{black} The abbreviations of notations and definitions are summarized in Table~\ref{tab:Abbreviation}.}
The remainder of the paper is organized as follows. In Section~\ref{sec:preliminaries}, we introduce the modeling framework and formally define the problem. Section~\ref{sec:AUTOMATON} presents a novel design of the automaton structure with its benefits to reformulate the problem. In Section~\ref{sec:RL}, the automatic dense reward scheme is proposed and proved to guide the RL-agent towards satisfaction with maximum probability. A modular deep RL architecture is developed to find approximate optimal policies over continuous space. Based on Section~\ref{sec:RL}, Section~\ref{sec:safe} describes a safety-critical methodology for learning with efficient guiding and safe feasibility enabled, and Section~\ref{sec:Solution} summarizes the overall algorithm. In Section~\ref{sec:experiment}, the experimental results are presented. Section~\ref{sec:conclusion} concludes the paper.

\begin{table}
	\caption{\label{tab:Abbreviation} Abbreviation Summary of Notations.}
	\centering{}\resizebox{0.48\textwidth}{!}{
	\begin{tabular}{|c|c|c|c|c|c|c}
	    \hline
		\textbf{Notation Name} & \textbf{Abbreviation}  \\ 
		\hline
		\text{Linear Temporal Logic} & \text{LTL} \\
		\hline
		\text{Limit-Deterministic Generalized B\"uchi Automaton} & \text{LDGBA} \\
		\hline
		\text{Exponential Control Barrier Function} & \text{ECBF} \\
		\hline
		\text{Reinforcement Learning} & \text{RL} \\
		\hline
		\text{Deep Deterministic Policy Gradient} & \text{DDPG} \\
		\hline
		\text{Gaussian Process} & \text{GP} \\
		\hline
	\end{tabular}}
\end{table}

\section{PRELIMINARIES AND PROBLEM FORMULATION\label{sec:preliminaries}}

\subsection{Markov Decision Processes and Reinforcement Learning \label{subsec:Labeled-MDP}}

The evolution of a dynamic system $\mathcal{S}$ starting from any initial state $s_{0}\in S_{0}$ is given by
\begin{equation}
    \dot{s}=f\left(s\right)+g\left(s\right)a+d\left(s\right),
    \label{eq:dynamics}
\end{equation}
where $s\in S\subseteq \mathbb{R}^{n}$ is the state vector in the compact set $S$ and $a\in A\subseteq \mathbb{R}^{m}$ is the control input. The Lipschitz continuous functions $f:\mathbb{R}^{n}\rightarrow\mathbb{R}^{n}$ and $g:\mathbb{R}^{n}\rightarrow\mathbb{R}^{n\times m}$ define the state dynamics, and $d$ is a deterministic disturbance function
that is locally Lipschitz continuous.
In (\ref{eq:dynamics}), the functions $f$ and $g$ are known, while $d$ is unknown.


\begin{assumption}
\label{assump:kernel}
The unknown function $d$ has low complexity, as measured under the reproducing kernel Hilbert space (RKHS) norm \cite{Paulsen2016}. $d$ has a bounded RKHS norm with respect to known kernel $k$ , that is $\left\Vert d_{j}\right\Vert _{k}<\infty$ for all $j\in\left\{1,\ldots,n\right\}$, where $d_{j}$ represents the $j$-th component of the vector function $d$. 
\end{assumption}
For most kernels used in practice, the RKHS is dense in the space of continuous functions restricted to a compact domain $S$. Thus, Assumption~\ref{assump:kernel} indicates we can uniformly approximate the continuous function $d$ on a compact set $S$ \cite{Seeger2008}.

We capture interactions between a robot's motion governed by the dynamical system $\mathcal{S}$ and the environment as a continuous labeled Markov decision processes (cl-MDP) \cite{Thrun200}.

\begin{defn}
A cl-MDP is a tuple $\mathcal{M}=\left(S, S_{0}, A,p_{S},\AP,L \right)$,
where $S\subseteq\mathbb{R}^{n}$ is a continuous state space, $S_{0}$ is a set of initial states, $A\subseteq\mathbb{R}^{m}$
is a continuous action space, $\AP$ is a set of atomic propositions,
$L:S\shortrightarrow2^{\AP}$ is a labeling function, and $p_{S}$ represents the system dynamics.
The distribution $p_{S}:\mathfrak{B}\left(\mathbb{R}^{n}\right)\times A\times S\shortrightarrow\left[0,1\right]$
is a Borel-measurable conditional transition kernel such that $p_{S}\left(\left.\cdot\right|s,a\right)$
is a probability measure of the next state given current $s\in S$ and $a\in A$ over the Borel
space $\left(\mathbb{R}^{n},\mathfrak{B}\left(\mathbb{R}^{n}\right)\right)$,
where $\mathfrak{B}\left(\mathbb{R}^{n}\right)$ is the set of all
Borel sets on $\mathbb{R}^{n}$. The transition probability $p_{S}$
captures the motion uncertainties of the agent.
\end{defn}

The cl-MDP $\mathcal{M}$ evolves by taking an action $a_{i}$ at each
stage $i$. A control policy is a sequence of decision rules $\boldsymbol{\xi}=\xi_{0}\xi_{1}\ldots$ at each time step, which yields a path $\boldsymbol{s}=s_{0}a_{0}s_{1}a_{1}s_{2}a_{1}\ldots$ over $\mathcal{M}$ such that for each transition $s_{i}\overset{a_{i}}{\rightarrow}s_{i+1}$, $a_{i}$ is generated based on $\xi_{i}$,
where $s_i \overset{a_i}{\to} s_{i+1}$ denoted $p_S(s_{i+1}\mid s_i, a_i) > 0$.

\begin{defn}
A control policy $\boldsymbol{\xi}=\xi_{0}\xi_{1}\ldots$ is stationary, memoryless and deterministic
if $\xi_{t}=\xi, \forall t\geq 0$ and $\boldsymbol{\xi}:S\to {A}$.
A policy is finite-memory
if $\boldsymbol{\xi}$ is a finite state Markov chain.
\end{defn}

A key aspect of cl-MDPs is that optimal policies may be finite-memory
which leads to intractability \cite{Bacchus1996}.
Similarly, stochastic policies are not ideal when considering continuous action-space.
In this work, we use stationary, memoryless, and deterministic policies to efficiently
address synthesis problems.

In the following, we identify the dynamical system in~\eqref{eq:dynamics} with a cl-MDP,
where the state and action spaces are the same.
A policy of the cl-MDP is mapped to a control input for~\eqref{eq:dynamics}
as piecewise-constant functions.
Since $d$ is an unknown function, $p_{S}$ of $\mathcal{M}$ is unknown a priori.
Thus, we must learn desired policies from data.

\begin{assumption}
\label{assump:measurements}
We have access to observations of $s(t)$ and $L(s_{t})$
at every time step $t \in \mathbb{Z}_{\geq 0}$.
\end{assumption}

Given a cl-MDP $\mathcal{M}$, let $\varLambda:S\times A\times S\shortrightarrow\mathbb{R}$ denote
a reward function. Given a discounting function $\gamma:S\times A\times S\shortrightarrow\mathbb{R}$,
the expected discounted return under policy \textbf{$\boldsymbol{\xi}$} starting
from $s\in S$ is defined as

$$U^{\boldsymbol{\xi}}\left(s\right)=\mathbb{E}^{\boldsymbol{\xi}}\left[\stackrel[i=0]{\infty}{\sum}\gamma^{i}\left(s_{i},a_{i},s_{i+1}\right)\cdot\varLambda\left(s_{i},a_{i},s_{i+1}\right)\left|s_{0}=s\right.\right].$$

An optimal policy $\boldsymbol{\xi}^{*}$ maximizes the expected
return i.e., $$\boldsymbol{\xi}^{*}=\underset{\boldsymbol{\xi}}{\arg\max}\, U^{\boldsymbol{\xi}}(s).$$

The function $U^{\boldsymbol{\xi}}\left(s\right)$ is often referred
to as the value function under policy \textbf{$\boldsymbol{\xi}$}.
Without information on $p_{S}$, reinforcement learning (RL) \cite{Watkins1992}
can be employed as a powerful strategy to find the optimal policy.
In this work, we focus on policy gradient methods employing deep neural
networks to parameterize the policy model, due to their excellent performance on control problems over continuous state-space and action-space.
The details are in Section~\ref{sec:RL}.

\subsection{Barrier Functions}
\label{subsec:Barrier}

\begin{defn}
\label{def:safe}
A safe set $\mathcal{C}$ is defined by the super-level set of a continuous differential (barrier) function $h:\mathbb{R}^{n}\rightarrow\mathbb{R}$,
\begin{equation}
\begin{array}{c}
\mathcal{C}=\left\{ s\in\mathbb{R}^{n}:h\left(s\right)\geq0\right\} ,\\
\mathcal{\partial C}=\left\{ s\in\mathbb{R}^{n}:h\left(s\right)=0\right\} ,\\
Int\left(\mathcal{C}\right)=\left\{ s\in\mathbb{R}^{n}:h\left(s\right)>0\right\} 
\end{array}
\label{eq:safe_set}
\end{equation}
where $\mathcal{\partial C}$ is the boundary and $Int\left(\mathcal{C}\right)$ is the interior of the $\mathcal{C}$. The set $\mathcal{C}$ is forward invariant for system~\eqref{eq:dynamics} if $\forall s_{0}\in\mathcal{C}\cap S_{0}$, the condition $s_{t}\in\mathcal{C}$, $\forall t\geq0$, holds.
System~\eqref{eq:dynamics} is safe with respect to the set $\mathcal{C}$ if it is forward invariant. 
\end{defn}

This framework applies barrier functions in the form $h\geq0$ to define safe properties of a given cl-MDP. Consequently, it is assumed that any safety-critical constraint can be represented as~\eqref{eq:safe_set}.
We also assume RL-agents start from initial state $s_{0}\in\mathcal{C}$ such that $h\geq0$, which indicates safety-critical constraints are not violated at the beginning.

\subsection{Linear Temporal Logic}
\label{subsec:LTL}

Linear temporal logic (LTL) is a formal language to describe high-level specifications of a system. The ingredients of an LTL formula are a set of atomic propositions, and combinations of Boolean and temporal operators. The syntax of an LTL formula is defined inductively as~\cite{Baier2008}
\begin{equation*}
        \phi   :=  \text{true} \mid a \mid \phi_1 \land \phi_2 \mid \lnot \phi_1 \mid \Next\phi \mid \phi_1 \UntilOp \phi_2\:, 
\end{equation*}
where $a\in\AP$ is an atomic proposition, $\text{true}$, negation $\lnot$, conjunction $\land$ are propositional logic operators, and next $\next$, until $\UntilOp$ are  temporal operators. 
The semantics of an LTL formula are interpreted over words, which is an
infinite sequence $o=o_{0}o_{1}\ldots$ where $o_{i}\in2^{\AP}$ for
all $i\geq0$, and $2^{\AP}$ represents the power set of $\AP$.
Denote by $o\models\phi$ if the word $o$ satisfies the LTL formula
$\phi$.
For a infinite word $o$ starting from state indexed $0$, let $o(t), t\in\mathbb{N}$ denotes the value at step $t$, and $o[t{:}]$ denotes the word starting from step $t$.
The semantics of LTL satisfaction are defined as~\cite{Baier2008}:
\begin{equation*}
\arraycolsep=1.4pt
\begin{array}{lcl}
o \models \text{true}  \\
o \models \pi  & \Leftrightarrow & \pi\in  o(0)  \\
o \models \phi_{1}\land\phi_{2} &  \Leftrightarrow & o\models \phi_{1} \text{ and } o \models \phi_{2}  \\
o \models \lnot\phi  & \Leftrightarrow & o \not\models \phi  \\
o \models \Next\phi  & \Leftrightarrow & o[1{:}] \models\phi  \\
o \models \phi_1 \UntilOp \phi_2  & \Leftrightarrow & \exists t \text{ s.t. }o[t{:}]\models\phi_{2}, \forall t'\in [0,t),  o[t'{:}]\models\phi_{1}  \\
\end{array} 
\end{equation*}

Alongside the standard operators introduced above, other propositional logic operators such as $\text{false}$, disjunction $\lor$, implication $\rightarrow$, and temporal operators always $\Always$, eventually $\Eventually$ can be derived in LTL.

In this article, we restrict our attention to LTL formulas that exclude the \emph{next} temporal operator, which is not meaningful for continuous time execution~\cite{kloetzer2008fully, luo2021abstraction}.

\subsection{Problem Formulation \label{subsec:problem}}

Given a barrier function $h$ (or mutilple ones), the safe set is defined as $\mathcal{C}=\left\{ s\in\mathbb{R}^{n}:h(s)\geq0\right\}$. We can build a connection between safe sets and LTL formulas as


\begin{defn}
\label{def:safe_task}
Given a cl-MDP and a safe set $\mathcal{C}$, we denote by $\phi_{Safe}$
the safety proposition for the state $s$ of system~\eqref{eq:dynamics}.
Formally, $s \models \phi_{Safe}$ if and only if $s(0) \in \mathcal{C}$.
The safety-critical task is defined as $\Always\phi_{Safe}$ such that
$s(t)\in\mathcal{C},\forall t\in\mathbb{Z}_{\geq 0}$, where $s(t)$ is the state at $t$.
\end{defn}

{\color{black} Consider an RL-agent with dynamics $\mathcal{S}$ and the corresponding safe set $\mathcal{C}$
that performs a mission described by the LTL formula
$\phi=\Always\phi_{safe}\land\phi_{g}$, where $\phi_{safe}$ represents the safety-critical task introduced in definition~\ref{def:safe}, and $\phi_{g}$ denotes a general high-level task in the form of LTL formulas.}
The interaction of the RL-agent with the environment
is modeled by a cl-MDP $\mathcal{M}=\left(S, S_{0}, A,p_{S},\AP,L\right)$.
The induced path under a policy $\boldsymbol{\xi}=\xi_{0}\xi_{1}\ldots$ over $\mathcal{M}$ is $\boldsymbol{s}_{\infty}^{\boldsymbol{\xi}}=s_{0}\ldots s_{i}s_{i+1}\ldots$.
Let $L\left(\boldsymbol{s}_{\infty}^{\boldsymbol{\xi}}\right)=l_{0}l_{1}\ldots$
be the sequence of labels associated with $\boldsymbol{s}_{\infty}^{\boldsymbol{\xi}}$
such that $l_{i}= L(s_{i}), \forall i\in \left\{1,2,\ldots\right\}$. Denote the satisfaction relation of the induced 
trace for $\phi$ by $L(\boldsymbol{s}_{\infty}^{\boldsymbol{\xi}})\models\phi$. The probabilistic satisfaction under the policy $\xi$ from
an initial state $s_{0}\in S_{0}$ is denoted by
\begin{equation}
{\Pr{}_{M}^{\boldsymbol{\xi}}(\phi)=\Pr{}_{M}^{\boldsymbol{\xi}}(L(\boldsymbol{s}_{\infty}^{\boldsymbol{\xi}})\models\phi\,\big|\,\boldsymbol{s}_{\infty}^{\boldsymbol{\xi}}\in\boldsymbol{S}_{\infty}^{\boldsymbol{\boldsymbol{\xi}}}),}\label{eq:probabilistic-satisfaction}
\end{equation}
where $\boldsymbol{S}_{\infty}^{\boldsymbol{\xi}}$ is a set of admissible
paths from the initial state $s_{0}$ under the policy ${\boldsymbol{\xi}}$, and $\Pr{}_{M}^{\boldsymbol{\xi}}(\phi)$ can be computed from \cite{Baier2008}.

\begin{assumption}
\label{assu:task}
It is assumed that there exists at least one policy whose induced traces satisfy the task $\phi$ with non-zero probability. And there are no conflicts between $\phi_{g}$ and $\Always\phi_{safe}$.
\end{assumption} 

Assumption~\ref{assu:task} indicates the existence of policies satisfying $\phi$. We formulate the control learning problem as follows.

\begin{problem}
	\label{Prob1}
	Given  a cl-MDP $\mathcal{M}=(S, S_{0}, A, p_{S}, \AP, L)$ with unknown transition probabilities $p_{S}$, and an LTL task $\phi=\Always\phi_{safe}\land\phi_{g}$ with corresponding safe set $\mathcal{C}$,
\\	
	(i) learn an optimal policy $\boldsymbol{\xi}^{*}$
	that maximizes the satisfaction probability, i.e.,
	$\boldsymbol{\xi}^{*}=\underset{\boldsymbol{\xi}}{\arg\max}\Pr{}_{M}^{\boldsymbol{\xi}}(\phi)$, in the limit,
\\	
	(ii) and maintain the satisfaction of safety-critical task $\Always\phi_{safe}$ during the learning process and policy execution. 
\end{problem}

\section{AUTOMATON SYNTHESIS}
\label{sec:AUTOMATON}

\subsection{E-LDGBA\label{subsec:E-LDGBA}} 

The satisfaction of the LTL formulas can be captured by Limit Deterministic Generalized B\"uchi automata (LDGBA)~\cite{Sickert2016}.
Before defining LDGBA, we first introduce Generalized B\"uchi Automata
(GBA).
\begin{defn}
	\label{def:GBA} A GBA is a tuple $\mathcal{A}=(Q,\Sigma,\delta,q_{0},F)$,
	where $Q$ is a finite set of states; $\Sigma=2^{\AP}$ is a finite
	alphabet, $\delta\colon Q\times\Sigma\shortrightarrow2^{Q}$ is the
	transition function, $q_{0}\in Q$ is an initial state, and $F=\left\{ F_{1},F_{2},\ldots,F_{f}\right\} $
	is a set of accepting sets with $F_{i}\subseteq Q$, $\forall i\in\left\{ 1,\ldots, f\right\} $. 
\end{defn}
Denote by $\boldsymbol{q}=q_{0}q_{1}\ldots$ a run of a GBA, where
$q_{i}\in Q$, $i=0,1,\ldots$. The run $\boldsymbol{q}$ is accepted
by the GBA, if it satisfies the generalized B\"uchi acceptance condition,
i.e., $\inf(\boldsymbol{q})\cap F_{i}\neq\emptyset$, $\forall i\in\left\{ 1,\ldots f\right\} $,
where $\inf(\boldsymbol{q})$ denotes the set of states that repeat infinitely often in $\boldsymbol{q}$.
\begin{defn}
	\label{def:LDGBA} A GBA is an LDGBA if the transition function $\delta$
	is extended to $Q\times(\Sigma\cup\left\{ \epsilon\right\} )\shortrightarrow2^{Q}$,
	and the state set $Q$ is partitioned into a deterministic set $Q_{D}$
	and a non-deterministic set $Q_{N}$, i.e., $Q_{D}\cup Q_{N}=Q$ and
	$Q_{D}\cap Q_{N}=\emptyset$, where 
	\begin{itemize}
		\item the state transitions in $Q_{D}$ are total and restricted within
		it, i.e., $\bigl|\delta(q,\alpha)\bigr|=1$ and $\delta(q,\alpha)\subseteq Q_{D}$
		for every state $q\in Q_{D}$ and $\alpha\in\Sigma$, 
		\item the $\epsilon$-transition is not allowed in the deterministic set,
		i.e., for any $q\in Q_{D}$, $\delta(q,\epsilon)=\emptyset$,
		and 
		\item the accepting sets are only in the deterministic set, i.e., $F_{i}\subseteq Q_{D}$
		for every $F_{i}\in F$. 
	\end{itemize}
\end{defn}
To convert an LTL formula to an LDGBA,
readers are referred to Owl~\cite{Kretinsky2018}.
However, directly using an LDGBA and deterministic policies may fail to satisfy LTL specifications due to its multiple accepting sets, because there do not exist deterministic policies to select several actions for the same state to visit several accepting sets. 
{\color{black} Many advanced deep RL algorithms, e.g., DDPG~\cite{Lillicrap2016}, Distributed Distributional Deterministic Policy Gradients (D$4$PG)~\cite{barth2018distributed},  Twin Delayed DDPG (TD$3$)~\cite{fujimoto2018addressing}, are based on the assumption that there exists at least one deterministic policy to achieve the desired objective. In addition, if such an assumption holds, we are also allowed to apply variants of deep RL using stochastic policies.}
As a result, LDGBA can not be adopted with the DDPG algorithm. To overcome the drawback, we propose E-LDGBA and verify its expressivity as follows.



For an LDGBA $\mathcal{A}$, a tracking-frontier set $T \subseteq F$ is designed
to keep track of unvisited accepting sets.
We initialize $T$ as to include all accepting sets, i.e., $T=F$.
We abuse notation and denote by $F(q) = \{F_j\mid q \in F_j \land F_j \in F\}$
the sets of accepting states that contain state $q$. 
{\color{black}
Since the general LTL task over an infinite horizon can be represented as repetitive patterns, i.e., lasso form~\cite{Baier2008},
we denote the case when all accepting sets have been visited once such that the tracking-frontier set $T$ becomes an empty set, as one round.}
The tracking frontier $(T',\mathcal{B})=f_{V}(q,T)$ is updated as
\begin{equation}\small
f_{V}(q,T)=
\begin{cases}
(T\setminus F(q), \operatorname{false}) & \text{if } F(q) \neq \emptyset \text{ and } F_j \in T\\
(F \setminus F(q), \operatorname{true}) & \text{if } F(q) \neq \emptyset \text{ and } T=\emptyset\\
(T, \operatorname{false}) & \text{otherwise}
\end{cases}
\label{eq:Trk-fontier}
\end{equation}
{\color{black}where the Boolean variable $\mathcal{B}$ indicates the satisfaction of the acceptance condition for each round, i.e., $ T=\emptyset$ if $\mathcal{B}=\operatorname{true}$, otherwise $\mathcal{B}=\operatorname{false}$.}


\begin{defn}[Embedded LDGBA]
	\label{def:E-LDGBA}
Given an LDGBA $\mathcal{A}=\left(Q,\Sigma,\delta,q_{0},F\right)$,
its corresponding E-LDGBA is denoted by $\mathcal{\overline{A}}=(\overline{Q},\Sigma,\overline{\delta},\overline{q_{0}},\overline{F},f_{V})$,
where the tracking frontier is initially set as $T_{0}=F$ s.t. $\overline{q}_{0}=(q_{0},T_{0})$;
$\overline{Q}=Q\times2^{F}$ is the set of augmented states and $2^{F}$ denotes all subsets of $F$, i.e., $\overline{q}=(q,T)$;
the finite alphabet $\Sigma$ is the same as the LDGBA;
the transition function $\overline{\delta}\colon \overline{Q}\times(\Sigma\cup\left\{ \epsilon\right\} )\shortrightarrow2^{\overline{Q}}$
is defined such that $\overline{q}'=\overline{\delta}(\overline{q},\overline{\sigma})$ with $\overline{\sigma}\in(\Sigma\cup\left\{ \epsilon\right\})$, $\overline{q}=(q,T)$ and $\overline{q'}=(q',T)$, if it satisfies
(1) $q'=\delta(q,\overline{\sigma})$, and
(2) $T$ is synchronously updated as $T'=f_{V}\left(q',T\right)$ after transition $\overline{q'}=\overline{\delta}\left(\overline{q},\overline{\sigma}\right)$;
$\overline{F}=\left\{ \overline{F_{1}},\overline{F_{2}}\ldots \overline{F_{f}}\right\}$
with $\overline{F_{j}}=\left\{ (q, T)\in \overline{Q}\bigl|q\in F_{j}\land F_{j} \subseteq T\right\}$
for all $j=1,\ldots f$, is the sets of accepting states. 
\end{defn}
In Def.~\ref{def:E-LDGBA}, the state-space is embedded with the tracking-frontier set $T$ that can be practically represented via one-hot encoding based on the indices of accepting sets, and is synchronously updated after each transition.
Once an accepting set $F_{j}$ is visited, it is removed from $T$,
and if $T$ is empty, it resets to $F\setminus F_{j}$.
The accepting states of E-LDGBA ensure that progress is made towards each accepting set
of the LDGBA before considering the next ones.
The novel design ensures all accepting sets of original LDGBA are
visited in each round under deterministic policies.

For an LTL formula $\phi$, let $\mathcal{\overline{A}}_{\phi}$ and $\mathcal{A}_{\phi}$ be
the corresponding E-LDGBA and LDGBA, respectively.
Let $\mathcal{L}(\mathcal{A}_{\phi})\subseteq \Sigma^{\omega}$ and $\mathcal{L}(\mathcal{\overline{A}}_{\phi})\subseteq \Sigma^{\omega}$ be the accepted languages of $\mathcal{A}_{\phi}$ and $\mathcal{\overline{A}}_{\phi}$, respectively, over the same alphabet $\Sigma$.
Based on~\cite{Baier2008},
$\mathcal{L}(\mathcal{A}_{\phi})$ is the set of all infinite words that satisfy LTL formula $\phi$.

\begin{lem}
\label{lem:language}
For any LTL formula $\phi$, we can construct LDGBA $\mathcal{A}_{\phi}=(Q,\Sigma,\delta,q_{0},F)$ and E-LDGBA $\mathcal{\overline{A}}_{\phi}=(\overline{Q},\Sigma,\overline{\delta},\overline{q_{0}},\overline{F},f_{V})$. It holds that
$\mathcal{L}(\mathcal{\overline{A}}_{\phi})=\mathcal{L}(\mathcal{A}_{\phi})$.
\end{lem}
\begin{proof}
Proof can be found in Appendix~\ref{Appendix:E-LDGBA}.
\end{proof}

Lemma~\ref{lem:language} illustrates that both E-LDGBA and LDGBA accept the same language. Consequently, E-LDGBA can be used to ensure the satisfaction of LTL specifications.

\begin{defn}
\label{def:Non-accepting Sink Component} A non-accepting sink component
$\overline{Q}_{sink}\subseteq \overline{Q}$ of an E-LDGBA is a strongly connected
directed graph with no outgoing transitions s.t. the acceptance condition
can not be satisfied if starting from any state in $\overline{Q}_{sink}$. We denote the union of all non-accepting sink components as $\overline{Q}_{sinks}$. Thus, a trace reaching them is doomed to not satisfy the given LTL property.
\end{defn}


\begin{defn}
\label{def:Non-accepting unsafe sink} Given an LTL formula $\phi=\oblong\phi_{safe}\land\phi_{g}$, a set of non-accepting unsafe states
$\overline{Q}_{unsafe}\subseteq \overline{Q}_{sinks}$ of an E-LDGBA is a set of sink states s.t. $\overline{Q}_{unsafe}=\left\{(\overline{q}'\in \overline{Q})\bigl|\forall\overline{q}\in \overline{Q},\overline{q'}=\overline{\delta}(\overline{q},\lnot\phi_{safe})\right\}$.
\end{defn}
The automaton system enters into $\overline{Q}_{unsafe}$ whenever $\Always\phi_{safe}$ is violated, which means $\phi=\oblong\phi_{safe}\land\phi_{g}$ can not be satisfied anymore.
The set $\overline{Q}_{unsafe}$ can be used as an indicator of unsafety during learning process.

\subsection{Product MDP}
\label{subsec:EP-MDP}

To satisfy a complex LTL-defined task over infinite horizons, we can define a product structure.

\begin{defn}
\label{def:P-MDP}
Given a cl-MDP $\mathcal{M}$ and an E-LDGBA $\mathcal{\overline{A}}_{\phi}$,
the Product MDP (P-MDP) is defined as $\mathcal{P}=\mathcal{M}\times\mathcal{\overline{A}}_{\phi}=(X,U^{\mathcal{P}},p^{\mathcal{P}},x_{0},F^{\mathcal{P}},f_{V})$,
where $X=S\times Q\times2^{F}$ is the set of product states, i.e.,
$x=(s,\overline{q})=(s,q,T)\in X$; $U^{\mathcal{P}}=A\cup\left\{ \epsilon\right\} $
is the set of actions, where the $\epsilon$-actions are only allowed
for transitions from $Q_{N}$ to $Q_{D}$; $x_{0}=(s_{0},\overline{q}_{0})$
is the initial state; $F^{\mathcal{P}}=\left\{ F_{1}^{\mathcal{P}},F_{2}^{\mathcal{P}}\ldots F_{f}^{\mathcal{P}}\right\} $
where $F_{j}^{\mathcal{P}}=\left\{(s,\overline{q})\in X\mid\overline{q}\in \overline{F_{j}},j\in\left\{1,\ldots, f\right\}\right\}$; $p^{\mathcal{P}}$
is the transition kernel for any transition $\left(x,u^{\mathcal{P}},x'\right)$
with $x=(s,\overline{q})$ and $x'=(s',\overline{q}')$ such that
: (1) $p^{\mathcal{P}}(x,u^{\mathcal{P}},x')=p_{S}(\left.s'\right|s,a)$
if $s^{\prime}\backsim p_{S}(\left.\cdot\right|s,a)$,
$\overline{\delta}(\overline{q},L(s))=\overline{q}^{\prime}$ where
$u^{\mathcal{P}}=a\in A$ \; (2) $p^{\mathcal{P}}(x,u^{\mathcal{P}},x')=1$
if $\ensuremath{u^{\mathcal{P}}\in\left\{ \epsilon\right\} }$, $\overline{q}^{\prime}\in{\delta}(\overline{q},\epsilon)$
and $s'=s$; and (3) $p^{\mathcal{P}}(x,u^{\mathcal{P}},x')=0$
otherwise.
\end{defn}
The P-MDP $\mathcal{P}$ captures the identification of admissible agent motions over $\mathcal{M}$ that satisfy
the task ${\phi}$. Let $\boldsymbol{\pi}$ denote a policy over $\mathcal{P}$
and denote by $\boldsymbol{x}_{\infty}^{\boldsymbol{\pi}}=x_{0}\ldots x_{i}x_{i+1}\ldots$
an infinite path generated by $\boldsymbol{\pi}$.
Any memory-less policy $\boldsymbol{\pi}$ of $\mathcal{P}$ can be projected onto $\mathcal{M}$ to obtain a finite-memory policy $\boldsymbol{\xi}$~\cite{Baier2008}.

A path $\boldsymbol{x}_{\infty}^{\boldsymbol{\pi}}$
satisfies the acceptance condition if $\inf\left(\boldsymbol{x}_{\infty}^{\boldsymbol{\pi}}\right)\cap F_{i}^{\mathcal{P}}\neq\emptyset$
, $\forall i\in\left\{ 1,\ldots f\right\}$, which can be denoted as $\boldsymbol{x}_{\infty}^{\boldsymbol{\pi}}\models\Acc_{p}$.
An accepting path satisfies the LTL task $\phi$. We denote $\Pr^{\mathbf{\boldsymbol{\pi}}}\left[x_{0}\models\Acc_{p}\right]$
as the probability of satisfying the acceptance condition of $\mathcal{P}$
under policy $\boldsymbol{\pi}$ starting from initial state $x_{0}$,  
and denote $\Pr_{max}\left[x_{0}\models\Acc_{p}\right]=\underset{\boldsymbol{\pi}}{\max}\Pr_{M}^{\boldsymbol{\pi}}\left[x_{0}\models\Acc_{p}\right]$. 
In Problem~\ref{Prob1}, finding a policy $\boldsymbol{\xi}$ of $\mathcal{M}$ to satisfy $\phi$ is equivalent to searching for a policy $\boldsymbol{\pi}$ of $\mathcal{P}$ to satisfy the acceptance condition. The properties of P-MDP related to rigorous analysis of optimality can be found in Appendix~\ref{subsec:properties}.

\begin{rem}
Explicitly constructing the P-MDP is impossible over continuous space. In this work, we generate P-MDP on-the-fly which means the approach tracks the states of an underlying structure based on Def.~\ref{def:P-MDP}.
\end{rem}

To monitor the safety-critical requirement via automaton structure, 
we define the sink components of violating $\Always\phi_{safe}$ in an P-MDP as:
\begin{defn}
\label{def:unsafe_EP-states}
Given an P-MDP $\mathcal{P}=\mathcal{M}\times\mathcal{{A}}_{\phi}$
, the non-accepting unsafe sink component can be defined as: $X_{unsafe}\subseteq X$ s.t. $X_{unsafe}=\left\{
(s,\overline{q})\in X)\mid\overline{q}\in\overline{Q}_{unsafe} \right\}$.
\end{defn}

Based on Def.~\ref{def:Non-accepting unsafe sink}, if the system enters $X_{unsafe}$, it implies the violation of the safety constraint over $\mathcal{P}$. Thus, Problem~\ref{Prob1} can be reformulated as:
\begin{problem}
\label{Prob:2} Given a user-specified LTL task $\phi=\oblong\phi_{safe}\land\phi_{g}$ and a cl-MDP with unknown transition probability, the goal consists of two parts:

(i). Find a policy $\boldsymbol{\boldsymbol{\pi}}^{*}$ satisfying the
acceptance condition of $\mathcal{P}$ with maximum probability in the limit,
i.e., $\Pr^{\boldsymbol{\pi}^{*}}\left[\boldsymbol{x}\models\Acc_{p}\right]=\Pr_{max}\left[\boldsymbol{x}\models\Acc_{p}\right]$;

(ii). Avoid entering $X_{unsafe}$ during the learning process.
\end{problem}
Section~\ref{sec:RL} constructs a modular RL architecture to generate RL controllers for solving part (i) of Problem~\ref{Prob:2}. Section~\ref{sec:safe} presents an approach to fulfill requirement (ii) of Problem~\ref{Prob:2} by proposing a GP-based ECBF compensators for the RL controllers to ensure safety during training.

\section{LEARNING-BASED CONTROL}
\label{sec:RL}
First, we briefly introduce a reward-based scheme and a reward shaping procedure to improve the reward density in Section~\ref{subsec:RL-reward}. Then, Section~\ref{subsec:RL} shows how to apply the shaped reward to construct a modular deep RL architecture based on E-LDGBA to solve part (i) of Problem~\ref{Prob:2}.

\subsection{Dense Reward Scheme\label{subsec:RL-reward}}

Let $F_{U}^{\mathcal{P}}$ denote the union of accepting states, i.e.,
$F_{U}^{\mathcal{P}}=\left\{ x\in X \mid x\in F_{i}^{\mathcal{P}},\forall i\in\left\{ 1,\ldots f\right\}\right\} $. For each transition $\left(x,u^{\mathcal{P}},x'\right)$ in the P-MDP, the reward and discounting function only depend on current state $x$, i.e., $R\left(x,u^{\mathcal{P}},x'\right)=R\left(x\right)$ and $\gamma\left(x,u^{\mathcal{P}},x'\right)=\gamma\left(x\right)$.

We apply the reward function
\begin{equation}
R\left(x\right)=\left\{ \begin{array}{cc}
1-r_{F}, & \text{if }x\in F_{U}^{\mathcal{P}},\\
0, & \text{otherwise,}
\end{array}\right.\label{eq:reward_function}
\end{equation}
and the discounting function
\begin{equation}
\gamma\left(x\right)=\left\{ \begin{array}{cc}
r_{F}, & \text{if }x\in F_{U}^{\mathcal{P}},\\
\gamma_{F}, & \text{otherwise,}
\end{array}\right.\label{eq:discount_function}
\end{equation}
where $r_{F}\left(\gamma_{F}\right)$ is a function of
$\gamma_{F}$ satisfying $\underset{\gamma_{F}\shortrightarrow1^{-}}{\lim}r_{F}\left(\gamma_{F}\right)=1$
and $\underset{\gamma_{F}\shortrightarrow1^{-}}{\lim}\frac{1-\gamma_{F}}{1-r_{F}\left(\gamma_{F}\right)}=0$. Let $U^{\boldsymbol{\pi}}\left(x\right)$ denotes the expected return by applying the reward function~\eqref{eq:reward_function} and discount function ~\eqref{eq:discount_function}. We have

\begin{thm}
\label{thm:conclusion} \cite{Cai2021modular} Given an P-MDP $\mathcal{P}$, by selecting $\gamma_{F}\shortrightarrow1^{-}$,
the optimal policy $\boldsymbol{\pi}^{*}$ that maximizes the $U^{\boldsymbol{\pi}}\left(x\right)$ also maximizes
the probability of satisfying the acceptance condition, i.e., $\Pr^{\boldsymbol{\pi}^{*}}\left[x_{0}\models Acc_{\mathcal{P}}\right]=\Pr_{max}\left[x_{0}\models Acc_{\mathcal{P}}\right]$. 
\end{thm}


In the above design, the reward signal becomes sparse for the state $x\notin F_{U}^{\mathcal{P}}$. To further increase the density of the reward, we apply
a potential function $\Phi:X\shortrightarrow\mathbb{R}$, and transform the reward as:
\begin{equation}
R'\left(x,u^{\mathcal{P}},x'\right)=R\left(x\right)+\gamma\left(x\right)\cdot\Phi\left(x'\right)-\Phi\left(x\right)\label{eq:Shaped_Reward}
\end{equation}

Given $\mathcal{{P}=M}\times\mathcal{\overline{A}}_{\phi}={\left(X,U^{\mathcal{P}},p^{\mathcal{P}},x_{0},F^{\mathcal{P}},f_{V}\right)}$ and the corresponding LDGBA $\mathcal{A}_{\phi}$,
let $F_{U}=\left\{ q\in Q \mid q\in F_{i},\forall i\in\left\{ 1,\ldots f\right\} \right\} $.
For the states $x=(s,q,T)$ of $\mathcal{P}$ whose automaton states $q$ belong to $Q\setminus\left(F_{U}\cup q_{0}\cup Q_{sink}\right)$,
where $Q_{sink}$ is the sink component of $\mathcal{A}_{\phi}$, it is desirable to assign positive rewards when the agent first visits them and assign large value of reward to the accepting states to enhance the convergence of neural networks, see Section~\ref{subsec:RL}.
Starting from the initial automaton state, exploring any automaton state in  $Q\setminus\left(F_{U}\cup q_{0}\cup Q_{sink}\right)$ can enhance the guiding of task satisfaction.
To this end, a reward tracking-frontier set $T_{\Phi}$ is designed to keep
track of unvisited automaton components $Q\setminus\left(q_{0}\cup Q_{sink}\right)$.
$T_{\Phi}$ is initialized as $T_{\Phi0}=Q\setminus\left(q_{0}\cup Q_{sink}\right)$.
The set $T_{\Phi0}$ is then updated after each transition $\left(\left(s,q,T\right),u^{\mathcal{P}},\left(s',q',T\right)\right)$
of $\mathcal{P}$
\begin{equation}
f_{\Phi}\left(q',T_{\Phi}\right)=\left\{ \begin{array}{cc}
T_{\Phi}\setminus F(q'), & \text{if }q\in T_{\Phi},\\
T_{\Phi0}\setminus F(q') & \text{if }\ensuremath{\mathcal{B}}=\operatorname{True},\\
T_{\Phi}, & \text{otherwise. }
\end{array}\right.\label{eq:automaton-fontier}
\end{equation}
where $F(q')$ is the same as~(\ref{eq:Trk-fontier}).
The set $T_{\Phi}$ will only be reset when $\mathcal{B}$ in $f_{V}$ becomes $\operatorname{True}$,
indicating that all accepting sets in the current round have been visited.
Then the potential function $\Phi\left(x\right)$
for $x=\left(s,q,T\right)$ is constructed as:
\begin{equation}
\Phi\left(x\right)=\left\{ \begin{array}{cc}
\eta_{\Phi}\cdot(1-r_{F}), & \text{if }q\in T_{\Phi},\\
0, & \text{otherwise}
\end{array}\right.\label{eq:potential_function}
\end{equation}
where $\eta_{\Phi}> 0$ is the shaping parameter. Intuitively, the value of potential function for
unvisited and visited states in $T_{\Phi0}$ is equal to $\eta_{\Phi}\cdot(1-r_{F})$
and $0$ respectively, which improves the efficiency of exploration. 

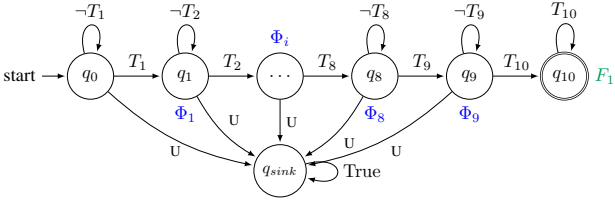
\begin{figure}[!t]\centering
	{{
	    \scalebox{.7}{
			\begin{tikzpicture}[shorten >=1pt,node distance=1.8cm,on grid,auto] 
			\node[state,initial] (q_0)   {$q_0$}; 
			\node[state] (q_1) [right=of q_0, label=below:$\textcolor{blue}{\Phi_1}$]  {$q_1$};
			
			\node[state, label=above:$\textcolor{blue}{\Phi_i}$] (q_2) [right=of q_1]  {$\ldots$};
			\node[state, label=below:$\textcolor{blue}{\Phi_8}$] (q_3) [right=of q_2]  {$q_8$};
			
			\node[state, label=below:$\textcolor{blue}{\Phi_9}$] (q_6) [right=of q_3]  {$q_9$};
			
			\node[state,accepting, label=right:$\textcolor{ForestGreen}{F_1}$] (q_5) [right=of q_6]  {$q_{10}$};	
			
			\node[state] (q_4) [below=of q_2]  {$\Scale[0.9]{q_{sink}}$};
			\path[->] 
			(q_0) edge [bend left=0] node {$T_1$} (q_1)
			(q_0) edge [loop above] node {$\lnot T_1$} (q_0)
			(q_0) edge [bend right=15 ] node [below] {\Scale[0.8]{$U$}} (q_4)
			(q_1) edge [bend left=0] node {$T_2$} (q_2)
			(q_1) edge [loop above] node {$\lnot T_2$} (q_1)
			(q_1) edge [bend right=15] node {\Scale[0.8]{$U$}} (q_4)
			(q_2) edge [bend left=0] node {$T_8$} (q_3)
			(q_3) edge [bend left=15] node {\Scale[0.8]{$U$}} (q_4)
			(q_2) edge [bend left=0] node {\Scale[0.8]{$U$}} (q_4)
		    (q_3) edge [bend left=0] node {$T_{9}$} (q_6)
		    (q_6) edge [bend left=0] node {$T_{10}$} (q_5)
		    (q_6) edge [bend left=15] node {\Scale[0.8]{$U$}} (q_4)
		    (q_5) edge [loop above] node {$T_{10}$} (q_5)
			(q_3) edge [loop above] node {$\lnot T_8$} (q_3)
			(q_6) edge [loop above] node {$\lnot T_9$} (q_6)
			(q_4) edge [loop right] node {$\operatorname{True}$} (q_4);
			\end{tikzpicture}
			}
			}}
		\caption{\label{fig:reward_shaping} LDGBA $\mathcal{A}_{\varphi_{P}}$ has states from $q_{0}$ to $q_{10}$ and a sink state expressing the LTL formula $\varphi_{P}=\lozenge\left(\mathtt{T1}\land\lozenge\mathtt{\left(\mathtt{T2}\land\lozenge\mathtt{\ldots\land\lozenge{T}10}\right)}\right)\land\lnot\oblong\mathtt{U}$, which requires visiting the regions labeled from $\mathtt{T1}$ to $\mathtt{T10}$ sequentially while always avoiding unsafe regions labeled $\mathtt{U}$.}
	\end{figure}

\begin{figure*}
	\centering{}\includegraphics[scale=0.50]{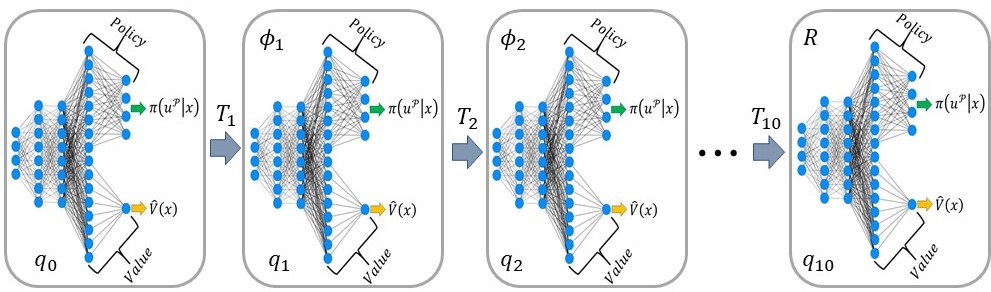}\caption{\label{fig:modular} Example of modular architecture based on reward shaping for the LTL formula $\varphi_{P}$. The distributed actor-critic neural networks are constructed based on the automata components and are learned synchronously online.}
\end{figure*}

\begin{example}
\label{example: reward_shaping}
As a running example of the reward shaping technique. Fig.~\ref{fig:reward_shaping} shows an LDGBA of the LTL formula $\varphi_{P}=\lozenge\left(\mathtt{T1}\land\lozenge\mathtt{\left(\mathtt{T2}\land\lozenge\mathtt{\ldots\land\lozenge{T}10}\right)}\right)\land\lnot\oblong\mathtt{U}$ with only one accepting set. Let's denote any
state of P-MDP with the same automaton component and an arbitrary MDP state as
$x=\left(\left\llbracket s\right\rrbracket ,q,T\right)$, where $\llbracket s\rrbracket$ denotes a subset of the MDP's state space, i.e., the MDP component can be different. For a trajectory $\boldsymbol{x}=\left(\left\llbracket s\right\rrbracket ,q_{0},T\right)u_{0}^{\mathcal{P}}\left(\left\llbracket s\right\rrbracket ,q_{1},T\right)u_{1}^{\mathcal{P}}\left(\left\llbracket s\right\rrbracket ,q_{2},T\right)u_{2}^{\mathcal{P}}\left(\left\llbracket s\right\rrbracket ,q_{3},T\right),$
the associated shaped reward for each transition is equal to $\eta_{\Phi}\cdot(1-r_{F})$, instead of zero.
\end{example}


Given a path $\boldsymbol{x}_{t}=x_{t}x_{t+1}\ldots$ starting from
$x_{t}$ associated with the corresponding action sequence $\boldsymbol{u}_{t}^{\mathcal{P}}=u_{t}^{\mathcal{P}}u_{t+1}^{\mathcal{P}}\ldots$, the return is reformulated by applying (\ref{eq:Shaped_Reward})
\begin{equation}
\mathcal{D}'\left(\boldsymbol{x}_{t}\right)\coloneqq\stackrel[i=0]{\infty}{\sum} R'\left(\boldsymbol{x}_{t+i},\boldsymbol{u}_{t+i}^{\mathcal{P}},\boldsymbol{x}_{t+i+1}\right)\cdotp\stackrel[j=0]{i-1}{\prod}\gamma'\left(\boldsymbol{x}_{t+j}\right)\label{eq:DisctRetrn_shaped}
\end{equation}

Also, the shaped expected return of any state
$x\in X$ under policy $\boldsymbol{\pi}$ is
\\
\begin{equation}
U'^{\boldsymbol{\pi}}\left(x\right)=\mathbb{E}^{\boldsymbol{\pi}}\left[\mathcal{D}'\left(\boldsymbol{x}_{t}\right)\mid x_{t}=x\right].\label{eq:ExpRetrn_shaped}
\end{equation}
\\

\begin{proposition}
    \label{prop:shaped_probability} Given a cl-MDP $\mathcal{M}$ and an E-LDGBA
	$\mathcal{\overline{A}}_{\phi}$, by selecting $\gamma_{F}\shortrightarrow1^{-}$,
	the optimal policy $\boldsymbol{\pi}^{*}$ that maximizes the expected
	return in (\ref{eq:ExpRetrn_shaped}) by applying the shaped reward (\ref{eq:Shaped_Reward}) in the corresponding P-MDP also
	maximizes the probability of satisfying $\phi$, i.e., $\Pr^{\boldsymbol{\pi}^{*}}\left[x_{0}\models Acc_{\mathcal{P}}\right]=\Pr_{max}\left[x_{0}\models Acc_{\mathcal{P}}\right]$.
\end{proposition}
\begin{proof}
 The work of \cite{Ng1999} has shown that optimizing the expected return \eqref{eq:ExpRetrn_shaped} using the shaped reward~\eqref{eq:Shaped_Reward} is equivalent to optimizing $U^{\boldsymbol{\pi}}\left(x\right)$ with the base reward scheme, and the generated optimal policies from these two forms are the same. Thus, the result follow from Theorem~\ref{thm:conclusion},.
\end{proof}

Next, we use the reward design in a deep RL algorithm to find optimal policies over continuous spaces.

\subsection{Modular Deep Reinforcement Learning}
\label{subsec:RL}

The objective of policy-based RL is to find the optimal policy that maximizes the long-term expected return~\eqref{eq:ExpRetrn_shaped}.
To address the MDPs with continuous state-action space, we implement the DDPG algorithm~\cite{Lillicrap2016}. 
The aim of DDPG is to approximate the
current deterministic policy via a parameterized function $\boldsymbol{\pi}\left(x\left|\theta^{u}\right.\right)$ called actor. The actor is a deep neural network whose weights are $\theta^{u}$. The critic function uses a deep neural network with parameters $\theta^{Q}$ to approximate the action-value function $Q\left(x,u^{\mathcal{P}}\left|\theta^{Q}\right.\right)$, which is updated by minimizing the loss function:
\begin{equation}
\begin{array}{c}
L\left(\theta^{Q}\right)=\mathbb{E}_{x\sim\rho^{\beta}}^{\boldsymbol{\pi}}\left[\left(y-Q\left(x,\boldsymbol{\pi}\left(x\left|\theta^{u}\right.\right)\left|\theta^{\boldsymbol{Q}}\right.\right)\right)^{2}\right],
\end{array}\label{eq:loss_function}
\end{equation}
where $\rho^{\beta}$ is the probability distribution of state visits over $X$ under any arbitrary stochastic policy $\beta$.
The actor is updated by applying the chain rule to the expected
return with respect to actor parameters $\theta^{u}$ as the following
policy gradient procedure
%
\begin{equation}
\begin{array}{c}
\nabla_{\theta^{u}}U^{u}\left(x\right)\thickapprox\mathbb{E}_{s\sim\rho^{\beta}}^{\boldsymbol{\pi}}\left[\nabla_{\theta^{u}}Q\left(x,\boldsymbol{\pi}\left(x\left|\theta^{u}\right.\right)\left|\theta^{Q}\right.\right)\right]\\
=\mathbb{E}_{s\sim\rho^{\beta}}^{\boldsymbol{\pi}}\left[\nabla_{u^{\mathcal{P}}}Q\left(x,u^{\mathcal{P}}\left|\theta^{Q}\right.\right)\left|_{u^{\mathcal{P}}=\boldsymbol{\pi}\left(x\left|\theta^{u}\right.\right)}\nabla_{\theta^{u}}\boldsymbol{\pi}\left(x\left|\theta^{u}\right.\right)\right.\right].
\end{array}\label{eq:actor_update}
\end{equation}

Based on the reward shaping, this framework applies a modular DDPG architecture to reduce the global variance of the policy gradient algorithm and improve the performance of satisfying complex tasks. The intuitive idea is to divide the LTL task into several sub-tasks based on its automaton structure and apply several DDPG for each sub-task.

The complex
LTL task $\phi$ is divided into simple composable modules. Each
state of the automaton in the LDGBA is a module
and each transition between these automaton states is a "task divider". In
particular, given $\phi$ and its LDGBA $\mathcal{{A}}_{\phi}$,
we propose a modular architecture of $\left|Q\right|$ DDPG, i.e., $\boldsymbol{\pi}_{q_{i}}\left(x\left|\theta^{u}\right.\right)$
and $Q_{q_{i}}\left(x,u^{\mathcal{P}}\left|\theta^{Q}\right.\right)$
with $q_{i}\in Q$, {\color{black} along with their own replay buffer defined as a data memory to save the RL-agent's transitions.} Experience
samples are stored in each replay modular buffer $B_{q_{i}}$ in the
form of $\left(x,u^{\mathcal{P}},R\left(x\right),\gamma\left(x\right),x'\right)$.
By dividing the LTL task into sub-stages, the set of neural nets acts
in a global modular DDPG architecture, which allows the
agent to jump from one module to another by switching between the
set of neural nets based on transition relations of $\mathcal{{A}}_{\phi}$.

\begin{example}
\label{example: modular}
Continuing with example~\ref{example: reward_shaping}, Fig.~\ref{fig:modular} shows the modular DDPG architecture corresponding to the LTL formula $\varphi_{P}$  based on the reward shaping scheme, where each pair of actor-critic neural networks represents the standard DDPG structure along with an automaton state, and the transitions of them are consistent with the LDGBA structure. 
\end{example}

\begin{rem}
In the modular architecture, instead of dividing complex tasks by the states of E-LDGBA that have more automaton states $\overline{Q}$ due to the embedded frontier set,  we decompose the overall task based on the more compact set of states $Q$ of LDGBA, which sparsifies the number of distributed actor-critic pairs. This design reduces the memory complexity and achieves the same objective. 
\end{rem}

\section{SAFE LEARNING AND EXPLORATION GUIDING\label{sec:safe}}

This section focuses on adding the safety guard as a "shield" during the learning process. First, Section~\ref{subsec:GPs} introduces GPs to approximate the unknown model in (\ref{eq:dynamics}). Then Section~\ref{subsec:CBF} provides a continuous form of exponential control barrier functions (ECBF) for higher relative degrees, which can be incorporated with GPs to safeguard with a bounded probability.
Section~\ref{subsec:CBF_Learning} integrates the GP-based ECBF compensators for control laws generated from the modular RL policy in Section~\ref{subsec:RL} to ensure safety-critical requirements during the learning process. To improve the efficiency of exploration and preserve the original formal optimality, Section~\ref{subsec:guiding} proposes an automaton-based guiding strategy that enhances RL policies being explored within the set of safe policies during training.

\subsection{Gaussian Processes\label{subsec:GPs}}
Gaussian Processes (GP) are non-parametric regression methods to approximate the unknown system dynamics and their uncertainties from data \cite{GPs2006}.
We use GP regression to identify the unknown disturbance function of a nonlinear map $d(s): S\rightarrow\mathbb{R}^{n}$ in~\eqref{eq:dynamics}. The main advantage of applying GPs compared with the feedforward neural network of nonlinear regressions is the quantifiable confidence of predictions.
Informally, a GP is a distribution over functions, and each component $d_{i}(s)$ of $n$-dimensional $d(s)$ can be approximated by a GP distribution denoted as $\mathcal{GP}$ with a mean function $u_{i}(s)$ and a covariance kernel function $k_{i}(s,s')$ which measures similarity between any two states, i.e., $d_{i}(s)\sim\mathcal{GP}(u_{i}(s),k_{i}(s,s'))$.
The class of the prior mean function and covariance kernel function is selected to characterize the model. The approximation of $d(s)$ with $n$ independent GPs is
$$\hat{d}(x)=\left\{ \begin{array}{c}
\hat{d}_{1}(s)\sim\mathcal{GP}(u_{i}(s),k_{i}(s,s'))\\
\vdots\\
\hat{d}_{n}(s)\sim\mathcal{GP}(u_{i}(s),k_{i}(s,s'))
\end{array}\right.$$

Based on Assumption~\ref{assump:measurements}, given a set of $N_{m}$ input data $\left\{(s^{(1)},a^{(1)}),\ldots,(s^{(N_{m})},a^{(N_{m})})\right\}$, and corresponding measurements $\left\{(y^{(1)},\ldots,y^{(N_{m})}\right\}$ subject to additive Gaussian noise $w\sim\mathcal{N}(\boldsymbol{0}_{n},\sigma^{2}_{noise}\boldsymbol{I}_{n})$,  where $y^{(i)}=f(s^{(i)})+g(s^{(i)})a^{(i)}+d(s^{(i)})+\omega^{(i)}$, $\forall i\in\left\{(1,\ldots,(N_{m}\right\}$,  the mean $\hat{u}_{i}$ and covariance $\hat{\sigma}_{i}$ of posterior distribution for $d_{i},i\in\left\{1,\ldots,n\right\}$ at an arbitrary query state $s_{*}\in S$ is

\begin{equation}
\begin{array}{cc}
\hat{u}_{i}(s_{*})=\hat{k}^{T}_{i}(K_{i}+\sigma^{2}_{noise}\boldsymbol{I}_{N_{m}})^{-1}y_{i}\\
\hat{\sigma}^{2}_{i}(s_{*})=k_{i}((s_{*},(s_{*})-\hat{k}^{T}_{j}(K_{j}+\sigma^{2}_{noise}\boldsymbol{I}_{N_{m}})^{-1}\hat{k}_{i}\label{eq:GP_estimation},
\end{array}
\end{equation}
where $\hat{k}_{i}=[k_{i}(s^{(1)},s_{*}),\ldots,k_{i}(s^{(N_{m})},s_{*})]^{T}$, $y_{i}=[(y^{(1)}_{i},\ldots,y^{(N_{m})}_{i}]$, and $K_{i}\in\mathbb{R}^{N_{m}\times N_{m}}$ is a kernel matrix s.t. $K^{(jl)}_{i}=k_{i}(s^{(j)},s^{(l)})$ with $j,l\in\left\{1,\ldots,N_{m}\right\}$.

{\color{black}
\begin{proposition}
\label{prop:GP-estimation}
    Consider a system $\mathcal{S}$ in~\eqref{eq:dynamics} with Assumption~\ref{assump:kernel}, the uncertain dynamics $d(s)$ is estimated over a multivariate GP with mean $u(s)=[\hat{u}_{i}(s),\ldots,\hat{u}_{n}(s)]^{T}$ and standard deviation $\sigma(s)=[\hat{\sigma}_{i}(s),\ldots,\hat{\sigma}_{n}(s)]^{T}$ via~\eqref{eq:GP_estimation}. Then, the model estimation error is bounded with probability $(1-\delta)^{n}$ as: 
    \begin{equation}
    \Pr\left\{|d(s) - u(s) | \leq k_{\delta} \cdot \sigma(s)\right\}\geq(1-\delta)^{n},\label{eq:GP}
    \end{equation}
    where  $\Pr\left\{\cdot\right\}$ represents the measurement probability, and $k_{\delta}$ is a design a parameter determining $\delta$.
\end{proposition}
The proposition can be proved by extending the scalar inequality in the work~\cite{srinivas2012information} to the $n-$dimensional state-set.}
Note that applying GPs for large amounts of training data is intractable and problematic due to the expensive matrix computation in \eqref{eq:GP}. However, we alleviate this issue via the episodic sampling method \cite{Cheng2019}. Any other methods for model estimation can be also used with our framework. Note that similar as~\cite{emam2021safe},
$d$ can be straightforwardly extended to be stochastic (with additive Gaussian noises) since GPs are used to learn the
disturbance.

\subsection{Probabilistic Exponential Control Barrier Function\label{subsec:CBF}} 

For continuous nonlinear systems, control barrier function (CBF) is an efficient tool for maintaining safety \cite{Ames2016}. First, we recount the definition of first order CBFs.

\begin{defn}
	\label{def:CBF}
Consider the system $\mathcal{S}$ in (\ref{eq:dynamics}) and assume $d(x)$ is known, and the safe set $\mathcal{C}\subseteq S$ with a continuous differentiable (barrier) function $h:S\rightarrow\mathbb{R}$ in (\ref{eq:safe_set}). If $\frac{\partial h}{\partial s}\neq 0$ for all $s\in\mathcal{\partial C}$ and there exists an extended $\mathcal{K}$ function $\alpha$ s.t.
\begin{equation}
L_{f}h(s)+L_{g}h(s)a+L_{d}h(s)\geq-\alpha(h\left(s\right)),\label{eq:CBF}
\end{equation}
then for a trajectory $\boldsymbol{s}=s_{0}\ldots s_{N_{\boldsymbol{s}}}$ of system (\ref{eq:dynamics}) starting from any $s_{0}\in\mathcal{C}$ under controllers satisfying (\ref{eq:CBF}), one has $s_{i}\in\mathcal{C}$ $\forall i\in\left\{0,\ldots,N_{\boldsymbol{s}}\right\}$.
\end{defn}

\begin{defn}
The relative degree of a (sufficiently many times) differentiable function $h:S\rightarrow\mathbb{R}$ with respect to system~\eqref{eq:dynamics} is the number of times it needs to be differentiated along
its dynamics until the control $a$ explicitly shows in the corresponding derivative.
\end{defn}

For the continuous differentiable function $h$ of the system $\mathcal{S}$ with higher relative degree $r_{b}\geq 1$, denote $f'(s)=f(s)+d(s)$ and the $r_{b}^{\text{th}}$ time-derivative of $h(s)$ is
$$h^{(r_{b})}(s)=L^{r_{b}}_{f'}h(s)+L_{g}L^{r_{b}-1}_{f'}h(s)a,$$

Defining a traverse variable as
$$\begin{array}{c}
\xi_{b}=[h(s),\dot{h}(s),\ldots,h^{(r_{b}-1)}(s)]^{T}\\
=[h(s),L_{f'}h(s),\ldots,L_{f'}^{r_{b}-1}h(s)]^{T},
\end{array}$$
a linearized system of $\mathcal{S}$ can be formulated:
$$\begin{array}{c}
\xi_{b}\left(s\right)=A_{b}\dot{\xi}_{b}\left(s\right)+B_{b}u\\
h\left(s\right)=C_b^T\xi_{b}\left(s\right)
\end{array}$$
where
{\small
$$\begin{array}{ccc}
A_{b}=\left[\begin{array}{ccccc}
0 & 1 & 0 & \ldots & 0\\
0 & 0 & 1 & \ldots & 0\\
\vdots & \vdots & \vdots & \ddots & \vdots\\
0 & 0 & 0 & \ldots & 1\\
0 & 0 & 0 & 0 & 0
\end{array}\right], & B_{b}=\left[\begin{array}{c}
0\\
0\\
\vdots\\
0\\
1
\end{array}\right], & C_{b}=\left[\begin{array}{c}
1\\
0\\
\vdots\\
0\\
0
\end{array}\right],
\end{array}$$}%
and $u=L^{r_{b}}_{f'}h(s)+L_{g}L^{r_{b}-1}_{f'}h(s)a$ is the input-output linearized control. The linearized form allows to extend the CBFs for higher relative degrees such as Exponential CBF (ECBF) \cite{Nguyen2016}.

\begin{defn}
	\label{def:ECBF} Consider the system $\mathcal{S}$ in~\eqref{eq:dynamics} and assume $d(x)$ is known, and the safe set $\mathcal{C}\subseteq S$ with a continuous differentiable (barrier) function $h:S\rightarrow\mathbb{R}$ in (\ref{eq:safe_set}) that has relative degree $r_{b}\geq 1$. $h(s)$ is an ECBF if there exists $K_{b}\in\mathbb{R}^{1\times r_{b}}$ s.t.
\begin{equation}
 L^{r_{b}}_{f'}h(s)+L_{g}L^{r_{b}-1}_{f'}h(s)a+K_{b}\xi_{b}\geq 0\label{eq:ECBF}
\end{equation}
The system $\mathcal{S}$ is forward invariant in $\mathcal{C}$ if starting from $s_{0}\in\mathcal{C}$, there exists an ECBF $h(s)$ and controllers satisfy~\eqref{eq:ECBF}.
\end{defn}

\begin{rem}
The row vector $K_{b}$ should be selected to render a stable close-loop matrix $A_{b}-B_{b}K_{b}$. Moreover, for $r_{b}>1$, ECBF is a special case of the general Higher Order CBF (HOCBF) \cite{Xiao2019} which can easily be used with our framework.
\end{rem}

Now, we can relax the assumption of full system knowledge, and extend results to the unknown system described in Section~\ref{subsec:Labeled-MDP} by incorporating GPs. Specifically, the unknown part $d(x)$ is approximated by the learned GP model with mean $u(s)$, covariance $\sigma(s)$ and $k_{\delta}$ as in (\ref{eq:GP}). Let's denote by $\hat{f}(s)$ the estimation of the function $f'(s)$, e.g., $\hat{f}(s)=f(s)+\hat{d}(s)$, where $\hat{d}(s)\in[u\left(s\right)-k_{\delta},u\left(s\right)+k_{\delta}]$. The GP-based traverse variable is represented by $\hat{\xi_{b}}=[h(s),L_{\hat{f}}h(s),\ldots,L_{\hat{f}}^{r_{b}-1}h(s)]^{T}$. We propose a GP-based ECBF for the nominal system.

\begin{thm}
	\label{thm:GP_CBF} Consider the system $\mathcal{S}$ in ~\eqref{eq:dynamics} with unknown $d(x)$ and $\mathcal{C}\subseteq S$ with a differentiable (barrier) function $h:S\rightarrow\mathbb{R}$ s.t. $\forall s\in \mathcal{C}, h\left(s\right)\geq0$. If a close-loop stable $K_{b}$ exists s.t.
\begin{equation}
L^{r_{b}}_{\hat{f}}h(s)+L_{g}L^{r_{b}-1}_{\hat{f}}h(s)a+K_{b}\hat{\xi_{b}}\geq 0,\label{eq:GP_CBF_condition}
\end{equation}
then starting from any $s_{0}\in\mathcal{C}$, controllers of the system $\mathcal{S}$ satisfying ~\eqref{eq:GP_CBF_condition}, render the set $\mathcal{C}$ forward invariant with probability at least $(1-\delta)^{n}$.
\end{thm}
\begin{proof}
Since $\hat{d}(s)\in[u\left(s\right)-k_{\delta},u\left(s\right)+k_{\delta}]$ and {\color{black}Proposition~\ref{prop:GP-estimation} holds}, we have 
$$\Pr\left\{F(s,a)+\mathcal{GP}_{l}\leq \hat{F}(s,a)\leq  F(s,a)+\mathcal{GP}_{h}\right\}\geq(1-\delta)^{n}$$
where $F(s,a)=f(s)+g(s)a$, $\hat{F}(s,a)=F(s,a)+\hat{d}(s)$, $\mathcal{GP}_{l}=u(s)-k_{\delta}\sigma(s)$, and $\mathcal{GP}_{h}=u(s)+k_{\delta}\sigma(s)$. Thus, we obtain the conclusion
$$\Pr\left\{L^{r_{b}}_{\hat{f}}h(s)+L_{g}L^{r_{b}-1}_{\hat{f}}h(s)a+K_{b}\hat{\xi_{b}}\geq 0\right\}\geq(1-\delta)^{n}.$$
\end{proof}

{\color{black}
\begin{rem}
In Theorem~\ref{thm:GP_CBF}, the probability $(1-\delta)^{n}$ of the set being forward invariant is concluded based on Proposition~\ref{prop:GP-estimation}. It depends globally on the accuracy of the estimated bounded model errors using GP.
\end{rem}
}

Next, we formulate the relaxed ECBF condition \eqref{eq:GP_CBF_condition} as the quadratic program (QP)

\begin{equation}
\begin{array}{c}
\left(a_{CBF},\epsilon\right)=\underset{a,\epsilon}{\arg\min}(\frac{1}{2}a^{T}H(s)a+K_{\epsilon}\epsilon)\\
\begin{array}{cc}
\text{s.t} & L^{r_{b}}_{\hat{f}}h(s)+L_{g}L^{r_{b}-1}_{\hat{f}}h(s)a+K_{b}\hat{\xi_{b}}+\epsilon\geq 0\\
 & a_{low}^{(i)}\leq a^{(i)}\leq a_{high}^{(i)}, i\in\left\{0,\ldots,n\right\}\label{eq:QP},
\end{array}
\end{array}
\end{equation}
where $H(s)$ is a positive definite matrix (pointwise in $s$), $a_{low}^{(i)}$ and $a_{high}^{(i)}$ represent the lower and higher bounds of each control input $a^{(i)}$.
To ensure the existence of solutions for the QP, $\epsilon$ is a relaxation variable, and $K_{\epsilon}$ is a large value that penalizes safety violations. The solution of the ECBF-QP enforces the safe condition with minimum norm
(minimum control energy). 


\begin{lem}
\label{lem:CBF-QP}
For dynamic system ~\eqref{eq:dynamics} with unknown $d(s)$ and $s_{0}\in\mathcal{C}=\left\{ s\in\mathbb{R}^{n}:h\left(s\right)\geq0\right\}$, 
\\
(i) feasible solutions of ~\eqref{eq:QP} for all $s\in S$ with $\epsilon=0$ renders the safe set $\mathcal{C}$ forward invariant with probability $(1-\delta)^{n}$. 
\\
(ii) feasible solutions of ~\eqref{eq:QP} for all $s\in S$ with $\epsilon_{high} \geq \epsilon >0$ approximately renders the safe set $\mathcal{C}_{\epsilon}=\left\{ s\in\mathbb{R}^{n}:h_{\epsilon}(s)=h\left(s\right)+\frac{\epsilon_{high}}{\eta}\geq0\right\}$ forward invariant with probability $(1-\delta)^{n}$.
\end{lem}
\begin{proof}
For part (i), since $\epsilon=0$, the feasible solutions  of ~\eqref{eq:QP} strictly follows the ECBF condition in Theorem~\ref{thm:GP_CBF}, and it provides the probabilistic bound $(1-\delta)^{n}$ of the GP model~\eqref{eq:GP}.

For part (ii), we utilize the discrete-time system $\mathcal{S}_{d}$ as an approximation of~\eqref{eq:dynamics} given the sampling time $\varDelta t$, and the safe set in \eqref{eq:safe_set} is estimated as $\mathcal{C}_{d}$ over $\mathcal{S}_{d}$.  Inspired by \cite{Agrawal2017}, there exists a discrete-time ECBF $h_{d}$ with $\eta\in(0,1]$ that renders the set invariant as
\begin{equation}
    h_{d}\left(\hat{F}(s_{t},a)\right)\geq\left(1-\eta\right)h_{d}\left(s_{t}\right),
\end{equation}
The constraint $h_{d}\left(\hat{F}(s_{t},a)\right)\geq\left(1-\eta\right)h_{d}\left(s_{t}\right)+\epsilon$ in~\eqref{eq:QP} is reformulated as
\begin{equation}
     h_{d}\left(\hat{F}(s_{t},a)\right)+\frac{\epsilon_{high}}{\eta}\geq0\geq\left(1-\eta\right)(h_{d}\left(s_{t}\right)+\frac{\epsilon_{high}}{\eta}).
\end{equation}
This equation and Theorem~\ref{thm:GP_CBF} conclude the proof.
\end{proof}
Note that constructing the accurate discrete-time ECBF over $\mathcal{S}_{d}$ for general safe requirements is challenging even if there exists one, and lemma~\ref{lem:CBF-QP} theoretically applies it to evaluate the performance of the relaxed ECBF-QP.  

\begin{rem}
One can easily extend the framework for data-driven based barrier functions \cite{Srinivasan2020,long2021learning}. This work focuses on efficient safe learning for the optimal policy that satisfies high-level LTL over infinite horizons and bypasses the consideration of unknown ECBFs. Also, we can combine multiple ECBFs as constraints into~\eqref{eq:QP} to define complex safe regions.
\end{rem}

\subsection{ECBF-Based Safe Learning\label{subsec:CBF_Learning}} 

Before developing the safety-critical methodology, it is important to show that there is no conflicts between safe exploration and optimal policies generated from Section~\ref{sec:RL}.

\begin{lem}
\label{lem:safe_policy}
Given an LTL formula of the form $\phi=\oblong\phi_{safe}\land\phi_{g}$ defined in Section~\ref{subsec:LTL}, let $\pi_{opt}$ denotes the optimal policy obtained from Section~\ref{sec:RL} and $\boldsymbol{\Pi}_{safe}$ denotes a set of all safe policies satisfying $\oblong\phi_{safe}$. One has the property $\pi_{opt}\subseteq\boldsymbol{\Pi}_{safe}$.
\end{lem}
Lemma~\ref{lem:safe_policy} follows straightforward from the fact that $\Always\phi_{safe}$ is encoded the LTL objective $\phi$, and $\pi_{opt}$ is guaranteed to satisfy $\phi$. This relationship is shown in Fig.~\ref{fig:guiding}.

Let $x_{t}=(s_{t},\overline{q}_{t})$ and $\boldsymbol{\pi}_{t}$ denote the product state and learned policy at time-step $t$, respectively. 
The action $u_{t}^{\mathcal{P}}$ is obtained based on the policy $\boldsymbol{\pi}_{t}$, i.e., $u_{t}^{\mathcal{P}}\sim \boldsymbol{\pi}_{t}$. Based on Def.~\ref{def:P-MDP},
the action of $\mathcal{M}$ is
\begin{equation}
a_{RL}(s_{t})=\begin{cases}
u_{t}^{\mathcal{P}}, & \text{if } u^{\mathcal{P}}\notin\left\{ \epsilon\right\}\\
0, & \text{otherwise}
\end{cases} .\label{eq:action_mapping}
\end{equation}

The controller for any state $s$ during training can be generated based on (\ref{eq:action_mapping}) as $a_{RL}(s)$. However, such a controller may not be safe. To overcome this issue, we build the QP according to GP-based ECBF in~\eqref{eq:QP} as a safeguard module to provide the minimal perturbation $a_{pt}(s)$ for the original control $a_{RL}(s)$. 

\begin{equation}
\begin{aligned}
\left(a_{pt},\epsilon\right)=\underset{a,\epsilon}{\arg\min}\;\frac{1}{2}(a_{RL}(s)+a)^{T}H(s)(a_{RL}(s)+a)+K_{\epsilon}\epsilon\\
\begin{array}{cc}
\text{s.t} & L^{r_{b}}_{\hat{f}}h(s)+L_{g}L^{r_{b}-1}_{\hat{f}}h(s)(a_{RL}(s)+a)+K_{b}\hat{\xi_{b}}+\epsilon\geq 0\\
 & a_{low}^{(i)}\leq (a_{RL}(s)+a)^{(i)}\leq a_{high}^{(i)}, i\in\left\{0,\ldots,n\right\}\label{eq:RL_QP},
\end{array}
\end{aligned}
\end{equation}

Consequently, the actual implemented safety-critical controllers use $a_{pt}$ from (\ref{eq:RL_QP}) as
\begin{equation}
a_{Safe}(s)=a_{RL}(s)+a_{pt}(s)\label{eq:safe_control}
\end{equation}
During the evolution of the dynamic system (\ref{eq:dynamics}), the "shield" of (\ref{eq:RL_QP}) compensates the model-free RL controller $a_{RL}(s)$ based on the GP-based ECBF condition, and keeps the state safe via deploying the final safe controller $a_{Safe}(s)$. 

However, purely combining GPs and ECBF during the learning process may negatively influence the exploration and the original optimal convergence shown in Section~\ref{sec:RL}.
Since the optimal parameterized RL-policy $\boldsymbol{\pi}^{*}$ of the controller $a_{RL}(s)$ attempts to optimize the expected return and $\boldsymbol{\pi}^{*}$ is generated based on the distribution of the policy-gradient optimization in (\ref{eq:actor_update}), the feedback reward at each time should correspond to the controller $a_{RL}(s)$. While the actual reward collection in the reply buffer is associated with the controller $a_{CBF}(s)$, which is not consistent with the RL-policy $\boldsymbol{\pi}^{*}$ and induces undesired behaviors.
As a result, the modular DDPG receives no informative feedback about the unsafe behaviors compensated via~\eqref{eq:RL_QP}. Moreover, another crucial issue is that the functionality of $a_{Safe}(s)$ is too monotonous to guide the policy exploration, and the corresponding RL policy $\boldsymbol{\pi}^{*}$ may always linger around the margins of unsafe sets, as illustrated in Fig.~\ref{fig:guiding} (a).

\subsection{Exploration Guiding \label{subsec:guiding}} 

\begin{figure}
	\centering{}\includegraphics[scale=0.35]{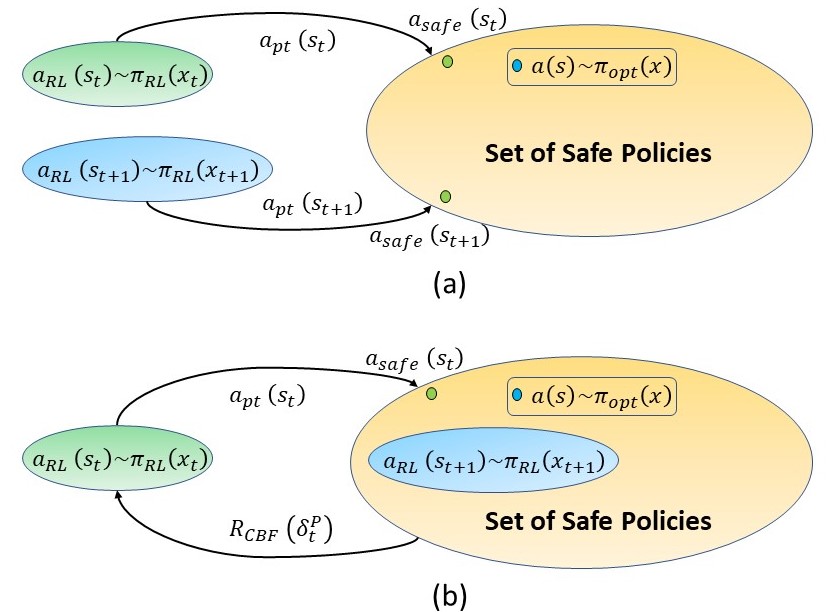}\caption{\label{fig:guiding} Illustration of the improvement of exploration guiding. (a) Policy optimization without enforcing exploration where the RL policies receive no information on the safety requirements. (b) Policy optimization with automaton-based guiding to enforce the RL policy generated within the set of safety policies.}
\end{figure}

\begin{figure}
	\centering{}\includegraphics[scale=0.4]{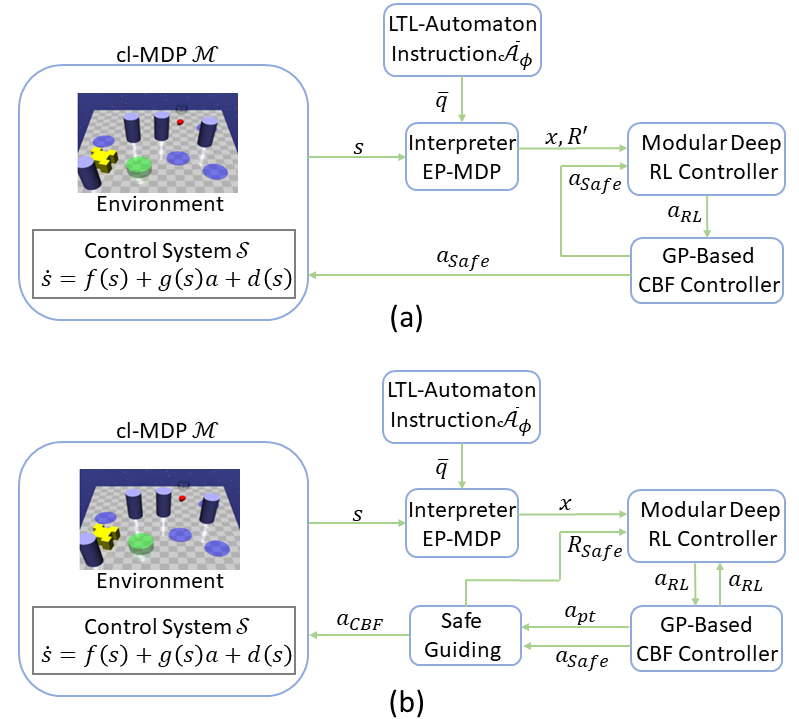}\caption{\label{fig:safe_guiding} Control synthesis integrates with the modular RL controllers and GP-based CBF controllers. (a) Modular RL architecture that directly uses the CBF-based safe controllers without enforcing exploration. (b) Framework enabled with safe guiding and improvement of learning efficiency.}
\end{figure}

\begin{algorithm}
\caption{\label{Alg2} Safe Learning: LTL-Guided Control Algorithm}


\small 
\singlespacing

\begin{algorithmic}[1]

\Procedure {Input: } {Dynamic system $\mathcal{S}$ , LDGBA $\mathcal{{A}}_{\phi}$}

{Output: } {Optimal policy $\boldsymbol{\pi}^{*}$
} 

{Initialization: } {$\left|Q\right|$ actor $\boldsymbol{\pi}_{q_{i}}\left(x\left|\theta^{u_{q_{i}}}\right.\right)$
and critic networks $Q_{q_{i}}\left(x,u^{\mathcal{P}}\left|\theta^{Q_{q_{i}}}\right.\right)$
with arbitrary weights $\theta^{u_{q_{i}}}$ and $\theta^{\boldsymbol{Q}_{q_{i}}}$
for all $q_{i}\in Q$}; 
{$\left|Q\right|$ corresponding target
networks $\boldsymbol{\pi}_{q_{i}}'\left(x\left|\theta^{u_{q_{i}}'}\right.\right)$
and $Q_{q_{i}}'\left(x,u^{\mathcal{P}}\left|\theta^{Q_{q_{i}}'}\right.\right)$
with weights $\theta^{u_{q_{i}}'}$ and $\theta^{Q_{q_{i}}'}$ for
each $q_{i}\in Q$, respectively};
{$\left|Q\right|$ replay buffers
$B_{q_{i}}$}; {$\left|Q\right|$ random processes noise $N_{q_{i}}$}; 
{a global reply buffer $B_{\mathcal{GP}}$ for GP model}

\State set maximum episodes $E$ and iteration number $\tau$

\State \textbf{Repeat} Each Episode

\State sample an initial state $s_{0}$ of $\mathcal{M}$ and $q_{0}$
of $\mathcal{{A}}_{\phi}$ as $s_{t},q_{t}$


\While {$t\leq\tau$ }

\State select action $u_{t}^{\mathcal{P}}=\boldsymbol{\pi}_{q_{t}}\left(x\left|\theta^{u_{q_{i}}}\right.\right)+R_{q_{t}}$.

\State set $a_{RL}(s_{t})=u_{t}^{\mathcal{P}}$ and obtain $a_{Safe}(s_{t})$ via QP (\ref{eq:RL_QP}) as:
\[
\begin{aligned}
\left(a_{pt},\epsilon\right)=\underset{a,\epsilon}{\arg\min}\;\frac{1}{2}(a_{RL}(s)+a)^{T}H(s)(a_{RL}(s)+a)+K_{\epsilon}\epsilon\\
\begin{array}{cc}
\text{s.t} & L^{r_{b}}_{\hat{f}}h(s)+L_{g}L^{r_{b}-1}_{\hat{f}}h(s)(a_{RL}(s)+a)+K_{b}\hat{\xi_{b}}+\epsilon\geq 0\\
 & a_{low}^{(i)}\leq (a_{RL}(s)+a)^{(i)}\leq a_{high}^{(i)}, i\in\left\{0,\ldots,n\right\},
\end{array}
\end{aligned}
\]

\State execute $a_{Safe}(s_{t})$ and observe
$$x_{t+1}=\left(s_{t+1},q_{t+1},T\right), R'\left(x_{t},u_{t}^{\mathcal{P}},x_{t+1}\right), \gamma\left(x_{t+1}\right)$$

\State execute three-step safe guiding process in Def.~\ref{def:safe_guding}: 
$$x_{t+1}=\left(s_{t+1},q_{t+1},T\right), R_{CBF}\left(x_{t},u_{t}^{\mathcal{P}},x_{t+1}\right)$$

\State store the transition information in replay buffers $B_{q_{t}}$ and {\scriptsize{}\hspace*{3.8em}}global reply buffer $B_{\mathcal{GP}}$:
$$\left(x_{t},a_{RL}(s_{t}),R_{CBF}\left(x_{t},u_{t}^{\mathcal{P}},x_{t+1}\right),x_{t+1}\right)$$

\State execute the updates via $f_{V}\left(q_{t+1},T\right)$
and $f_{\Phi}\left(q_{t+1},T_{\Phi}\right)$

\State update GP model using measurements from global {\scriptsize{}\hspace*{3.9em}}reply buffer as:
$$\hat{u}_{i}(s_{*})=\hat{k}^{T}_{i}(K_{i}+\sigma^{2}_{noise}\boldsymbol{I}_{N_{m}})^{-1}y_{i}$$
$$\hat{\sigma}^{2}_{i}(s_{*})=k_{i}((s_{*},(s_{*})-\hat{k}^{T}_{j}(K_{j}+\sigma^{2}_{noise}\boldsymbol{I}_{N_{m}})^{-1}\hat{k}_{i}$$

\State apply Alg.~\ref{Alg3} to update the modular actor-critic architecture

\EndWhile

\State \textbf{Until} End of trial

\EndProcedure

\end{algorithmic}
\end{algorithm}

In order to achieve safe and efficient guiding, the work in \cite{Cheng2019} estimates the previous history of CBF perturbations to improve the efficiency of the learning process. However, such a design
may negatively impact the exploration of DDPG for original optimal
policies proposed in Section~\ref{sec:RL}, and it can not provide formal guarantees on LTL satisfaction.

To overcome these challenges, we propose an automaton-based guiding process combining the properties of E-LDGBA and ECBF. Given an LTL formula of the form $\phi=\oblong\phi_{safe}\land\phi_{g}$, where $\Always\phi_{safe}$ is the specification to render the safe set $\mathcal{C}$ forward invariant, the intuition for (\ref{eq:safe_control}) is that $a_{pt}(s)\neq0$ implies the RL controller $a_{RL}(s)$ is unsafe, i.e., it violates $\oblong\phi_{safe}$ in $\phi$. Consequently, the objective of efficient exploration is to encourage the RL-agent operating in the safe region $\mathcal{C}$, and to enforce $a_{pt}(s)$ decaying to zero.

Recall during the learning process, we have $x_{t}=(s_{t},\overline{q}_{t})$, $\boldsymbol{\pi}_{t}$ at time $t$, and obtain $a_{RL}(s_{t})$ from~\eqref{eq:action_mapping} as the corresponding $u_{t}^{\mathcal{P}}\sim \boldsymbol{\pi}_{t}$.
GP-based QP \eqref{eq:RL_QP} generates $a_{pt}(s_{t})$ and $a_{Safe}(s_{t})$. The next state $x_{t+1}=(s_{t+1},\overline{q}_{t+1})$ is generated by taking the safe action $a_{Safe}(s_{t})$ s.t. $p_{S}(\left.s_{t+1}\right|s_{t},a_{Safe}(s_{t}))\neq 0$ and $\overline{q}_{t+1}=\overline{\delta}(\overline{q}_{t},L(s_{t}))$.
 The automaton-based safe guiding consists of three steps: ECBF-based reward shaping, violation-based automaton updating, and relay buffer switching. The procedure of safe guiding is shown in Fig.~\ref{fig:safe_guiding} (b), compared with the method of directly integrating the safe execution in Fig.~\ref{fig:safe_guiding} (a).
 
 \begin{defn}\label{def:safe_guding}
Given the transition $\delta^{\mathcal{P}}_{t}=(x_{t},a_{Safe}(s_{t}),x_{t+1})$, where $x_{t}=(s_{t},\overline{q}_{t})$, $a_{RL}(s_{t})$ is obtained from $u_{t}^{\mathcal{P}}\sim \boldsymbol{\pi}_{t}$ as \eqref{eq:action_mapping}, $a_{Safe}(s_{t})=a_{RL}(s_{t})+a_{pt}(s_{t})$, and $x_{t+1}=(s_{t+1},\overline{q}_{t+1})$ is generated based on Def.~\ref{def:P-MDP} during learning process, the three-step exploration guiding is defined as
\\
(1) The reward is shaped based on the safe properties:
\begin{equation}
R_{CBF}\left(\delta^{\mathcal{P}}_{t}\right)= \begin{cases}
r_{n}\cdot |a_{pt}(s_{t})|, & \text{if }a_{pt}(s_{t})\neq 0,\\
R'\left(\delta^{\mathcal{P}}_{t}\right), & \text{if }a_{pt}(s_{t})=0,
\end{cases}\label{eq:CBF_reward}
\end{equation}
where $r_{n}\in\mathbb{R}$ is a constant parameter s.t. $r_{n}<0$;
\\
(2) The automaton component $\overline{q}_{t+1}$ of product state $x_{t+1}$ is updated:
\begin{equation}
\overline{q}_{t+1}= \begin{cases}
\overline{q}\text{ s.t. }\overline{q}\in\overline{Q}_{unsafe}, & \text{if }a_{pt}(s_{t})\neq 0\\
\overline{\delta}(\overline{q}_{t},L(s_{t})), & \text{if }a_{pt}(s_{t})=0.
\end{cases}\label{eq:CBF_automaton}
\end{equation}
\\
(3) Instead of storing information $a_{Safe}(s_{t})$ and $R'\left(\delta^{\mathcal{P}}_{t}\right)$, add $\left(x_{t},a_{RL}(s_{t}),R_{CBF}\left(\delta^{\mathcal{P}}_{t}\right),x_{t+1}\right)$ to replay buffer for training.
\end{defn}

In \eqref{eq:CBF_reward}, $r_{n}<0$ s.t. $r_{n}\cdot|a_{pt}(s_{t})|$ represents how much the $a_{RL}(s_{t})$ violates the safety constraint $\oblong\phi_{safe}$, i.e., propositional to the absolute value of CBF compensators. Similarly, \eqref{eq:CBF_automaton} switches the automaton state $\overline{q}_{t+1}$ to an unsafe state in $\overline{Q}_{unsafe}$, see Def.~\ref{def:Non-accepting unsafe sink}.
Different from~\cite{Cheng2019} that takes $a_{Safe}(s_{t})$ into the replay buffer for training, the third step keeps the original modular RL controller $a_{RL}(s_{t})$ and integrates it with the shaped reward $R_{CBF}$ into training.
As for actor-critic methods of the modular RL, such a design can improve the efficiency of exploration and stabilize the learning results, since the controllers $a_{RL}$ in the relay buffer are generated from the policy distributions of the actor, whereas controllers $a_{Safe}$ are only for safe execution and are not consistent with the outputs of the modular actor-critic architecture.
Finally, we show that the original optimal policies generated in Section~\ref{sec:RL} remain invariant via the safe learning and guiding processes.

\begin{thm}
    \label{thm:safe_RL} Given a cl-MDP $\mathcal{M}$ and an E-LDGBA
	$\mathcal{\overline{A}}_{\phi}$, the optimal policy $\boldsymbol{\pi}^{*}$ generated by combining the modular DDPG and three-step exploration guiding procedure in Def.~\ref{def:safe_guding} maximizes the probability of satisfying $\phi$ in the limit.
\end{thm}
\begin{proof}
Safe guiding efficiently enhances the rapid decaying of $a_{pt}(s_{t})$ to $0$.
Based on Def.~\ref{def:safe_guding}, the intuition of safe guiding is to bridge the connection between the unsafe component $\overline{Q}_{unsafe}$ of E-LDGBA and ECBF controllers $a_{pt}(s_{t})$ s.t. $a_{pt}(s_{t}0\neq 0$ indicates the original $a_{RL}(s_{t})$  violates the safe requirement of $\phi=\oblong\phi_{safe}\land\phi_{g}$.
Consequently, the objective is to verify that assigning negative reward to the states $x$ with automaton components $q\in\overline{Q}_{unsafe}$ preserves optimal solutions in Proposition~\ref{prop:shaped_probability} and Theorem~\ref{thm:shaped_probability}. We prove this by contradiction in Appendix~\ref{apped:thm:safe_RL}.
\end{proof}

Safe guiding integrates the property of violation of the LTL formula $\phi=\oblong\phi_{safe}\land\phi_{g}$ and the safe set to ensure efficient exploration. Thus, $a_{pt}(s_{t})$ gradually decays to $0$  and becomes inactive.
The overall structure pushes the RL policies generated from the set of safe polices without altering the original optimality, while maintaining safety during learning.

\section{Algorithm Summary\label{sec:Solution}}

\begin{algorithm}
\caption{\label{Alg3} Modular Deep Deterministic Policy Gradient}
\small 
\singlespacing

\begin{algorithmic}[1]

\Procedure {Input: } {Actor $\boldsymbol{\pi}_{q_{t}}\left(x\left|\theta^{u_{q_{t}}}\right.\right)$, critic $Q_{q_{t}}\left(x,u^{\mathcal{P}}\left|\theta^{Q_{q_{t}}}\right.\right)$, corresponding target
networks $\boldsymbol{\pi}_{q_{t}}'\left(x\left|\theta^{u_{q_{t}}'}\right.\right)$
and $Q_{q_{t}}'\left(x,u^{\mathcal{P}}\left|\theta^{Q_{q_{t}}'}\right.\right)$, reply buffer $B_{q_{t}}$}

{Output: } {Updated weights $\theta^{u_{q_{t}}}$, $\theta^{\boldsymbol{Q}_{q_{t}}}$, $\theta^{u_{q_{t}}'}$, and $\theta^{Q_{q_{t}}'}$
} 

\State calculate target values for each $i\in N$ 
(mini-batch sampling {\scriptsize{}\hspace*{2.0em}}of $B_{q_{t}}$) as:

\[
{y_{i}=R'\left(x_{i},u_{i}^{\mathcal{P}},x_{i+1}\right)+\gamma\left(x_{i}\right)\cdot Q_{q_{i+1}}'\left(x_{i+1},u_{i+1}^{\mathcal{P}}\left|\theta^{Q_{q_{i+1}}'}\right.\right)}
\]

\State update weights $\theta^{Q_{q_{t}}}$ of critic neural network
{\scriptsize{}\hspace*{2.0em}} $Q_{q_{t}}\left(x,u^{\mathcal{P}}\left|\theta^{Q_{q_{t}}}\right.\right)$ by minimizing the loss function:

\[
L=\frac{1}{N}\stackrel[i=1]{N}{\sum}\left(y_{i}-Q_{q_{t}}\left(x_{i},u_{i}^{\mathcal{P}}\left|\theta^{Q_{q_{t}}}\right.\right)\right)^{2}
\]

\State update weights $\theta^{u_{q_{t}}}$ of actor neural network
$\boldsymbol{\pi}_{q_{i}}\left(x\left|\theta^{u_{q_{t}}}\right.\right)$ {\scriptsize{}\hspace*{2.0em}}by maximizing the policy gradient:

$$\begin{aligned}\nabla_{\theta^{u_{q_{t}}}}U^{q_{t}}\thickapprox & \frac{1}{N}\stackrel[i=1]{N}{\sum}\left(\nabla_{u^{\mathcal{P}}}Q_{q_{t}}\left(x_{i},u^{\mathcal{P}}\left|\theta^{Q_{q_{t}}}\right.\right)\left|_{u^{\mathcal{P}}=\boldsymbol{\pi}_{q_{t}}\left(x_{i}\left|\theta^{u_{q_{t}}}\right.\right)}\right.\right.\\
 & \left.\cdot\nabla_{\theta^{u_{q_{t}}}}\boldsymbol{\pi}_{q_{t}}\left(x_{i}\left|\theta^{u_{q_{t}}}\right.\right)\right)
\end{aligned}$$

\State soft update of target networks:

$$
\begin{array}{c}
\theta^{u_{q_{t}}'}\leftarrow\tau\theta^{u_{q_{t}}}+\left(1-\tau\right)\theta^{u_{q_{t}}'}\\
\theta^{Q_{q_{t}}'}\leftarrow\tau\theta^{Q_{q_{t}}}+\left(1-\tau\right)\theta^{Q_{q_{t}}'}
\end{array}\label{eq:soft_update}
$$

\EndProcedure

\end{algorithmic}
\end{algorithm}

The idea of enforcing optimal policies is illustrated in Fig.~\ref{fig:guiding} (b) and Fig.~\ref{fig:safe_guiding} (b). The overall procedure of safe learning and guiding is illustrate in  Alg.~\ref{Alg2}.
The procedure to synthesize the P-MDP between
a continuous MDP and an automaton of the
LTL specification is summarized in  Alg.~\ref{Alg2}. Line (6-10). Instead of constructing the P-MDP a priori, product states of P-MDP are synthesized
on-the-fly. Note that for each iteration we first observe the output of the shaped reward function $R'$, then execute the update process via $f_{V}$
and $f_{\Phi}$ (line 10-11). In line7, the safe "shield" QP~\eqref{eq:RL_QP} synthesizes the minimum perturbation to generate safe controllers. The GP model is updated online during training by sampling measurements from the global reply buffer $B_{\mathcal{GP}}$ (line 12). 

We assign each DDPG an individual
replay buffer $B_{q_{i}}$ and a random process noise $N_{q_{i}}$.
The corresponding weights of modular networks, i.e., $Q_{q_{i}}\left(x,u^{\mathcal{P}}\left|\theta^{\boldsymbol{Q}_{q_{i}}}\right.\right)$
and $\boldsymbol{\pi}_{q_{i}}\left(x\left|\theta^{u_{q_{i}}}\right.\right)$,
are also updated at each iteration (line 13) via Alg. \ref{Alg3}. 

Alg. \ref{Alg3} shows the procedure of Section \ref{subsec:RL}. In Alg. \ref{Alg3}, all neural networks
are trained using their own replay buffer, which is a finite-sized cache that
stores transitions sampled from exploring the environment. Since the direct
implementation of updating target networks can be unstable and divergent
\cite{Mnih2015}, soft update \eqref{eq:soft_update} is employed,
where target networks are slowly updated via relaxation (line 5). 

\begin{thm}
    \label{thm:shaped_probability} Given a cl-MDP $\mathcal{M}$ and an E-LDGBA
	$\mathcal{\overline{A}}_{\phi}$, the optimal policy $\boldsymbol{\pi}^{*}$ generated from modular DDPG by applying the shaped reward (\ref{eq:Shaped_Reward}) maximizes the probability of satisfying $\phi$ in the limit i.e., $\Pr^{\boldsymbol{\pi}^{*}}\left[x_{0}\models Acc_{\mathcal{P}}\right] = \Pr_{max}\left[x_{0}\models Acc_{\mathcal{P}}\right]$.
\end{thm}
\begin{proof}
Theorem~\ref{thm:shaped_probability} follows directly from Proposition~\ref{prop:shaped_probability}, since Theorem~\ref{thm:safe_RL} preserves the optimality of Proposition~\ref{prop:shaped_probability}.
However, Proposition~\ref{prop:shaped_probability} and Theorem~\ref{thm:safe_RL} assume that all state-action values can be exactly optimized, which is not in practical when considering continuous space. {\color{black} 
As common for Deep Policy Gradient algorithms used with DNNs, the training process will terminate after reaching the maximum number of episodes, a tunable hyperparameter.} Consequently, we have to stop the training after maximum number of steps in practice, and the synthesised policy derived from this nonlinear regression process might be sub-optimal with respect to the true $\boldsymbol{\pi}^{*}$. 
\end{proof}
Note that our algorithm can be easily extended by replacing other advanced off-policy algorithms e.g., SAC, TD3.


\begin{table}[tb]
	\caption{Baseline variants tested in case study.}
	\label{tab:baselines}
	\centering{}
	\small
	\setlength{\tabcolsep}{2.5pt}
	\begin{tabular}{|m{3.3cm}|m{1.3cm}|m{1.2cm}|m{1.8cm}|}
	    \hline
		\textbf{Baseline} & \textbf{Modular} & \textbf{Safe Module} & \textbf{Exploration Guiding}  \\ 
		\hline
		Modular-DDPG-ECBF (On)-Guiding & $\checkmark$ & $\checkmark$ & $\checkmark$ \\
		\hline
		Modular-DDPG-ECBF Off-Guiding & $\checkmark$ & $\checkmark$ & $\text{\sffamily X}$ \tabularnewline
		\hline
		Modular-DDPG & $\checkmark$ & $\text{\sffamily X}$ & $\text{\sffamily X}$ \tabularnewline
		\hline
		Standard-DDPG &$\text{\sffamily X}$ & $\text{\sffamily X}$ & $\text{\sffamily X}$ \tabularnewline
		\hline 
	\end{tabular}
\end{table}

\section{EXPERIMENTS\label{sec:experiment}}

We demonstrate the framework in several robotic environments with corresponding LTL tasks. To show the effectiveness of safe modular DDPG with guiding enabled, we compare our framework referred to as safe modular with guiding (Modular-DDPG-ECBF-Guiding) with three baselines: (i) safe standard DDPG with guiding (Standard-DDPG-ECBF-Guiding), (ii) modular or standard DDPG without safe module enabled (Standard-DDPG, Modular-DDPG), (iii) safe modular DDPG without guiding (Modular-DDPG-ECBF-off-Guiding).
We therefore consider four variants
of the baselines as summarized in Table~\ref{tab:baselines}.
We used Owl~\cite{Kretinsky2018} to convert LTL specifications
to LDGBA that are then transformed into E-LDGBA.
Various implementations based on OpenAI gym are carried out on a machine with 3.60 GHz quad-core CPU, 16 GB RAM, an external Nvidia RTX 1080 GPU and Cuda enabled. The details of experimental setup can be found in Appendix~\ref{apped:experiment}. 

It is worth pointing out that it is more challenging for RL agents satisfying tasks over infinite horizons. Consequently, we focus on the evaluation of the infinite-horizon formulas and analyze their success rates (Fig.~\ref{fig:success_rate}). The safe sets and control barrier functions are defined separately for each dynamical system, and Fig.~\ref{fig:safety_rate} shows the safety rates of all tasks for different baselines. The video demonstrations can be found on our YouTube channel\footnote{\url{https://youtu.be/liB1Po7oXeo}}.

\subsection{Robotic Joint Dynamics \label{subsec:Cartpole}}

\begin{figure}
	\centering{}
	\includegraphics[scale=0.28]{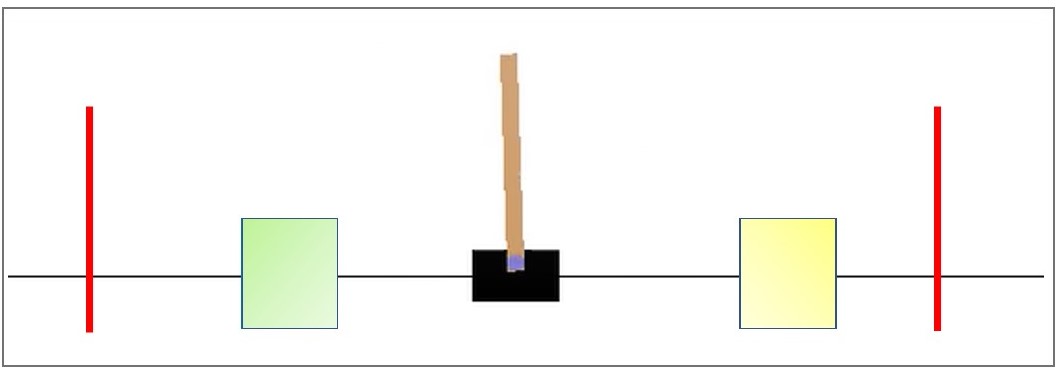}
	\caption{{\color{black} LTL specifications $\phi_{Cart1}$ and $\phi_{Cart2}$ require the Cart-Pole to visit the green and yellow rectangular regions over infinite and finite horizons, respectively, while satisfying safety tasks expressed as ECBF, i.e., the cart never exceeds the red boundaries and the pole never falls down.}}
	\label{fig:Cartpole}
\end{figure}

We first tested our algorithms on controlling two systems in simulated OpenAI gym environments.
The LTL formulas over infinite horizons are of the form $\phi_{g_{1}}=\oblong\lozenge\mathtt{R}_{\mathtt{green}}\land\oblong\lozenge\mathtt{R}_{\mathtt{yellow}}$, which require visiting the green and yellow regions infinitely often while staying within the safe set $\phi_{safe}$.
The tasks over finite horizons have the form $\phi_{g_{2}}=\lozenge\left(\mathtt{R}_{\mathtt{green}}\wedge\lozenge\mathtt{R}_{\mathtt{yellow}}\right)$.

\textbf{Cart-Pole:} The physical simulation of Cart-Pole is shown in Fig.~\ref{fig:Cartpole}.
A pendulum is attached to a cart that moves horizontally along a frictionless track.
The control input is a horizontal force on the cart.
Denote $s=\left[\theta_{p},s_{c}, \dot{\theta}_{p}, \dot{s}_{c}\right]^{T}$ as the state (cart position, pendulum angle and corresponding velocities) of the dynamic system.
Its true dynamics are defined as follows:
\begin{align*}
\ddot{\theta_{p}}&=\frac{(u-d)+m_{1}\cos\theta_{p}\sin\theta_{p}+m_{2}\dot{\theta}^{2}_{p}\sin\theta_{p}}{m_{3}+m_{4}\cos^{2}\theta},
\\
\ddot{s}_{c}&=\frac{\kappa_{1}(u-d)\cos\theta_{p}+\kappa_{2}\dot{\theta}^{2}_{p}\cos\theta_{p}\sin\theta_{p}+\kappa_{3}\sin\theta}{\kappa_{4}+\kappa_{5}\cos^{2}\theta_{p}},
\end{align*}
where external control force $u$ is limited  to $u\in\left[-20N,20N\right]$, and $m_{1}, \ldots,m_{4}$ and $\kappa_{1}, \ldots,\kappa_{4}$ are the physical parameters.
To introduced model uncertainty, we assume $30\%$ error in the physical constants.
The safe set consists of two control barrier functions
$$\mathcal{C}_{1}=\left\{(\theta_{p},s_{c}):(12^{2}-\theta_{p}^{2}\geq0) \land(2.4^{2}-s_{c}^{2}\geq0)\right\}.$$
The corresponding LTL formula that holds the system in the safe set $\mathcal{C}_{1}$ for the current state is $\phi_{safe}=\phi_{\mathcal{C}_{1}}$.
The overall LTL formula over infinite horizon is $\phi_{\mathtt{Cart1}}=\oblong\phi_{\mathcal{C}_{1}}\land\phi_{g_{\mathtt{C1}}}$,
where $\phi_{g_{\mathtt{C1}}}$ requires the agent to periodically visit the green and yellow regions located between $-1.44$m to $-0.96$m, and $0.96$m to $1.44$m, respectively.
The LTL formula over finite horizon is $\phi_{\mathtt{Cart2}}=\oblong\phi_{\mathcal{C}_{1}}\land\phi_{g_{\mathtt{2}}}$,  
where $\phi_{g_{\mathtt{2}}}$ requires the agent to visit the green and yellow regions once. The results of $\phi_{\mathtt{Cart1}}$ are shown in Fig.~\ref{fig:CartPole_rewards},~\ref{fig:CartPole_training_CBFs},~\ref{fig:CartPole_learned_CBFs}.

\begin{figure}
	\centering{}\includegraphics[scale=0.32]{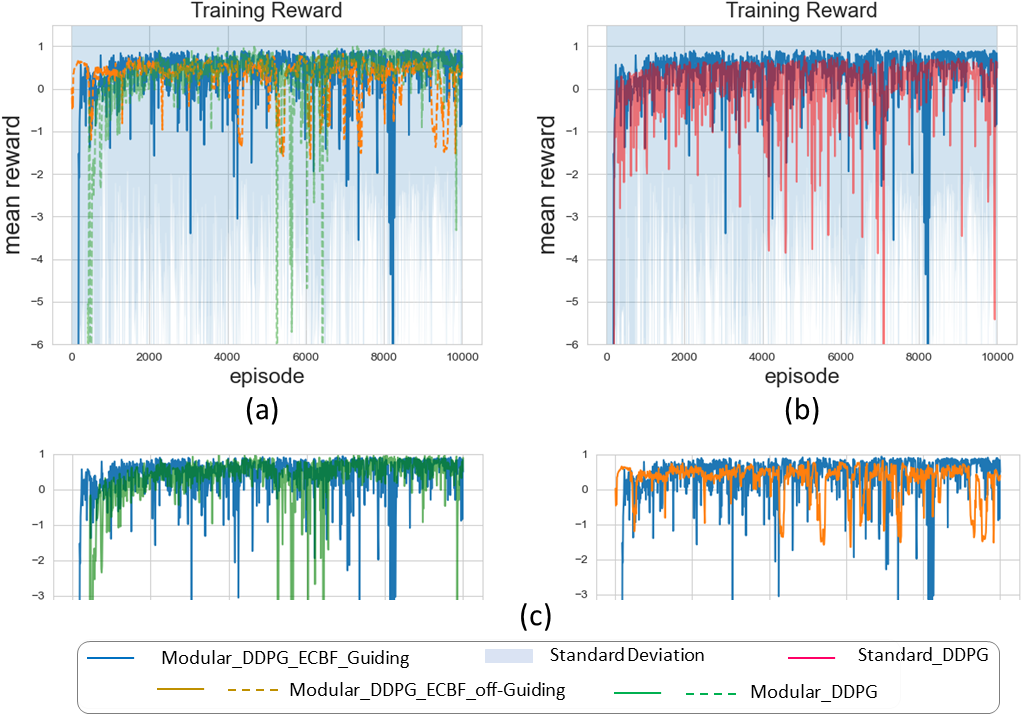}\caption{\label{fig:CartPole_rewards} Cart-Pole: Average training rewards of the LTL formula $\phi_{\mathtt{Cart1}}$. {\color{black} The learning curves are represented with solid and dashed lines represent the mean reward. Sold and dashed lines of the same color correspond to the same baseline.} (a) The results of baselines: Modular DDPG, safe modular with guiding, safe modular without guiding. (b) The comparison between safe modular and standard DDPG. (c) Close-up of the results in (a) for detailed comparison.  }
\end{figure}

\begin{figure}
	\centering{}\includegraphics[scale=0.40]{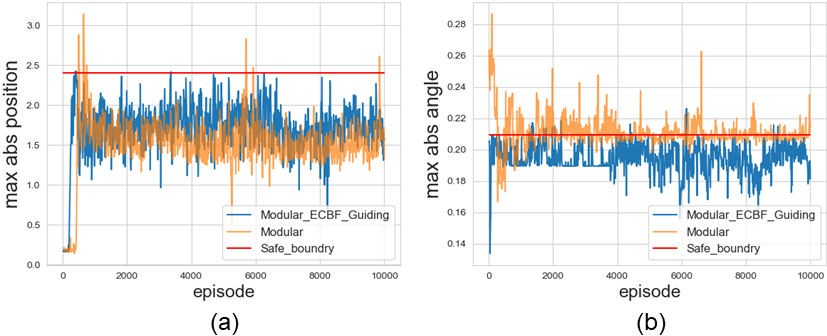}\caption{\label{fig:CartPole_training_CBFs} Cart-Pole: Maximum absolute value of position (a) and angle (b) of the LTL formula $\phi_{\mathtt{Cart1}}$ for baselines modular DDPG and safe modular DDPG during training.}
\end{figure}

\begin{figure}
	\centering{}\includegraphics[scale=0.30]{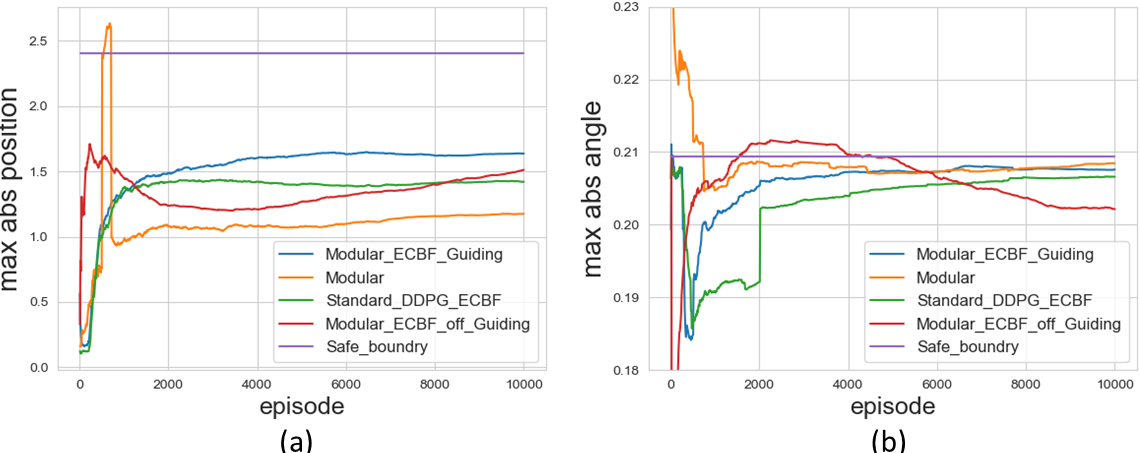}\caption{\label{fig:CartPole_learned_CBFs} Cart-Pole: Maximum absolute value of positions (a) and angle (b) of the LTL formula $\phi_{\mathtt{Cart1}}$ for different baselines during evaluation process.}
\end{figure}

Fig.\ref{fig:CartPole_rewards} compares the mean reward achieved via different baselines. Fig.\ref{fig:CartPole_rewards} (c) extracts the results of  Fig.\ref{fig:CartPole_rewards} (a)  for more detailed comparison. Fig.\ref{fig:CartPole_rewards} (a) and (c) show that safe modular DDPG and modular DDPG achieve the same performance of task satisfaction.

The methods with the safe module and guiding maintain safety during training and do not alter the RL optimality,
while the ones with the safe module and without guiding influence the exploration and RL optimality.

Fig.\ref{fig:CartPole_rewards} (b) shows that the modular DDPG has better performance than the standard DDPG (higher rewards).
Even though there's a slightly higher reward using the modular architecture in Fig.\ref{fig:CartPole_rewards} (b), it influences the success rates of optimal policies completing the task over infinite horizons as shown in Fig.~\ref{fig:success_rate}.

Fig.\ref{fig:CartPole_training_CBFs} and Fig.\ref{fig:CartPole_learned_CBFs} show the absolute value of the maximum angle and position in each episode during the learning and evaluation processes, respectively. First, Fig.\ref{fig:CartPole_training_CBFs} demonstrates the benefits of applying the safe module such that the ECBF compensators, as minimal perturbations, safeguard the RL-agent during learning. Then Fig.\ref{fig:CartPole_learned_CBFs} compares the safe performance against the baselines. Especially, it shows the advantage of the guiding procedure.
Since the dynamics of Cart-Pole is more complex and sensitive, and the RL controllers always steer the systems close to the margin of the safe set,
it makes the safety constraint easier to violate when the exploration guiding is not enabled.
The safety rates of infinite-horizon task $\phi_{\mathtt{Cart1}}$ and finite-horizon task $\phi_{\mathtt{Cart2}}$ using different baselines are shown in Fig.~\ref{fig:safety_rate}.
It illustrates the improvement due to exploration guiding.

\begin{figure}
	\centering{}\includegraphics[scale=0.15 ]{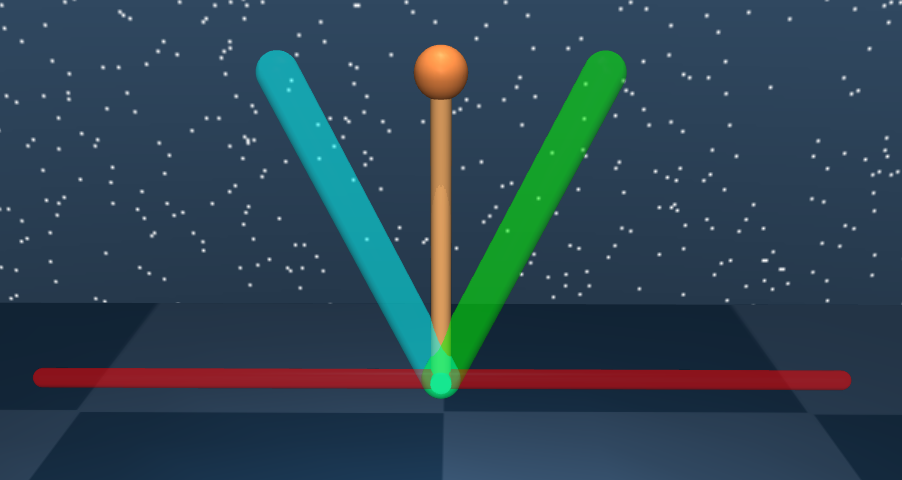}\caption{\label{fig:pendulum}  {\color{black}LTL specifications $\phi_{Pen1}$ and $\phi_{Pen2}$ require the pendulum to visit the green and yellow regions over infinite and finite horizons, respectively, while safeguarding the pendulum that never exceeds the horizontal red boundaries.}}
\end{figure}

\textbf{Inverted-Pendulum:} The physical simulation of inverted-pendulum is shown in Fig.~\ref{fig:pendulum}.
The true system dynamics of state $\theta$ with mass $m$, length $l$ and torque $u$ is
$$\dot{\theta}=\frac{3g}{2l}\sin\theta+\frac{3}{2ml^{2}}u,$$
where torque is limited to $u\in [-15,15]$ and the nominal model has $40\%$ error in the physical parameters.
{\color{black} The safe set is composed of a control barrier function $\mathcal{C}_{2}=\left\{\theta:\frac{\pi}{2}^{2}-\theta^{2}\geq0\right\}$.}
The LTL formula that holds the system in the safe set $\mathcal{C}_{2}$ is $\phi_{safe}=\phi_{\mathcal{C}_{2}}$.
The full LTL formula over infinite horizon is $\phi_{\mathtt{Pen1}}=\oblong\phi_{\mathcal{C}_{2}}\land\phi_{g_{\mathtt{1}}}$, where the blue and green regions are located at $-\frac{\pi}{4}$ rads and $\frac{\pi}{4}$ rads, respectively.
The LTL formula over finite horizon is $\phi_{\mathtt{Pen2}}=\oblong\phi_{\mathcal{C}_{2}}\land\phi_{g_{\mathtt{2}}}$. The analysis of $\phi_{\mathtt{pen1}}$ is shown in Fig.~\ref{fig:pendulum_reward-angle} and~\ref{fig:pendulum_ECBF}.

\begin{figure}
	\centering{}\includegraphics[scale=0.30 ]{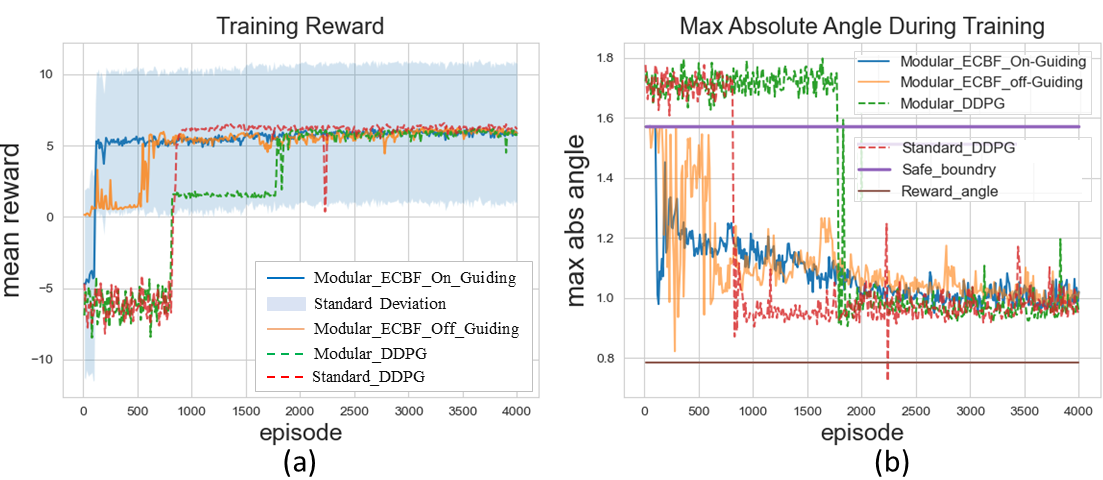}\caption{\label{fig:pendulum_reward-angle} Inverted-Pendulum: reward collection and maximum absolute angle for the formula $\phi_{\mathtt{Pen1}}$. (a) Mean reward. (b) Maximum absolute angle.}
\end{figure}

\begin{figure}
	\centering{}\includegraphics[scale=0.29 ]{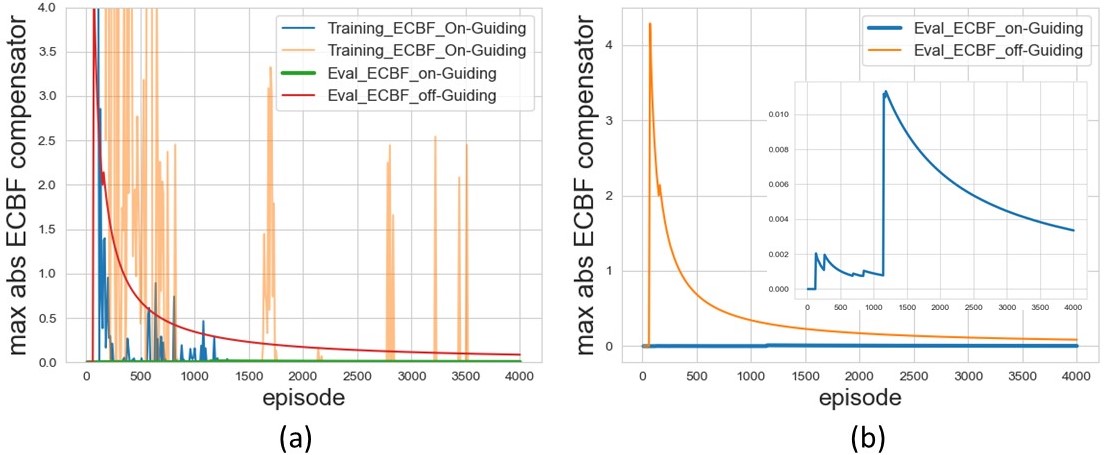}\caption{\label{fig:pendulum_ECBF} Inverted-Pendulum: comparison between the algorithms with and without safe guiding enabled for task $\phi_{\mathtt{Pen1}}$. (left) Maximum absolute value of ECBF compenators during learning. (right) Magnified ECBF compensators during evaluation.}
\end{figure}

Fig.\ref{fig:pendulum_reward-angle} (a) and (b) compare the mean reward and absolute value of maximum angles generated during the learning process via different baselines, respectively.
Fig.\ref{fig:pendulum_reward-angle} (a) shows that the safe modular learning with exploration guiding is more efficient for finding the optimal policies since the guiding module enforces the exploration within the set of safety policies.
In addition, Fig.\ref{fig:pendulum_reward-angle} (b) shows the importance of the safe module during learning.
At the same time, Fig.~\ref{fig:pendulum_ECBF} focuses on illustrating the effectiveness of exploration guiding and compares the baselines during learning and evaluation processes.
By assigning negative rewards when the ECBF controllers are involved, it shows that our algorithm enhances the RL-agent updating within the safe set and leads the output of ECBF compensators to dramatically decay. 

\subsection{Robotic Autonomous Vehicle}
\label{subsec:benchmark}

We tested our algorithms on controlling a robotic car-like model.
Let $s=[p_{x},p_{y}, \theta, v]^{T}$, $u=[a,\omega]^{T}$, and $L$ denote the state (position, heading, velocity), control variables (acceleration, steering angle), and length of the vehicle, respectively.
The dynamics of the model is
$$\dot{s}=[v\cos\theta,v\sin\theta, \frac{\tan\omega}{L}, K_{u}a]^{T}$$
where $K_{u}$ is the physical constant. To model uncertainty,  we set $25\%$ error in the parameters $L$ and $K_{u}$ and add Gaussian noise to the accelerations.

\begin{figure}
	\centering{}\includegraphics[scale=0.26 ]{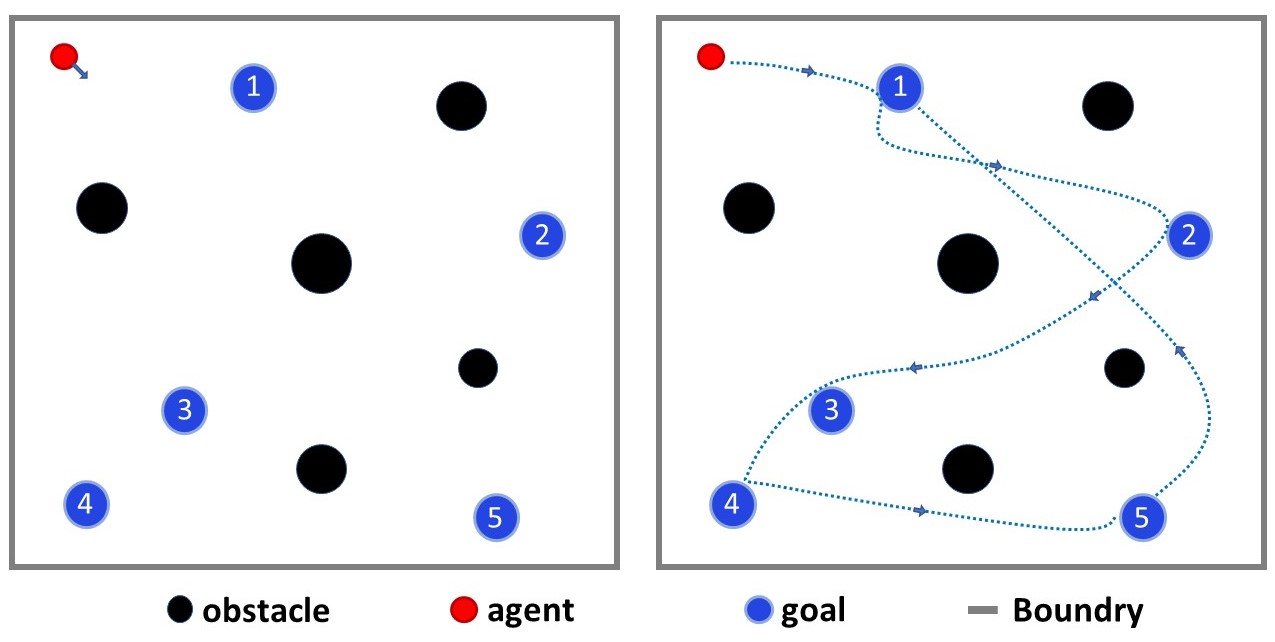}\caption{\label{fig:gym} The left figure shows the initial state of the system, where the red, blue, and black circles represent the mobile robot, the goals, and the obstacles, respectively. In addition, the RL-agent is also required to remain within the rectangular box. The right figure shows the simulated trajectory for the task $\varphi_{Gym_{1}}$ with a repetitive pattern.}
\end{figure}

\textbf{Particle Gym} We first test our algorithm for the Particle Gym as shown in Fig.~\ref{fig:gym}. The two missions require the autonomous vehicle (red circle) to sequentially visit the blue regions numbered $1$ to $5$ over infinite and finite horizons, respectively.
The safety constraint is to always avoid the black obstacles and stay within the rectangular workspace, which is encoded as multiple decentralized ECBFs. The LTL task over the infinite horizon is
\[
\phi_{Gym_{1}}=\oblong(\left(\lozenge\mathtt{R_{1}}\land\lozenge\mathtt{\left(\mathtt{R_{2}}\land\lozenge\mathtt{\ldots\land \mathtt{R_{5}}}\right)}\right))\land\oblong\phi_{\mathtt{C_{3}}},
\]
where $\mathtt{R_{i}}$ is $i$-th blue region indexed with number $i$, and $\varphi_{\mathtt{C_{3}}}$ represents the safety requirements associated with ECBFs. The simulated trajectory of $\varphi_{Gym_{1}}$ for one round of the repetitive satisfaction is shown in Fig.~\ref{fig:gym} (b). Also the task over finite horizon is $\phi_{Gym_{2}}=\left(\lozenge\mathtt{R_{1}}\land\lozenge\mathtt{\left(\mathtt{R{2}}\land\lozenge\mathtt{\ldots\land \mathtt{R_{5}}}\right)}\right)\land\oblong\phi_{\mathtt{C_{3}}}$. The results of mean reward collection for the task $\varphi_{Gym_{1}}$ during training compared with two baselines are shown in Fig.~\ref{fig:Reward_Car} (left), which illustrates better performance of the modular architecture and effectiveness of exploration guiding.

\begin{figure}
	\centering{}\includegraphics[scale=0.30 ]{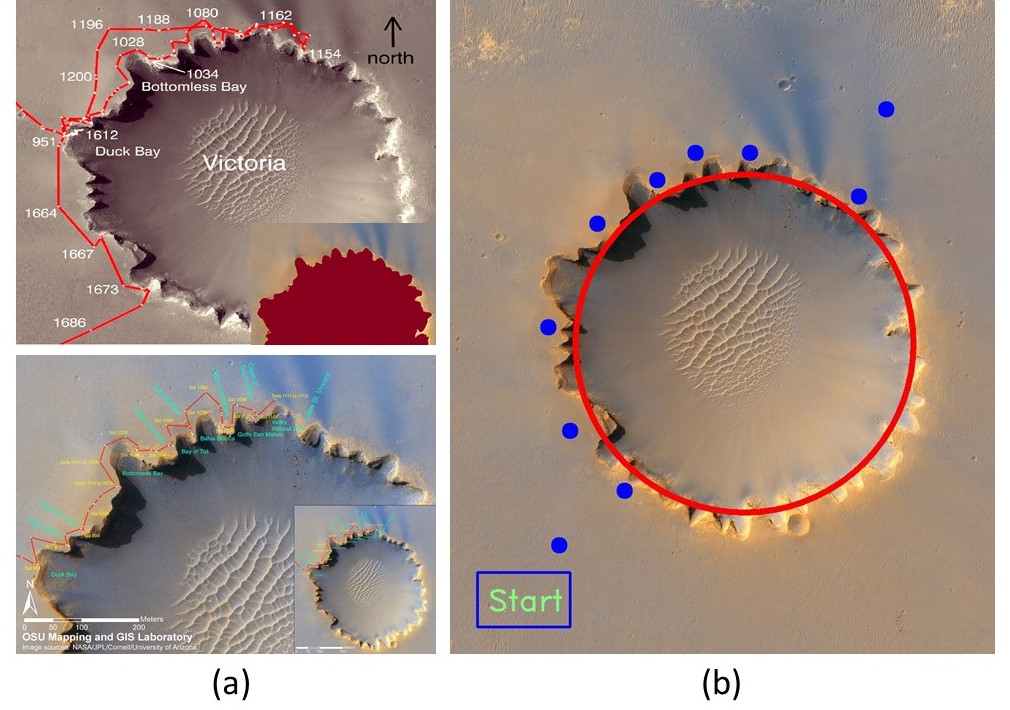}\caption{\label{fig:Mars} {\color{black} LTL Tasks $\phi_{V_{1}}$ and $\phi_{V_{2}}$ for Mars Exploration of the Victoria Crater requires the autonomous rover to visit $10$ spots in a sequential manner while never entering the crater marked as the red circle.} }
\end{figure}

\textbf{Robtic Mars Rover:}  Lastly, we tested our algorithms in a large scale robotic environment,
and used motion planning to complete complex exploration missions using
satellite images as shown in Fig.~\ref{fig:Mars}.
The missions are to explore areas around the Victoria Crater \cite{Squyres2009} shown in Fig.~\ref{fig:Mars}
(a), an impact crater located near the equator of Mars.
Layered sedimentary rocks are exposed along the wall of the crater, providing
information about the ancient surface condition of Mars.
The mission requires visiting all spots along the path of the Mars
rover Opportunity shown in Fig.~\ref{fig:Mars} (b),
and avoiding the unsafe areas (red circle).
The LTL specification of the mission over the infinite horizon is
\[
\phi_{V_{1}}=\oblong\left(\lozenge\mathtt{V_1}\land\lozenge\left(\mathtt{V_2}\land\lozenge(\ldots\land \lozenge (\mathtt{V_{10}}\land \lozenge\mathtt{V_{Start}}\right)\right){\color{orange}})\land\oblong\phi_{\mathtt{C_{4}}},
\]
where $\mathtt{V_{i}}$ denotes the $i$-th target (blue spot) numbered from bottom-left to top-right.
$\varphi_{\mathtt{C_{4}}}$ represents safety requirements (barrier functions) s.t. the agent always avoids the unsafe crater area marked with a red circle in Fig.~\ref{fig:Mars} (b).
The description of the overall task in English is "visit the targets $1$ to $10$ and then return to the start position, repetitively, while avoiding the unsafe regions".
Similarly, the finite horizon task is $\phi_{V_{2}}=\lozenge\mathtt{V_1}\land\lozenge(\mathtt{V_2}\land\lozenge(\ldots\land \lozenge(\mathtt{V_{10}}\land\lozenge\mathtt{V_{Start}})\ldots)\land\oblong\phi_{\mathtt{C_{4}}}$. The results of mean reward collection for the task $\varphi_{V_1}$ during training compared with two baselines are shown in Fig.~\ref{fig:Reward_Car} (right).
It shows the better performance and effectiveness of safe modular DDPG with exploration guiding enabled.

\begin{figure}
	\centering{}\includegraphics[scale=0.32 ]{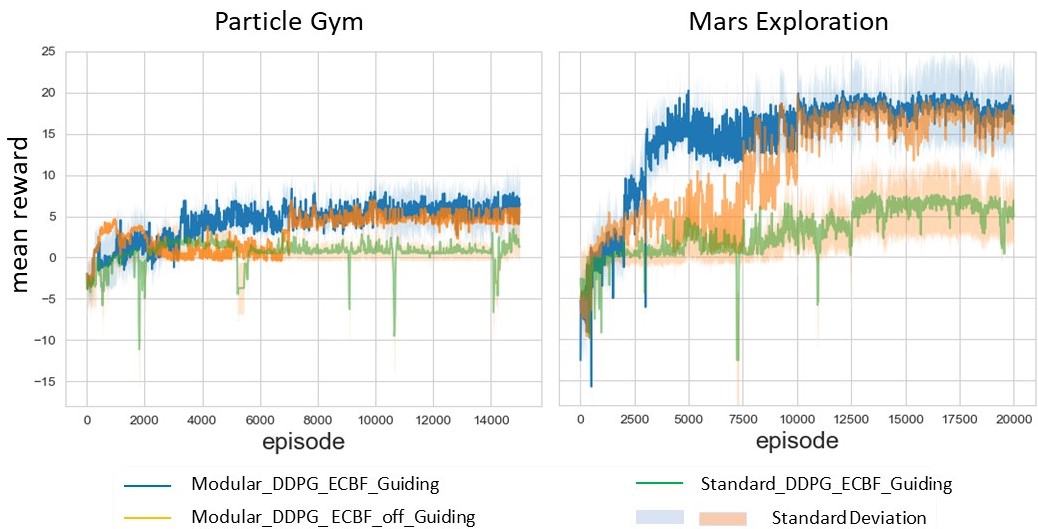}\caption{\label{fig:Reward_Car} Mean reward collection during training for particle gym (left) and Mars exploration (right). We analyze the results of tasks $\varphi_{Gym_{1}}$ and $\varphi_{V_{1}}$ over infinite horizons to demonstrate the effectiveness and efficiency of the modular architecture and safe learning process.}
\end{figure}

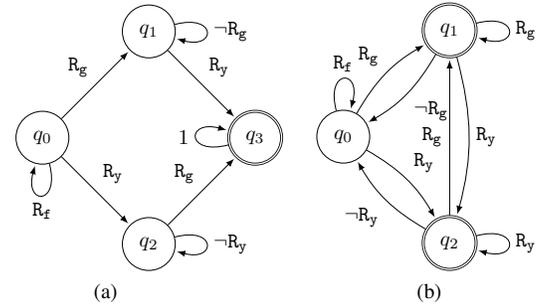
\begin{figure}[!t]\centering
	\subfloat[]{{
		\scalebox{.8}{
			\begin{tikzpicture}[shorten >=1pt,node distance=2.5cm,on grid,auto] 
			\node[] (s_0)   {}; 
			\node[state] (s_0) [above right=of s_0]  {$q_0$};
			\node[state] (s_1) [above right=of s_0]  {$q_1$};
			\node[state] (s_2) [below right=of s_0]  {$q_2$};
			\node[state,accepting] (s_3) [below right=of s_1]  {$q_3$};
			\path[->] 
			(s_0) edge node {$\mathtt{R}_{\mathtt{g}}$} (s_1)
			(s_0) edge [loop below] node {$\mathtt{R}_{\mathtt{f}}$} (s_0)
			(s_1) edge [loop right] node {$\lnot\mathtt{R}_{\mathtt{g}}$} (s_1)
			(s_2) edge [loop right] node {$\lnot\mathtt{R}_{\mathtt{y}}$} (s_2)
			(s_0) edge node {$\mathtt{R}_{\mathtt{y}}$} (s_2)
			(s_1) edge node {$\mathtt{R}_{\mathtt{y}}$} (s_3)
			(s_2) edge node {$\mathtt{R}_{\mathtt{g}}$} (s_3)
			(s_3) edge [loop left] node {$1$} (s_3);
			\end{tikzpicture}
			}
			}}
	\subfloat[]{{
	    \scalebox{.8}{
			\begin{tikzpicture}[shorten >=1pt,node distance=2.5cm,on grid,auto] 
			\node[state] (q_0)   {$q_0$}; 
			\node[state,accepting] (q_1) [above right=of q_0]  {$q_1$};
			\node[state,accepting] (q_2) [below right=of q_0]  {$q_2$};
			\path[->] 
			(q_0) edge [bend left=15] node {$\mathtt{R}_{\mathtt{g}}$} (q_1)
			(q_1) edge [bend left=15] node {$\lnot\mathtt{R}_{\mathtt{g}}$} (q_0)
			(q_0) edge [loop above] node {$\mathtt{R}_{\mathtt{f}}$} (q_0)
			(q_1) edge [loop right] node {$\mathtt{R}_{\mathtt{g}}$} (q_1)
			(q_2) edge [loop right] node {$\mathtt{R}_{\mathtt{y}}$} (q_2)
			(q_0) edge [bend left=15] node {$\mathtt{R}_{\mathtt{y}}$} (q_2)
			(q_2) edge [bend left=15] node {$\lnot\mathtt{R}_{\mathtt{y}}$} (q_0)
			(q_1) edge [bend left=15] node {$\mathtt{R}_{\mathtt{y}}$} (q_2)
			(q_2) edge node {$\mathtt{R}_{\mathtt{g}}$} (q_1);
			\end{tikzpicture}
			}
			}}
		\caption{\label{fig:automaton_general}  LDGBA of LTL formulas where $\mathtt{R}_{\mathtt{g}}$, $\mathtt{R}_{\mathtt{y}}$ and $\mathtt{R}_{\mathtt{f}}$  denote $\mathtt{R}_{\mathtt{green}}$, $\mathtt{R}_{\mathtt{yellow}}$, and free space $\lnot(\mathtt{R}_{\mathtt{yellow}}\land\mathtt{R}_{\mathtt{green}})$, respectively. (a)
		$\phi_{g_{2}}=\lozenge\mathtt{R}_{\mathtt{green}}\land\lozenge\mathtt{R}_{\mathtt{yellow}}$ where $q_{3}$ is the accepting state. (b) $\phi_{g_{1}}=\oblong\lozenge\mathtt{R}_{\mathtt{green}}\land\oblong\lozenge\mathtt{R}_{\mathtt{yellow}}$, where there are two accepting sets $F=\{\{q_1\}, \{q_2\}\}$. }
	\end{figure}

\begin{figure*}
	\centering{}\includegraphics[scale=0.40 ]{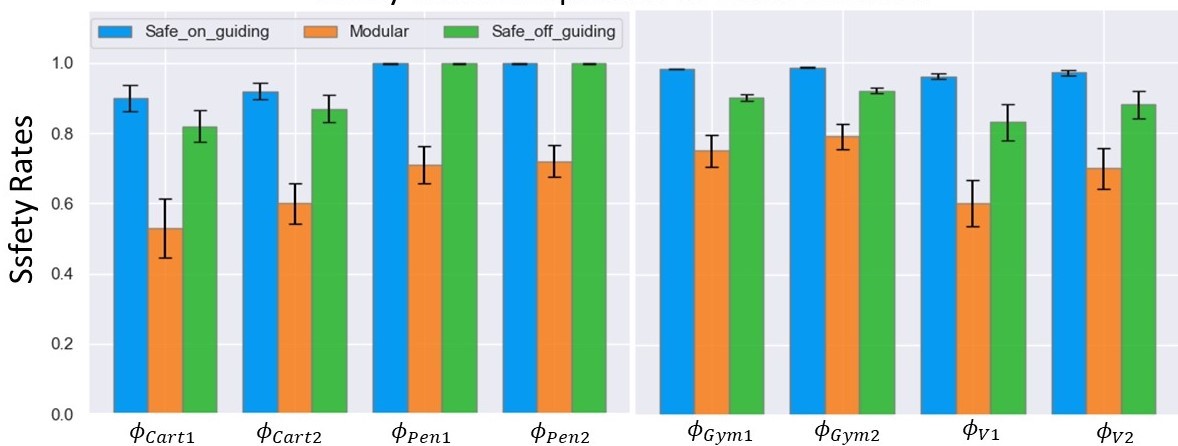}\caption{\label{fig:safety_rate} Safety rates analysis for each environments that includes corresponding infinite-horizons and finite-horizons tasks. The results are all conducted using modular RL algorithm and enabling different modules, and are taken over 10 independent learning trials.}
\end{figure*}

\begin{table*}
	\caption{\label{tab:training}Training time analysis of algorithms with and without safe module (ECBF-based control and exploration guiding).}
	\centering{}\resizebox{0.9\textwidth}{!}{
	\begin{tabular}{|ccc|cccc|c|c}
	    \hline
	    \multicolumn{3}{|c|}{Tasks and Training Parameters} & \multicolumn{4}{|c|}{Training Time (hour)} \\ 
		\hline 
		LTL Task & Maximum steps & Episode & Modular DDPG & Safe Modular DDPG & Standard DDPG & Safe Standard DDPG  \\ 
		\hline
		$\varphi_{B1}$ & $300$ & $10000$ & $5.5$ & $7.2$ & $5.0$ & $7.1$ \\
		
		$\varphi_{B2}$ & $300$ & $10000$ & $4.9$ & $6.3$ & $4.6$ & $6.3$ \tabularnewline
		
		$\varphi_{C1}$ & $200$ & $4000$ & $3.6$ & $4.0$ & $3.5$ & $4.0$ \tabularnewline
		
		$\varphi_{C2}$ & $200$ & $4000$ &$2.0$ & $2.5$ & $2.1$ & $2.8$ \tabularnewline
		
		$\varphi_{Gym1}$ & $1000$ & $15000$ & $9.7$ & $11.9$ & $9.3$ & $12.1$ \tabularnewline
		
		$\varphi_{Gym2}$ & $1000$ & $15000$ & $7.3$ & $10.5$ & $8.5$ & $12.0$ \tabularnewline
		
		$\varphi_{V_{1}}$ & $5000$ & $20000$ & $22.7$ & $40.4$ & $26.9$ & $41.6$ \tabularnewline
		
		$\varphi_{V_{2}}$ & $5000$ & $20000$ & $14.3$ & $18.2$ & $21.0$ & $28.5$ \tabularnewline
	
		\hline 
	\end{tabular}}
\end{table*}

\begin{figure}
	\centering{}\includegraphics[scale=0.35 ]{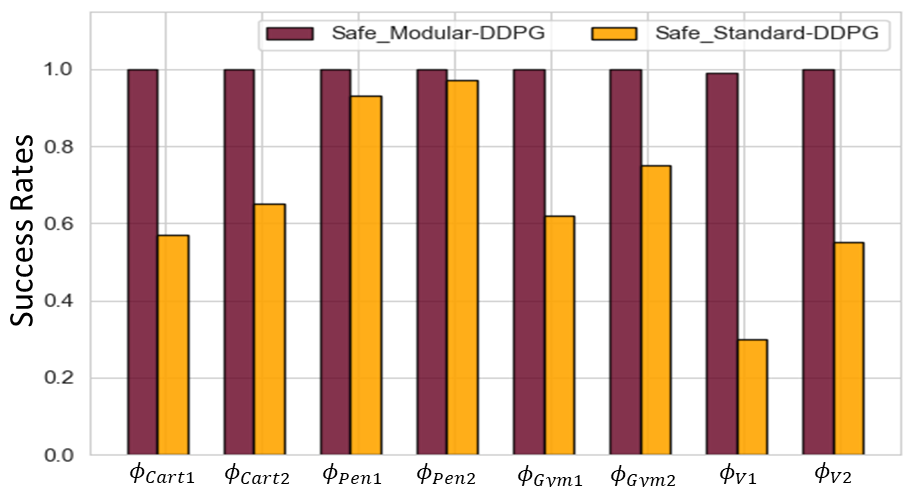}\caption{\label{fig:success_rate} Performance evaluation through success rates for all tasks. Every evaluation is conducted with the same steps in Table~\ref{tab:training}.}
\end{figure}

{\color{black}
\subsection{Discussion on LTL\label{subsec:general LTL formula}} 

This section discusses performance aspects of our framework with respect to LTL formula complexity.
First, we can observe that the some of above tasks, e.g., $\phi_{Gym_{1}}$, $\phi_{Gym_{2}}$, $\phi_{V_{1}}$, and $\phi_{V_{2}}$ are all in the form of predefined sequential orders.
Their decomposition is pre-defined before training.

Some LTL formulas are specified without specific sequential poses. For instances, the task $\phi_{g_{1}}=\oblong\lozenge\mathtt{R}_{\mathtt{green}}\land\oblong\lozenge\mathtt{R}_{\mathtt{yellow}}$ in $\phi_{\mathtt{Cart1}}$ and $\phi_{\mathtt{Pen1}}$, and its finite horizon form $\phi_{g_{2}}=\lozenge\mathtt{R}_{\mathtt{green}}\land\lozenge\mathtt{R}_{\mathtt{yellow}}$ in $\phi_{\mathtt{Cart2}}$ and $\phi_{\mathtt{Pen2}}$.
The LDGBA of $\phi_{g_{2}}$ and $\phi_{g_{1}}$ are shown in Fig.~\ref{fig:automaton_general} (a) and (b), respectively. We can observe that for both formulas, there exist multiple automata traces satisfying the acceptance condition, resulting in multiple valid decomposition choices.

Since the modular architecture employs automaton states to decompose the global task into a sequence of sub-tasks during training,
the combinatorial choices raise optimality challenges regarding the learning-based task decomposition, which replies on the exploration of deep RL.
Such a compositional optimality becomes more uncontrollable when the task is more complex, e.g., a task
$\phi_{V_{3}}=\oblong\lozenge\mathtt{V_1}\land\oblong\lozenge\mathtt{V_2}\land\oblong\lozenge\mathtt{V_3}\ldots\land\oblong\lozenge\mathtt{V_{10}}\land\oblong\phi_{\mathtt{C_{4}}}$. Even though such a task can be still decomposed during training and learned, the final results in the sense of task decomposition may not be the optimal one.
Our future work will consider optimally decomposing the task before training to effectively and efficiently learn the general LTL expressions.

}

\subsection{Complexity and Performance Analysis\label{subsec:performance}} 

First, we define safety rate as the number of safe episodes versus all episodes.
Fig.~\ref{fig:safety_rate} shows the safety rates for all tasks over all environments through different baselines.
It shows the benefits of the ECBF-based safe module and the improvement due to exploration guiding.
We analyzed the training complexity for various baselines shown in Table~\ref{tab:training}.
From the perspective of safe learning, the proposed safe module requires solving a quadratic program at each step.
The training time increased for both safe modular and safe standard DDPG methods.
This is reasonable since the algorithm needs to check whether controllers are safe at each step.
As for the modular architecture, even though it adopts several distributed actor-critic neural network pairs, they are concurrently trained, and each of them is only responsible for a sub-task.
Consequently, the training time is mainly influenced by the number of steps and episodes for both safe modular and standard DDPG.
For complex tasks, e.g., $\varphi_{V_{2}}$, the modular architecture can complete the task faster (terminate the episode earlier) during learning, and reduce the training time. 

To highlight the performance of the modular structure, we compare the safe modular DDPG and safe standard DDPG methods both with exploration guiding enabled.
We take $200$ runs applying the learned model, and analyze the success rate for all aforementioned tasks.
Due to growing dimensions of automaton structure for more complex tasks and a limited number of episodes and steps during training in practice, it becomes difficult for standard DDPG to explore entire tasks over an infinite horizon, and 
to recognize completion (satisfaction) of sub-parts of tasks due to variance issues inherent to policy gradient methods.
As shown in Fig.~\ref{fig:success_rate}, we conclude that modular architecture has better and more stable performance, whereas the standard DDPG yields poor performance for complex tasks with repetitive patterns (infinite horizons). {\color{black} Due to additive noises in the dynamics resulting in unrecognized interactions, there may exist cases where success rates are not perfect. In this case, the modular architecture allows to inspect each sub-tasks, and extract the ones to analyze, and keep training to improve their performances.}

\section{Conclusions}
\label{sec:conclusion}

Achievement of safe critical requirements during learning is a challenging problem with significant real-world robotic applicability. Part of the challenge stems from the uncertain and unknown dynamical systems and the impact of the exploration on optimal solutions.
Such problems become even more difficult, but more meaningful when RL-agents are tasked to accomplish complex human instructions over infinite horizons and continuous space.
The main difficulties are due to the nature of nonlinear regression to recognize each stage of task satisfaction and the need for large training episodes.
Therefore, we propose the ECBF-based safe RL framework combined with GPs for estimation of the nominal systems, and RL agent training guided by the LTL specifications describing high-level complex tasks.
These features are crucial in employing reinforcement learning in physical applications, where humans are able to formulate advanced objectives specified in the formal language.
They are also important in the case where problems require efficient computation and effective learning performance.

This framework combines model-free deep RL, GP-based ECBF control, and automata theory of compositional LTL syntax.
On the training side, the designed E-LDGBA allows us to apply the deterministic policy and overcome the issue of sparse rewards, and the reward shaping technique further enhances the dense rewards.
On the evaluation side, by fully utilizing the automaton structure, we propose an innovative modular DDPG architecture that relies on distributed neural networks to improve the performance of the learning results for complex tasks.
We also propose a novel approach by integrating the sink components of LTL automata and ECBF perturbations to enforce the guiding of exploration.
A significant formal result is that the above modules (safe and modular) do not impact desired behaviors with respect to original optimal solutions, i.e., satisfying LTL with maximum probability in the limit. 

We tested the overall algorithm in various control systems and demonstrate its benefits by comparing it with several baselines.
Our results are encouraging for several future directions.
This work assumes we are given a valid safe set of CBFs that can be rendered forward invariant, which opens the question of whether we can learn the CBFs in addition to the controllers.
{\color{black} Since deep policy gradient algorithms mainly rely on exploration for learning, which is sensitive to environmental settings, future work will investigate improving exploration techniques to handle complex environments. } Furthermore, future research will also consider multi-agent cooperative tasks and bridging the gap between simulations and real-world applications. 

\section{Acknowledgement}
The authors thank Richard Cheng and Hosein Hasanbeig for their helpful discussions.

\begin{appendices}

\section{supplementary materials}

\subsection{Proof of Lemma~\ref{lem:language}\label{Appendix:E-LDGBA}}

We prove~\eqref{lem:language} by showing that $\mathcal{L}(\mathcal{\overline{A}}_{\phi})\supseteq\mathcal{L}(\mathcal{A}_{\phi})$ and $\mathcal{L}(\mathcal{\overline{A}}_{\phi})\subseteq\mathcal{L}(\mathcal{A}_{\phi})$.

\textbf{Case 1:}  $\mathcal{L}(\mathcal{\overline{A}}_{\phi})\supseteq\mathcal{L}(\mathcal{A}_{\phi})$: For any accepted word $\boldsymbol{\omega}=\alpha_{0}\alpha_{1}\ldots\in\mathcal{L}(\mathcal{A}_{\phi})$, there exists a corresponding run $\boldsymbol{r}=q_{0}\alpha_{0}q_{1}\alpha_{1}\ldots$ of $\mathcal{A}_{\phi}$ s.t.
\\
\begin{equation}
\inf\left(\boldsymbol{r}\right)\cap F_{i}\neq\emptyset, \forall i\in\left\{ 1,\ldots, f\right\}. \label{eq:Lemma_case1}
\end{equation}
\\
For the run $\boldsymbol{r}$, we can construct a sequence $\boldsymbol{\overline{r}}=\overline{q}_{0}\alpha_{0}\overline{q}_{1}\alpha_{1}\ldots$ by adding to each state $q$ the set $T$, which is updated via~\eqref{eq:Trk-fontier} after each transition.
It can be verified that such a run $\boldsymbol{\overline{r}}$ is a valid run of $\mathcal{\overline{A}}_{\phi}$ based on Def.~\ref{def:E-LDGBA}. According to~\eqref{eq:Lemma_case1}, since the tracking-frontier set $T$ will be reset once all accepting sets have been visited, it holds $\inf\left(\boldsymbol{\overline{r}}\right)\cap \overline{F_{i}}\neq\emptyset, \forall i\in\left\{ 1,\ldots, f\right\}$, i.e., $\boldsymbol{\omega}\in\mathcal{L}(\mathcal{\overline{A}}_{\phi})$.  

\textbf{Case 2:} $\mathcal{L}(\mathcal{\overline{A}}_{\phi})\subseteq\mathcal{L}(\mathcal{A}_{\phi})$: Similarly, for any accepted word $\boldsymbol{\overline{\omega}}=\overline{\alpha}_{0}\overline{\alpha}_{1}\ldots\in\mathcal{L}(\mathcal{\overline{A}}_{\phi})$, there exists a corresponding run $\boldsymbol{\overline{r}}=\overline{q}_{0}\overline{\alpha}_{0}\overline{q}_{1}\overline{\alpha}_{1}\ldots$ of $\mathcal{\overline{A}}_{\phi}$ s.t.
\\
\begin{equation}
\inf\left(\boldsymbol{\overline{r}}\right)\cap \overline{F_{i}}\neq\emptyset, \forall i\in\left\{ 1,\ldots, f\right\}. \label{eq:Lemma_case2}
\end{equation}
\\
For the run $\boldsymbol{\overline{r}}$, we can construct a sequence $\boldsymbol{r}=q_{0}\overline{\alpha}_{0}q_{1}\overline{\alpha}_{1}\ldots$ by projecting each state $\overline{q}=(q,T)$ onto $q$. It follows that such a run $\boldsymbol{r}$ is a valid run of $\mathcal{A}_{\phi}$ based on Def.~\ref{def:E-LDGBA}. According to~\eqref{eq:Lemma_case2}, it holds $\inf\left(\boldsymbol{r}\right)\cap F_{i}\neq\emptyset, \forall i\in\left\{ 1,\ldots, f\right\}$, i.e., $\boldsymbol{\overline{\omega}}\in\mathcal{L}(\mathcal{A}_{\phi})$.

\subsection{Properties of P-MDP\label{subsec:properties}}

A sub-MDP $\mathcal{P}_{\left(X',U'\right)}$ of $\mathcal{P}$ is a pair $(X', U')$ where $X'\subseteq X$ and $U'$ is a action sub-space of $U^{\mathcal{P}}$ such that (i) $X'\neq\emptyset$, and $U'(x)\neq\emptyset$, $\forall x\in X'$;
(ii) $\forall x \in X'$ and $\forall u\in U'(x)$, if $p^{\mathcal{P}}(x,u,x')>0$ then $x'\in X'$.
The induced graph of $\mathcal{M}{}_{\left(X',U'\right)}$ is a directed graph $\mathcal{G}_{\left(X',U'\right)}$, where $X'$ is regarded as a set of nodes, and if $p^{\mathcal{P}}(x,u,x')>0$ for some $u\in U'(x)$ with $x,x'\in X'$, then there exists an edge between $x$ and $x'$ in $\mathcal{G}_{\left(X',U'\right)}$. A sub-MDP is called a strongly connected component (SCC) if its induced graph is strongly connected, i.e., for all pairs of nodes $x,x' \in X'$, there is a path from $x$ to $x'$. A bottom strongly connected component (BSCC) is an SCC
from which no outside state is reachable by applying the restricted action space. 

\begin{defn}
A Markov
chain $MC_{\mathcal{P}}^{\boldsymbol{\pi}}$ of the $\mathcal{P}$ is a sub-MDP of $\mathcal{P}$ induced by a policy $\boldsymbol{\pi}$ \cite{Baier2008}.
\end{defn}

\begin{defn}
A sub-MDP $\mathcal{P}{}_{\left(X',U'\right)}$ 
is called an end component (EC) of $\mathcal{P}$ if its induced graph is a BSCC. An EC $\mathcal{P}_{\left(X',U'\right)}$ is called a maximal end
component (MEC) if there is no other EC $\mathcal{P}_{\left(X'',U''\right)}$
such that $X'\subseteq X''$ and $U'\left(x\right)\subseteq U''\left(s\right)$,
$\forall x\in X'$.
\end{defn}

Consider a sub-MDP $\mathcal{P}_{\left(X',U'\right)}$ of $\mathcal{P}$, where
$X'\subseteq X$ and $U'\subseteq U^{\mathcal{P}}$. If $\mathcal{P}'_{\left(X',U'\right)}$
is a maximum end component (MEC) \cite{Baier2008} of $\mathcal{P}$ and $X'\cap F_{i}^{\mathcal{P}}\neq\emptyset$,
$\forall i\in\left\{ 1,\ldots f\right\} $, then $\mathcal{P}_{\left(X',U'\right)}$
is called an accepting maximum end component (AMEC) of $\mathcal{P}$.
Once a path enters an AMEC, the subsequent path will stay within it
by taking restricted actions from $U'$. Satisfying task $\phi$ is equivalent to reaching an AMEC \cite{Baier2008}.

\begin{defn}\label{def:induced_markov_chain}
\cite{Durrett1999} States of any Markov
chain $MC_{\mathcal{P}}^{\boldsymbol{\pi}}$ under policy $\boldsymbol{\pi}$ can be represented by a disjoint union of a transient
class $\ensuremath{\mathcal{T}_{\boldsymbol{\pi}}}$ and $n_R$ closed
irreducible recurrent classes $\ensuremath{\mathcal{R}_{\boldsymbol{\pi}}^{j}}$,
$j\in\left\{ 1,\ldots,n_{R}\right\} $, where a class is a set of states.
\end{defn}

\begin{lem}
\label{lemma:accepting set}\cite{Cai2021modular} Given an P-MDP $\mathcal{P}=\mathcal{M}\times\mathcal{{A}}_{\phi}$
, the recurrent class $R_{\boldsymbol{\pi}}^{j}$ of $MC_{\mathcal{P}}^{\boldsymbol{\pi}}$,
$\forall j\in\left\{ 1,\ldots,n_R\right\} $, induced by $\pi$ satisfies
one of the following conditions: (i) $\ensuremath{R_{\boldsymbol{\pi}}^{j}}\cap F_{i}^{\mathcal{P}}\neq\emptyset,\forall i\in\left\{ 1,\ldots f\right\} $,
or (ii) $R_{\boldsymbol{\pi}}^{j}\cap F_{i}^{\mathcal{P}}=\emptyset,\forall i\in\left\{ 1,\ldots f\right\} $. 
\end{lem}

Under policy $\boldsymbol{\pi}$, the notation
$\ensuremath{\mathcal{T}_{\boldsymbol{\pi}}}$ represents the behaviors before entering into MECs, and $\ensuremath{\mathcal{R}_{\boldsymbol{\pi}}^{j}}$ involves the behaviors after entering into a MEC.

\subsection{Proof of Theorem~\ref{thm:safe_RL}\label{apped:thm:safe_RL}}

For any policy $\boldsymbol{\pi}$, $MC_{\mathcal{P}}^{\boldsymbol{\pi}}=\ensuremath{\mathcal{T}_{\boldsymbol{\pi}}}\sqcup\ensuremath{\mathcal{R}_{\boldsymbol{\pi}}^{1}\sqcup\ensuremath{\mathcal{R}_{\boldsymbol{\pi}}^{2}}\ldots\ensuremath{\mathcal{R}_{\boldsymbol{\pi}}^{n_{R}}}}.$
Let $\boldsymbol{U}_{\boldsymbol{\pi}}=\left[\begin{array}{ccc}
U^{\boldsymbol{\pi}}\left(x_{0}\right) & U^{\boldsymbol{\pi}}\left(x_{1}\right) & \ldots\end{array}\right]^{T}\in\mathbb{R}^{\left|X\right|}$ denote the stacked expected return under policy $\boldsymbol{\pi}$, which we reorganize
\\
\begin{equation}\small
\begin{aligned}\left[\begin{array}{c}
\boldsymbol{U}_{\boldsymbol{\pi}}^{tr}\\
\boldsymbol{U}_{\boldsymbol{\pi}}^{rec}
\end{array}\right]= & \stackrel[n=0]{\infty}{\sum}\left(\stackrel[j=0]{n-1}{\prod}\left[\begin{array}{cc}
\boldsymbol{\boldsymbol{\gamma}}_{\boldsymbol{\pi}}^{\ensuremath{\mathcal{T}}} & \boldsymbol{\boldsymbol{\gamma}}_{\boldsymbol{\pi}}^{\ensuremath{tr}}\\
\boldsymbol{0}_{\sum_{i=1}^{m}N_{i}\times r} & \boldsymbol{\boldsymbol{\gamma}}_{\boldsymbol{\pi}}^{rec}
\end{array}\right]\right)\\
 & \cdot\left[\begin{array}{cc}
\boldsymbol{P}_{\boldsymbol{\pi}}\left(\ensuremath{\mathcal{T}},\ensuremath{\mathcal{T}}\right) & \boldsymbol{P}_{\boldsymbol{\pi}}^{tr}\\
\boldsymbol{0}_{\sum_{i=1}^{m}N_{i}\times r} & \boldsymbol{P}_{\boldsymbol{\pi}}\left(\mathcal{R},\mathcal{R}\right)
\end{array}\right]^{n}\left[\begin{array}{c}
\boldsymbol{R}_{\boldsymbol{\pi}}^{tr}\\
\boldsymbol{R}_{\boldsymbol{\pi}}^{rec}
\end{array}\right],
\end{aligned}
\label{eq: utility_function}
\end{equation}
\\
where $\boldsymbol{U}_{\boldsymbol{\pi}}^{tr}$ and $\boldsymbol{U}_{\boldsymbol{\pi}}^{rec}$
are the expected return of states in transient and recurrent classes
under policy $\boldsymbol{\pi}$, respectively. In \eqref{eq: utility_function},
$\boldsymbol{P}_{\boldsymbol{\pi}}\left(\ensuremath{\mathcal{T}},\ensuremath{\mathcal{T}}\right)\in\mathbb{R}^{r\times r}$
is the probability transition matrix between states in $\ensuremath{\mathcal{T}_{\boldsymbol{\pi}}}$,
and $\boldsymbol{P}_{\boldsymbol{\pi}}^{tr}=\left[P_{\boldsymbol{\pi}}^{tr_{1}}\ldots P_{\boldsymbol{\pi}}^{tr_{m}}\right]\in\mathbb{R}^{r\times\sum_{i=1}^{m}N_{i}}$
is the probability transition matrix where $P_{\boldsymbol{\pi}}^{tr_{i}}\mathbb{\in R}^{r\times N_{i}}$
represents the transition probability from a transient state in $\ensuremath{\mathcal{T}_{\boldsymbol{\pi}}}$
to a state of $\mathcal{R}_{\boldsymbol{\pi}}^{i}$. The $\boldsymbol{P}_{\boldsymbol{\pi}}\left(\mathcal{R},\mathcal{R}\right)$
is a diagonal block matrix, where the $i$-th block is a $N_{i}\times N_{i}$
matrix containing transition probabilities between states within $\mathcal{R}_{\boldsymbol{\pi}}^{i}$.
Note that $\boldsymbol{P}_{\boldsymbol{\pi}}\left(\mathcal{R},\mathcal{R}\right)$
is a stochastic matrix since each block matrix is a stochastic matrix
\cite{Durrett1999}. Similarly, the rewards $\boldsymbol{\boldsymbol{R}}_{\boldsymbol{\pi}}$
can also be partitioned into $\boldsymbol{R}_{\boldsymbol{\pi}}^{tr}$
and $\boldsymbol{R}_{\boldsymbol{\pi}}^{rec}$.

\textbf{Proof of Contradictions} The following proof is based on contradictions. Suppose there exists
a policy $\boldsymbol{\pi}^{*}$ that optimizes the expected return,
but derives the system intersecting with $\overline{Q}_{unsafe}$ with non-zero probability. Based on Lemma~\ref{lemma:accepting set}, the following is true:
$F_{k}^{\mathcal{P}}\subseteq\ensuremath{\mathcal{T}_{\boldsymbol{\pi}^{*}}},\forall k\in\left\{ 1,\ldots, f\right\} $,
where $\ensuremath{\mathcal{T}_{\boldsymbol{\pi}^{*}}}$ denotes the
transient class of Markov chain induced by $\boldsymbol{\pi}^{*}$ on $\mathcal{P}$.

Consider two types of states $x_{R}\in\mathcal{R}_{\boldsymbol{\pi}}^{j}$ and $x_{T}\in\mathcal{T}_{\boldsymbol{\pi}}^{j}$.
Let $\boldsymbol{P}_{\boldsymbol{\pi}}^{x_{R}R_{j}}$ denote
a row vector of $\boldsymbol{P}_{\boldsymbol{\pi}}^{n}\left(\mathcal{R},\mathcal{R}\right)$
that contains the transition probabilities from $x_{R}$ to the states
in the same recurrent class $\mathcal{R}_{\pi}^{j}$ after $n$
steps. The expected return of $x_{R}$ and  $x_{T}$  under  $\boldsymbol{\pi}$
are then obtained from~\eqref{eq: utility_function} respectively as 
\begin{align*}
U_{\boldsymbol{\pi}}^{rec}\left(x_{R}\right)=\stackrel[n=0]{\infty}{\sum}\gamma^{n}\left[\boldsymbol{0}_{k_{1}}^{T}\:\boldsymbol{P}_{\pi}^{x_{R}R_{j}}\:\boldsymbol{0}_{k_{2}}^{T}\right]\boldsymbol{R}_{\boldsymbol{\pi}}^{rec},
\\
\boldsymbol{U}_{\boldsymbol{\pi}}^{tr}>\underline{\gamma}^{n}\ldotp\boldsymbol{P}_{\boldsymbol{\pi}}^{tr}\boldsymbol{P}_{\boldsymbol{\pi}}^{n}\left(\mathcal{R},\mathcal{R}\right)\boldsymbol{R}_{\boldsymbol{\pi}}^{rec},
\end{align*}
\\
where $k_{1}=\sum_{i=1}^{j-1}N_{i}$, $k_{2}=\sum_{i=j+1}^{n}N_{i}$, $\boldsymbol{P}_{\boldsymbol{\pi}}^{tr}=\left[P_{\boldsymbol{\pi}}^{tr_{1}}\ldots P_{\boldsymbol{\pi}}^{tr_{m}}\right]\in\mathbb{R}^{r\times\sum_{i=1}^{m}N_{i}}$
is the probability transition matrix, and the $\boldsymbol{P}_{\boldsymbol{\pi}}\left(\mathcal{R},\mathcal{R}\right)$
is a diagonal block matrix.

Since $\ensuremath{\mathcal{R}_{\pi^{*}}^{j}}\cap \overline{X}_{unsafe}\neq\emptyset$ where $\overline{X}_{unsafe}$ is introduced in Def.~\ref{def:unsafe_EP-states}, all entries of $\boldsymbol{R}_{\boldsymbol{\pi}^{*}}^{rec}$
are non-positive. We can conclude $U_{\boldsymbol{\pi}^{*}}^{rec}\left(x_{R}\right)\leq 0$.
To show contradiction,  by selecting $\gamma_{F}\shortrightarrow1^{-}$ the following analysis demonstrates the contradiction, i.e., $U_{\bar{\boldsymbol{\pi}}}^{rec}\left(x_{R}\right)>U_{\boldsymbol{\pi}^{*}}^{rec}\left(x_{R}\right)$, where $\bar{\boldsymbol{\pi}}$ is a policy that satisfies the accepting
condition of $\mathcal{P}$:

There exists two cases of the analysis. (i) $x_{R}\in\mathcal{R}_{\bar{\boldsymbol{\pi}}}^{j}$. (ii) $x_{R}\in\mathcal{T}_{\bar{\boldsymbol{\pi}}}$. We show the contradictions for them respectively.

\textbf{Case 1:} If $x_{R}\in\mathcal{R}_{\bar{\boldsymbol{\pi}}}^{j}$,
there exist states such that $x_{\varLambda}\in\mathcal{R}_{\bar{\boldsymbol{\pi}}}^{j}\cap F_{i}^{\mathcal{P}}$.
From Lemma~\ref{lemma:accepting set}, the entries in $\boldsymbol{R}_{\bar{\boldsymbol{\pi}}}^{rec}$
corresponding to the recurrent states in $\mathcal{R}_{\bar{\boldsymbol{\pi}}}^{j}$
have non-negative rewards and at least there exist $f$ states in
$\mathcal{R}_{\bar{\boldsymbol{\pi}}}^{j}$ from different accepting
sets $F_{i}^{\mathcal{R}}$ with positive reward $1-r_{F}$. From \eqref{eq: utility_function},
$U_{\bar{\boldsymbol{\pi}}}^{rec}\left(x_{R}\right)$ can be lower
bounded as 
\\
\begin{equation}
   U_{\bar{\boldsymbol{\pi}}}^{rec}\left(x_{R}\right)  \geq\stackrel[n=0]{\infty}{\sum}\underline{\gamma}^{n}\left(P_{\bar{\boldsymbol{\pi}}}^{x_{R}x_{\varLambda}}r_{F}\right)
>0\label{eq:case1}, 
\end{equation}
\\
where $P_{\bar{\boldsymbol{\pi}}}^{x_{R}x_{\varLambda}}$ is the transition
probability from $x_{R}$ to $x_{\varLambda}$ in $n$ steps. We can
conclude in this case $U_{\bar{\boldsymbol{\pi}}}^{rec}\left(x_{R}\right)>U_{\boldsymbol{\pi}^{*}}^{rec}\left(x_{R}\right)$.

\textbf{Case 2:} If $x_{R}\in\mathcal{T}_{\bar{\boldsymbol{\pi}}}$,
there are no states of any accepting set $F_{i}^{\mathcal{P}}$ in
$\mathcal{T}_{\bar{\boldsymbol{\pi}}}$. As demonstrated in \cite{Durrett1999},
for a transient state $x_{tr}\in\mathcal{T}_{\bar{\boldsymbol{\pi}}}$,
there always exists an upper bound $\Delta<\infty$ such that $\stackrel[n=0]{\infty}{\sum}p^{n}\left(x_{tr},x_{tr}\right)<\Delta$,
where $p^{n}\left(x_{tr},x_{tr}\right)$ denotes the probability of
returning from a transient state $x_{T}$ to itself in $n$ time steps.
In addition, for a recurrent state $x_{rec}$ of $\mathcal{R}_{\bar{\boldsymbol{\pi}}}^{j}$,
it is always true that 
\\
\begin{equation}
\stackrel[n=0]{\infty}{\sum}\gamma^{n}p^{n}\left(x_{rec},x_{rec}\right)>\frac{1}{1-\gamma^{\overline{n}}}\bar{p},\label{eq:case2_1}
\end{equation}
\\
where there exists $\overline{n}$ such that $p^{\overline{n}}\left(x_{rec},x_{rec}\right)$
is nonzero and can be lower bounded by $\bar{p}$ \cite{Durrett1999}.
From~\eqref{eq: utility_function}, one has 
\\
\begin{equation}
\begin{aligned}\boldsymbol{U}_{\bar{\boldsymbol{\pi}}}^{tr} & >\stackrel[n=0]{\infty}{\sum}\left(\stackrel[j=0]{n-1}{\prod}\boldsymbol{\boldsymbol{\gamma}}_{\bar{\boldsymbol{\pi}}}^{tr}\right)\ldotp\boldsymbol{P}_{\bar{\boldsymbol{\pi}}}^{tr}\boldsymbol{P}_{\bar{\boldsymbol{\pi}}}^{n}\left(\mathcal{R},\mathcal{R}\right)\boldsymbol{R}_{\pi}^{rec}\\
 & {\color{black}>\underline{\gamma}^{n}\ldotp\boldsymbol{P}_{\bar{\boldsymbol{\pi}}}^{tr}\boldsymbol{P}_{\bar{\boldsymbol{\pi}}}^{n}\left(\mathcal{R},\mathcal{R}\right)\boldsymbol{R}_{\boldsymbol{\pi}}^{rec}}.
\end{aligned}
\label{eq:case2_2}
\end{equation}
\\
Let $\max\left(\cdot\right)$ and $\min\left(\cdot\right)$ represent
the maximum and minimum entry of an input vector, respectively. The
upper bound $\bar{m}=\left\{ \max\left(\overline{M}\right)\left|\overline{M}<\boldsymbol{P}_{\bar{\pi}}^{tr}\boldsymbol{\bar{P}}\boldsymbol{R}_{\pi}^{rec}\right.\right\} $
and $\bar{m}\geq0$, where $\boldsymbol{\bar{P}}$ is a block matrix
whose nonzero entries are derived similarly to $\bar{p}$ in~\eqref{eq:case2_1}.
The utility $U_{\bar{\boldsymbol{\pi}}}^{tr}\left(x_{R}\right)$ can
be lower bounded from~\eqref{eq:case2_1} and~\eqref{eq:case2_2}
as $U_{\bar{\boldsymbol{\pi}}}^{tr}\left(x_{R}\right)>\frac{1}{1-\underline{\gamma}^{n}}\bar{m}.$
Since $U_{\boldsymbol{\pi}^{*}}^{rec}\left(x_{R}\right)=0$, the contradiction
$U_{\bar{\boldsymbol{\pi}}}^{tr}\left(x_{R}\right)>0$ is achieved
if $\frac{1}{1-\underline{\gamma}^{n}}\bar{m}$. Thus, there exist
$0<\underline{\gamma}<1$ such that $\gamma_{F}>\underline{\gamma}$
and $r_{F}>\underline{\gamma}$, which implies $U_{\bar{\boldsymbol{\pi}}}^{tr}\left(x_{R}\right)>\frac{1}{1-\underline{\gamma}^{n}}\bar{m}\geq0$.
The procedure shows the contradiction of the assumption that $\boldsymbol{\pi}^{*}$
does not satisfy the acceptance condition of $\mathcal{P}$ with non-zero probability
is optimal. 

In summary (i) if $x_{R}\in\mathcal{R}_{\bar{\boldsymbol{\pi}}}^{j}$, we obtain
$U_{\bar{\boldsymbol{\pi}}}^{rec}\left(x_{R}\right)\geq 0$ from~\eqref{eq:case1}. (ii) if $x_{R}\in\mathcal{T}_{\bar{\boldsymbol{\pi}}}$, the inequality $U_{\bar{\boldsymbol{\pi}}}^{tr}\left(x_{R}\right)>0$ holds according to~\eqref{eq:case2_1} and~\eqref{eq:case2_2}. 

Accordingly, we prove the original optimal policies using the reward in Theorem~\ref{thm:conclusion} remain invariant during the procedure of safe guiding. Thus, we conclude that the safe guiding does not alter the original optimal polices of \eqref{eq:ExpRetrn_shaped} by applying shaped reward \eqref{eq:Shaped_Reward} from the work \cite{Ng1999}.

\subsection{Experimental Details \label{apped:experiment}}

In each experiment, the LTL tasks are converted into LDGBA that is applied to construct the modular DDPG algorithm. The P-MDP between E-LDGBA and cl-MDP is synthesized on the fly. As for each actor/critic structure, we used the same feed-forward neural network setting with 3 fully connected layers with $[64, 64, 64]$ units and ReLu activations. We initiate a Gaussian action distribution for the continuous action space parameterized via actors.
The parameters of the base reward function,  reward shaping, and exploration guiding are set as $r_{F}=0.9$, $\gamma_{F}=0.99$, $\eta_{\Phi}=1000$, and $r_{n}=-50$. The training settings and complexity analysis are shown in Table~\ref{tab:training} and provide a comprehensive comparison of time complexity for different tasks using
 various baselines.

\end{appendices}

\bibliographystyle{IEEEtran}
\bibliography{references}

\begin{thebibliography}{10}
\providecommand{\url}[1]{#1}
\csname url@samestyle\endcsname
\providecommand{\newblock}{\relax}
\providecommand{\bibinfo}[2]{#2}
\providecommand{\BIBentrySTDinterwordspacing}{\spaceskip=0pt\relax}
\providecommand{\BIBentryALTinterwordstretchfactor}{4}
\providecommand{\BIBentryALTinterwordspacing}{\spaceskip=\fontdimen2\font plus
\BIBentryALTinterwordstretchfactor\fontdimen3\font minus
  \fontdimen4\font\relax}
\providecommand{\BIBforeignlanguage}[2]{{%
\expandafter\ifx\csname l@#1\endcsname\relax
\typeout{** WARNING: IEEEtran.bst: No hyphenation pattern has been}%
\typeout{** loaded for the language `#1'. Using the pattern for}%
\typeout{** the default language instead.}%
\else
\language=\csname l@#1\endcsname
\fi
#2}}
\providecommand{\BIBdecl}{\relax}
\BIBdecl

\bibitem{Sutton2018}
R.~S. Sutton and A.~G. Barto, \emph{Reinforcement learning: An
  introduction}.\hskip 1em plus 0.5em minus 0.4em\relax MIT press, 2018.

\bibitem{Schulman2015}
J.~Schulman, S.~Levine, P.~Abbeel, M.~Jordan, and P.~Moritz, ``Trust region
  policy optimization,'' in \emph{International conference on machine
  learning}.\hskip 1em plus 0.5em minus 0.4em\relax PMLR, 2015, pp. 1889--1897.

\bibitem{Lillicrap2016}
T.~P. Lillicrap, J.~J. Hunt, A.~Pritzel, N.~Heess, T.~Erez, Y.~Tassa,
  D.~Silver, and D.~Wierstra, ``\BIBforeignlanguage{English}{Continuous control
  with deep reinforcement learning},'' in
  \emph{\BIBforeignlanguage{English}{Int. Conf. Learn. Represent.}}, San Juan,
  Puerto rico, 2016.

\bibitem{Schulman2017}
J.~Schulman, F.~Wolski, P.~Dhariwal, A.~Radford, and O.~Klimov, ``Proximal
  policy optimization algorithms,'' \emph{arXiv preprint arXiv:1707.06347},
  2017.

\bibitem{Ames2016}
A.~D. Ames, X.~Xu, J.~W. Grizzle, and P.~Tabuada, ``Control barrier function
  based quadratic programs for safety critical systems,'' \emph{IEEE
  Transactions on Automatic Control}, vol.~62, no.~8, pp. 3861--3876, 2016.

\bibitem{choi2020reinforcement}
J.~Choi, F.~Castaneda, C.~J. Tomlin, and K.~Sreenath, ``Reinforcement learning
  for safety-critical control under model uncertainty, using control lyapunov
  functions and control barrier functions,'' \emph{arXiv preprint
  arXiv:2004.07584}, 2020.

\bibitem{castaneda2021pointwise}
F.~Casta{\~n}eda, J.~J. Choi, B.~Zhang, C.~J. Tomlin, and K.~Sreenath,
  ``Pointwise feasibility of gaussian process-based safety-critical control
  under model uncertainty,'' \emph{arXiv preprint arXiv:2106.07108}, 2021.

\bibitem{Lillicrap2015}
T.~P. Lillicrap, J.~J. Hunt, A.~Pritzel, N.~Heess, T.~Erez, Y.~Tassa,
  D.~Silver, and D.~Wierstra, ``Continuous control with deep reinforcement
  learning,'' \emph{arXiv preprint arXiv:1509.02971}, 2015.

\bibitem{Garcia2015}
J.~Garc{\i}a and F.~Fern{\'a}ndez, ``A comprehensive survey on safe
  reinforcement learning,'' \emph{Journal of Machine Learning Research},
  vol.~16, no.~1, pp. 1437--1480, 2015.

\bibitem{Seeger2008}
M.~W. Seeger, S.~M. Kakade, and D.~P. Foster, ``Information consistency of
  nonparametric gaussian process methods,'' \emph{IEEE Transactions on
  Information Theory}, vol.~54, no.~5, pp. 2376--2382, 2008.

\bibitem{Fisac2018}
J.~F. Fisac, A.~K. Akametalu, M.~N. Zeilinger, S.~Kaynama, J.~Gillula, and
  C.~J. Tomlin, ``A general safety framework for learning-based control in
  uncertain robotic systems,'' \emph{IEEE Transactions on Automatic Control},
  2018.

\bibitem{Wang2018}
L.~Wang, E.~A. Theodorou, and M.~Egerstedt, ``Safe learning of quadrotor
  dynamics using barrier certificates,'' in \emph{IEEE International Conference
  on Robotics and Automation (ICRA)}.\hskip 1em plus 0.5em minus 0.4em\relax
  IEEE, 2018, pp. 2460--2465.

\bibitem{Cheng2019}
R.~Cheng, G.~Orosz, R.~M. Murray, and J.~W. Burdick, ``End-to-end safe
  reinforcement learning through barrier functions for safety-critical
  continuous control tasks,'' in \emph{Proceedings of the AAAI Conference on
  Artificial Intelligence}, vol.~33, no.~01, 2019, pp. 3387--3395.

\bibitem{dhiman2021control}
V.~Dhiman, M.~J. Khojasteh, M.~Franceschetti, and N.~Atanasov, ``Control
  barriers in bayesian learning of system dynamics,'' \emph{IEEE Transactions
  on Automatic Control}, 2021.

\bibitem{emam2021safe}
Y.~Emam, P.~Glotfelter, Z.~Kira, and M.~Egerstedt, ``Safe model-based
  reinforcement learning using robust control barrier functions,'' \emph{arXiv
  preprint arXiv:2110.05415}, 2021.

\bibitem{Cai2021modular}
M.~Cai, M.~Hasanbeig, S.~Xiao, A.~Abate, and Z.~Kan, ``Modular deep
  reinforcement learning for continuous motion planning with temporal logic,''
  \emph{IEEE Robotics and Automation Letters}, vol.~6, no.~4, pp. 7973--7980,
  2021.

\bibitem{Kloetzer2009}
M.~Kloetzer and C.~Belta, ``Automatic deployment of distributed teams of robots
  from temporal logic motion specifications,'' \emph{IEEE Transactions on
  Robotics}, vol.~26, no.~1, pp. 48--61, 2009.

\bibitem{Guo2015}
M.~Guo and D.~V. Dimarogonas, ``Multi-agent plan reconfiguration under local
  {LTL} specifications,'' \emph{The International Journal of Robotics
  Research}, vol.~34, no.~2, pp. 218--235, 2015.

\bibitem{Lahijanian2016}
M.~Lahijanian, M.~R. Maly, D.~Fried, L.~E. Kavraki, H.~Kress-Gazit, and M.~Y.
  Vardi, ``Iterative temporal planning in uncertain environments with partial
  satisfaction guarantees,'' \emph{EEE Transactionson Robotics}, vol.~32,
  no.~3, pp. 583--599, 2016.

\bibitem{sahin2019multirobot}
Y.~E. Sahin, P.~Nilsson, and N.~Ozay, ``Multirobot coordination with counting
  temporal logics,'' \emph{IEEE Transactions on Robotics}, vol.~36, no.~4, pp.
  1189--1206, 2019.

\bibitem{wang2015temporal}
J.~Wang, X.~Ding, M.~Lahijanian, I.~C. Paschalidis, and C.~A. Belta, ``Temporal
  logic motion control using actor--critic methods,'' \emph{The International
  Journal of Robotics Research}, vol.~34, no.~10, pp. 1329--1344, 2015.

\bibitem{Icarte2018}
R.~T. Icarte, T.~Klassen, R.~Valenzano, and S.~McIlraith, ``Using reward
  machines for high-level task specification and decomposition in reinforcement
  learning,'' in \emph{International Conference on Machine Learning}, 2018, pp.
  2107--2116.

\bibitem{Camacho2019}
A.~Camacho, R.~T. Icarte, T.~Q. Klassen, R.~A. Valenzano, and S.~A. McIlraith,
  ``{LTL} and beyond: Formal languages for reward function specification in
  reinforcement learning.'' in \emph{IJCAI}, vol.~19, 2019, pp. 6065--6073.

\bibitem{aksaray2021probabilistically}
D.~Aksaray, Y.~Yazicioglu, and A.~S. Asarkaya, ``Probabilistically guaranteed
  satisfaction of temporal logic constraints during reinforcement learning,''
  in \emph{2021 IEEE/RSJ International Conference on Intelligent Robots and
  Systems (IROS)}.\hskip 1em plus 0.5em minus 0.4em\relax IEEE, 2021.

\bibitem{hasanbeig2019reinforcement}
M.~Hasanbeig, Y.~Kantaros, A.~Abate, D.~Kroening, G.~J. Pappas, and I.~Lee,
  ``Reinforcement learning for temporal logic control synthesis with
  probabilistic satisfaction guarantees,'' in \emph{2019 IEEE 58th Conference
  on Decision and Control (CDC)}.\hskip 1em plus 0.5em minus 0.4em\relax IEEE,
  2019, pp. 5338--5343.

\bibitem{bozkurt2020control}
A.~K. Bozkurt, Y.~Wang, M.~M. Zavlanos, and M.~Pajic, ``Control synthesis from
  linear temporal logic specifications using model-free reinforcement
  learning,'' in \emph{2020 IEEE International Conference on Robotics and
  Automation (ICRA)}.\hskip 1em plus 0.5em minus 0.4em\relax IEEE, 2020, pp.
  10\,349--10\,355.

\bibitem{Sickert2016}
S.~Sickert, J.~Esparza, S.~Jaax, and J.~K{\v{r}}et{\'\i}nsk{\`y},
  ``Limit-deterministic {B}{\"u}chi automata for linear temporal logic,'' in
  \emph{Int. Conf. Comput. Aided Verif.}\hskip 1em plus 0.5em minus 0.4em\relax
  Springer, 2016, pp. 312--332.

\bibitem{Li2019}
X.~Li, Z.~Serlin, G.~Yang, and C.~Belta, ``A formal methods approach to
  interpretable reinforcement learning for robotic planning,'' \emph{Science
  Robotics}, vol.~4, no.~37, 2019.

\bibitem{vasile2020reactive}
C.~I. Vasile, X.~Li, and C.~Belta, ``Reactive sampling-based path planning with
  temporal logic specifications,'' \emph{The International Journal of Robotics
  Research}, vol.~39, no.~8, pp. 1002--1028, 2020.

\bibitem{kantaros2020reactive}
Y.~Kantaros, M.~Malencia, V.~Kumar, and G.~J. Pappas, ``Reactive temporal logic
  planning for multiple robots in unknown environments,'' in \emph{2020 IEEE
  International Conference on Robotics and Automation (ICRA)}.\hskip 1em plus
  0.5em minus 0.4em\relax IEEE, 2020, pp. 11\,479--11\,485.

\bibitem{kantaros2020stylus}
Y.~Kantaros and M.~M. Zavlanos, ``Stylus*: A temporal logic optimal control
  synthesis algorithm for large-scale multi-robot systems,'' \emph{The
  International Journal of Robotics Research}, vol.~39, no.~7, pp. 812--836,
  2020.

\bibitem{luo2021abstraction}
X.~Luo, Y.~Kantaros, and M.~M. Zavlanos, ``An abstraction-free method for
  multirobot temporal logic optimal control synthesis,'' \emph{IEEE
  Transactions on Robotics}, 2021.

\bibitem{srinivasan2020control}
M.~Srinivasan and S.~Coogan, ``Control of mobile robots using barrier functions
  under temporal logic specifications,'' \emph{IEEE Transactions on Robotics},
  vol.~37, no.~2, pp. 363--374, 2020.

\bibitem{schillinger2019hierarchical}
P.~Schillinger, M.~B{\"u}rger, and D.~V. Dimarogonas, ``Hierarchical ltl-task
  mdps for multi-agent coordination through auctioning and learning,''
  \emph{The international journal of robotics research}, 2019.

\bibitem{jagtap2020formal}
P.~Jagtap, S.~Soudjani, and M.~Zamani, ``Formal synthesis of stochastic systems
  via control barrier certificates,'' \emph{IEEE Transactions on Automatic
  Control}, vol.~66, no.~7, pp. 3097--3110, 2020.

\bibitem{Paulsen2016}
V.~I. Paulsen and M.~Raghupathi, \emph{An introduction to the theory of
  reproducing kernel Hilbert spaces}.\hskip 1em plus 0.5em minus 0.4em\relax
  Cambridge university press, 2016, vol. 152.

\bibitem{Thrun200}
T.~Sebastian, B.~Wolfram, and D.~Fox, ``Probabilistic robotics,''
  \emph{Communications of the ACM}, vol.~45, no.~3, pp. 52--57, 2002.

\bibitem{Bacchus1996}
F.~Bacchus, C.~Boutilier, and A.~Grove, ``Rewarding behaviors,'' in
  \emph{National Conference on Artificial Intelligence}, 1996, pp. 1160--1167.

\bibitem{Watkins1992}
C.~J. Watkins and P.~Dayan, ``Q-learning,'' \emph{Mach. Learn.}, vol.~8, no.
  3-4, pp. 279--292, 1992.

\bibitem{Baier2008}
C.~Baier and J.-P. Katoen, \emph{Principles of model checking}.\hskip 1em plus
  0.5em minus 0.4em\relax MIT press, 2008.

\bibitem{kloetzer2008fully}
M.~Kloetzer and C.~Belta, ``A fully automated framework for control of linear
  systems from temporal logic specifications,'' \emph{IEEE Transactions on
  Automatic Control}, vol.~53, no.~1, pp. 287--297, 2008.

\bibitem{Kretinsky2018}
J.~Kret{\'{\i}}nsk{\'{y}}, T.~Meggendorfer, and S.~Sickert, ``Owl: {A} library
  for {\(\omega\)}-words, automata, and {LTL},'' in \emph{Autom. Tech. Verif.
  Anal.}\hskip 1em plus 0.5em minus 0.4em\relax Springer, 2018, pp. 543--550.

\bibitem{barth2018distributed}
G.~Barth-Maron, M.~W. Hoffman, D.~Budden, W.~Dabney, D.~Horgan, D.~Tb,
  A.~Muldal, N.~Heess, and T.~Lillicrap, ``Distributed distributional
  deterministic policy gradients,'' \emph{arXiv preprint arXiv:1804.08617},
  2018.

\bibitem{fujimoto2018addressing}
S.~Fujimoto, H.~Hoof, and D.~Meger, ``Addressing function approximation error
  in actor-critic methods,'' in \emph{International conference on machine
  learning}.\hskip 1em plus 0.5em minus 0.4em\relax PMLR, 2018, pp. 1587--1596.

\bibitem{Ng1999}
A.~Y. Ng, D.~Harada, and S.~Russell, ``Policy invariance under reward
  transformations: Theory and application to reward shaping,'' in \emph{ICML},
  vol.~99, 1999, pp. 278--287.

\bibitem{GPs2006}
C.~Rasmussen and C.~Williams, \emph{Gaussian Processes for Machine Learning},
  ser. Adaptive Computation and Machine Learning.\hskip 1em plus 0.5em minus
  0.4em\relax Cambridge, MA, USA: MIT Press, Jan. 2006.

\bibitem{srinivas2012information}
N.~Srinivas, A.~Krause, S.~M. Kakade, and M.~W. Seeger, ``Information-theoretic
  regret bounds for gaussian process optimization in the bandit setting,''
  \emph{IEEE transactions on information theory}, vol.~58, no.~5, pp.
  3250--3265, 2012.

\bibitem{Nguyen2016}
Q.~Nguyen and K.~Sreenath, ``Exponential control barrier functions for
  enforcing high relative-degree safety-critical constraints,'' in
  \emph{American Control Conference (ACC)}.\hskip 1em plus 0.5em minus
  0.4em\relax IEEE, 2016, pp. 322--328.

\bibitem{Xiao2019}
W.~Xiao and C.~Belta, ``High order control barrier functions,'' \emph{IEEE
  Transactions on Automatic Control}, 2021.

\bibitem{Agrawal2017}
A.~Agrawal and K.~Sreenath, ``Discrete control barrier functions for
  safety-critical control of discrete systems with application to bipedal robot
  navigation.'' in \emph{Robotics: Science and Systems (RSS)}, 2017.

\bibitem{Srinivasan2020}
M.~Srinivasan, A.~Dabholkar, S.~Coogan, and P.~A. Vela, ``Synthesis of control
  barrier functions using a supervised machine learning approach,'' in
  \emph{2020 IEEE/RSJ International Conference on Intelligent Robots and
  Systems (IROS)}.\hskip 1em plus 0.5em minus 0.4em\relax IEEE, 2020, pp.
  7139--7145.

\bibitem{long2021learning}
K.~Long, C.~Qian, J.~Cort{\'e}s, and N.~Atanasov, ``Learning barrier functions
  with memory for robust safe navigation,'' \emph{IEEE Robotics and Automation
  Letters}, vol.~6, no.~3, pp. 4931--4938, 2021.

\bibitem{Mnih2015}
V.~Mnih, K.~Kavukcuoglu, D.~Silver, A.~A. Rusu, J.~Veness, M.~G. Bellemare,
  A.~Graves, M.~Riedmiller, A.~K. Fidjeland, G.~Ostrovski \emph{et~al.},
  ``Human-level control through deep reinforcement learning,'' \emph{Nature},
  vol. 518, no. 7540, pp. 529--533, 2015.

\bibitem{Squyres2009}
S.~W. Squyres, A.~H. Knoll, R.~E. Arvidson, J.~W. Ashley, J.~Bell, W.~M.
  Calvin, P.~R. Christensen, B.~C. Clark, B.~A. Cohen, P.~De~Souza
  \emph{et~al.}, ``Exploration of victoria crater by the mars rover
  opportunity,'' \emph{Science}, vol. 324, no. 5930, pp. 1058--1061, 2009.

\bibitem{Durrett1999}
R.~Durrett and R.~Durrett, \emph{Essentials of stochastic processes}.\hskip 1em
  plus 0.5em minus 0.4em\relax Springer, 1999, vol.~1.

\end{thebibliography}

\end{document}